\documentclass[11pt]{article}
\PassOptionsToPackage{numbers, compress}{natbib} 
\def\withcolors{1}
\def\withnotes{1}
\def\withindex{0}
\usepackage[T1]{fontenc}
\usepackage[utf8]{inputenc}

\usepackage{lmodern}
\usepackage{xspace}                                     
\usepackage[protrusion=true,expansion=true]{microtype}  

\usepackage[normalem]{ulem}

\usepackage{amsfonts,amsmath,bbold, amsthm, mathtools}
\usepackage{thm-restate}
\usepackage{dsfont} 

\usepackage{algorithmicx,algpseudocode,algorithm}

\usepackage[usenames,dvipsnames,table]{xcolor}

\usepackage{csquotes}

\usepackage{relsize}

\usepackage{multirow}
\usepackage{chngpage} 

\usepackage{mfirstuc}

\usepackage{tikz}
\usetikzlibrary{arrows}
\usetikzlibrary{calc,decorations.pathmorphing,patterns}

\ifnum\withindex=1
  \usepackage{makeidx}
  \usepackage{ifthen}
  \newcommand\indexed[2][]{\ifthenelse{\equal{#1}{}}{#2\index{#2}}{#2\index{#1}}}
  \makeindex 
\fi

\usepackage[colorlinks,citecolor=blue,bookmarks=true,linktocpage=true,
pagebackref=true]{hyperref}
\usepackage{aliascnt}
\usepackage[numbered]{bookmark}

\usepackage{fullpage}

\usepackage{titling}

\usepackage[shortlabels]{enumitem}
  \setitemize{noitemsep,topsep=3pt,parsep=2pt,partopsep=2pt} 
  \setenumerate{itemsep=1pt,topsep=2pt,parsep=2pt,partopsep=2pt}
  \setdescription{itemsep=1pt}
  
\usepackage{verbatim}

\usepackage{mleftright} 

\input{preamble}
\usepackage{framed}

\usepackage{mdframed}
\newmdtheoremenv{quest}{Q}
\usepackage{url}

\usepackage{subfigure}
\usepackage{graphicx}
\usepackage{makecell}
\usepackage{multirow}
\usepackage[T1]{fontenc} 
\usepackage[
letterpaper,
textheight=9in,
textwidth=6.5in,
top=1in
]{geometry}
\usepackage{mathtools}
\usepackage{algorithm}
\usepackage{subfigure}
\usepackage{booktabs}
\usepackage{tablefootnote}
\usepackage{threeparttable,adjustbox}

\usepackage{amsmath} 

\DeclareMathOperator*{\argmin}{arg\,min}

\newcommand{\E}[1]{{\mathbf{E}\left[{#1}\right]}}

\usetikzlibrary{automata, positioning, arrows}
\usepackage{xcolor}

\usepackage{color}
\definecolor{deepblue}{rgb}{0,0,0.5}
\definecolor{deepred}{rgb}{0.6,0,0}
\definecolor{deepgreen}{rgb}{0,0.5,0}
\usepackage{listings}
\usepackage[colorinlistoftodos,bordercolor=gray,backgroundcolor=white,linecolor=green,textsize=small]{todonotes}

\newcommand\T{\rule{0pt}{2.8ex}}       
\usepackage[skins]{tcolorbox}

\newcommand{\removed}[1]{}
\usepackage{pifont}

\usepackage{enumitem,amssymb}
\newlist{todolist}{itemize}{2}
\setlist[todolist]{label=$\square$}
\usepackage{pifont}

\makeatletter
\newcommand{\hhat}[1]{%
\begingroup%
  \let\macc@kerna\z@%
  \let\macc@kernb\z@%
  \let\macc@nucleus\@empty%
  \hat{\mathchoice%
    {\raisebox{.2ex}{\vphantom{\ensuremath{\displaystyle #1}}}}%
    {\raisebox{.2ex}{\vphantom{\ensuremath{\textstyle #1}}}}%
    {\raisebox{.16ex}{\vphantom{\ensuremath{\scriptstyle #1}}}}%
    {\raisebox{.14ex}{\vphantom{\ensuremath{\scriptscriptstyle #1}}}}%
    \smash{\hat{#1}}}%
\endgroup%
}
\makeatother

\usepackage[customcolors]{hf-tikz}
\usepackage{MnSymbol,amssymb}
\def \L {\mathcal{L}}
\usepackage{pifont}

\usepackage{natbib}

\usepackage{ocgx}
\newcommand*\samethanks[1][\value{footnote}]{\footnotemark[#1]}
\usepackage{color, colortbl}
\usepackage[table]{xcolor}
\usepackage{wrapfig,lipsum}

\usepackage{svg}
\usepackage{soul}
\usepackage{graphicx}

\usepackage[font=small,labelfont=bf]{caption}
\usepackage{cleveref}

\usepackage{algorithmicx,algpseudocode,algorithm}
\usepackage{csquotes}

\newtheorem{thm}{Theorem}
\newtheorem*{thm*}{Theorem}
\newtheorem{ass}{Assumption}

\def\cB{\mathcal{B}}
\def\cD{\mathcal{D}}

\def\cY{\mathcal{Y}}

\def\cT{\mathcal{T}}
\def\cI{\mathcal{I}}

\def\E{\mathbb E}
\def\R{\mathbb R}

\def\T{\mathbb T}

\def\D{\mathbb D}
\def\L{\mathcal L}

\def\w {\mathbf{w}}

\def\x{\mathbf{x}}
\def\z{\mathbf{z}}

\def\u{\mathbf{u}}
\def\t{\mathbf{t}}

\def\fD{\mathfrak{D}}

\newcommand{\vw}{{\mathbf{w}}}
\newcommand{\bw}{{\mathbf{w}}}

\usepackage{tikz}
\newcommand*\circled[1]{\tikz[baseline=(char.base)]{
            \node[shape=circle,draw,inner sep=1pt] (char) {#1};}}

\begin{document}
\title{A Retention-Centric Framework for Continual Learning with Guaranteed Model Developmental Safety}

\author{
    Gang Li \thanks{Department of Computer Science \& Engineering, Texas A\&M University, College Station, USA. } 
    \and
    Wendi Yu \textcolor{red}{\samethanks[1]}
    \and 
    Yao Yao\thanks{Department of Mathematics, The University of Iowa, Iowa City, USA.}
    \and
    Wei Tong\thanks{Global Research and Development, General Motors, Warren, USA.} 
    \and
    Yingbin Liang\thanks{Department of Electrical and Computer Engineering, The Ohio State University, Columbus, UAS.} 
    \and 
    Qihang Lin\thanks{Tippie College of Business, The University of Iowa, Iowa City, USA.} 
    \and   \\
    Tianbao Yang\textcolor{red} {\samethanks[1]}~\thanks{Correspondence to: \texttt{tianbao-yang@tamu.edu}.}
}

\date{}

\maketitle

 \vspace*{-0.55in}
\begin{abstract}

In real-world applications, learning-enabled systems often undergo iterative model development to address challenging or emerging tasks, which involve collecting new data, training a new model and validating the model. This continual model development process raises a significant issue that acquiring new or improving existing capabilities may inadvertently lose good capabilities of the old model, also known as catastrophic forgetting. While existing continual learning aims to mitigate catastrophic forgetting by trading off performance on previous tasks and new tasks to ensure good average performance, it often falls short in cost-sensitive applications, where failing to preserve essential established capabilities introduces unforeseen costs and risks and substantial expenses for re-improving these capabilities. To address this issue, we impose a requirement on learning systems to ensure that a new model strictly retains important capabilities of the old model while improving target-task performance, which we term model developmental safety. To ensure model developmental safety,  we propose a retention-centric framework with data-dependent constraints, and study how to continually develop a pretrained  CLIP model for acquiring new or improving existing capabilities of image classification. We propose an efficient constrained optimization algorithm with theoretical guarantees and use its insights to finetune the CLIP model with task-dependent heads for promoting the model developmental safety. Experiments on autonomous driving and scene recognition datasets validate the efficacy of our method\footnote{Code is available at \url{https://github.com/GangLii/DevSafety}}.

\end{abstract}

\section{Introduction}

Learning-enabled systems are rapidly transforming various sectors, with applications in autonomous vehicles, medical diagnosis, and financial prediction. These systems often rely on ML models that are trained on vast amounts of data.  However, the inherent complexity of
the environments in which these systems operate often presents critical challenges, e.g., dealing with corner cases and rare scenarios that deviate from the norm.  Additionally, real-world scenarios continuously evolve, presenting new challenges and requiring the system to adapt.  These necessitate an iterative development process where models are constantly refined and improved based on new data. Continuously updating the model has become a norm especially in the era of large foundation models, e.g., ChatGPT has experienced several cycles of development from GPT3.5 to GPT4 and GPT4o and recent GPTo1. 

However, this iterative model development process raises a significant issue, i.e., the model development for improving the existing capabilities or acquiring new capabilities may inadvertently lose the previously acquired capabilities of the old model. This issue has been widely observed and documented as catastrophic forgetting  when models are trained to learn a sequence of contents~\citep{mccloskey1989catastrophic}. Tremendous studies have been conducted to mitigate the forgetting problem in continual learning literature~\citep{zhou2022model,NEURIPS2019_fa7cdfad,NIPS2017_0efbe980,DBLP:journals/corr/LiH16e,doi:10.1073/pnas.1611835114}. However, these works primarily focus on mitigating the catastrophic forgetting problem, by trading off performance on previous tasks and new tasks to have good average performance~\citep{wang2024comprehensive}, but do not strictly retain critical existing abilities (i.e., ensuring zero forgetting) while learning new tasks. Ensuring zero forgetting is crucial for many cost-sensitive applications, as failure to strictly preserve the model’s essential capabilities not only introduces unforeseen costs and risks but also imposes substantial expenses in the re-improving of these measures, such as in domains like autonomous driving, where established capabilities are usually critical, and re-validation and re-verification is complex and could cost billions of dollar~\citep{rajabli2020software,koopman2016challenges,McKinseyCompany}. This presents a significant challenge for iterative model development process.

To address this challenge, this paper formally introduces {\bf model developmental safety (MDS)} as a requirement of a learning system that in the model development process the new model should strictly retain the existing important capabilities of the old model while improving its performance on target tasks. 
This concept subtly differs from trading off performance between previous tasks and new tasks to have good average performance of existing continual learning approaches. Moreover, MDS cannot be achieved by the naive weighting method that optimizes a weighted loss by combining protected and target tasks' losses and tuning the weight to preserve protected capabilities. This approach does not necessarily retain the old model's performance across \textbf{all} protected tasks even if the weight is large enough, observed in our experiments (Table~\ref{tab:rm}), and will yield no improvement on target tasks if the weight is too large. A more effective algorithm is required to efficiently find a model that not only retains the performance on protected tasks but also improves the performance on target tasks.  
To the best of our knowledge, no such algorithm currently exists.

\begin{figure}
    
  \begin{minipage}[c]{0.35\textwidth}
    \caption{
       Performance of recognizing 6 weather conditions for autonomous driving with two rounds of model development using new data. The Round 1  development targets at \textit{overcast} and Round 2  aims to improve recognizing \textit{foggy}. Base refers to the CLIP model finetuned on BDD100K data.
    } 
    \label{fig:multi-round}
  \end{minipage} \hfill
  \begin{minipage}[c]{0.64\textwidth}
    \includegraphics[width=\textwidth]{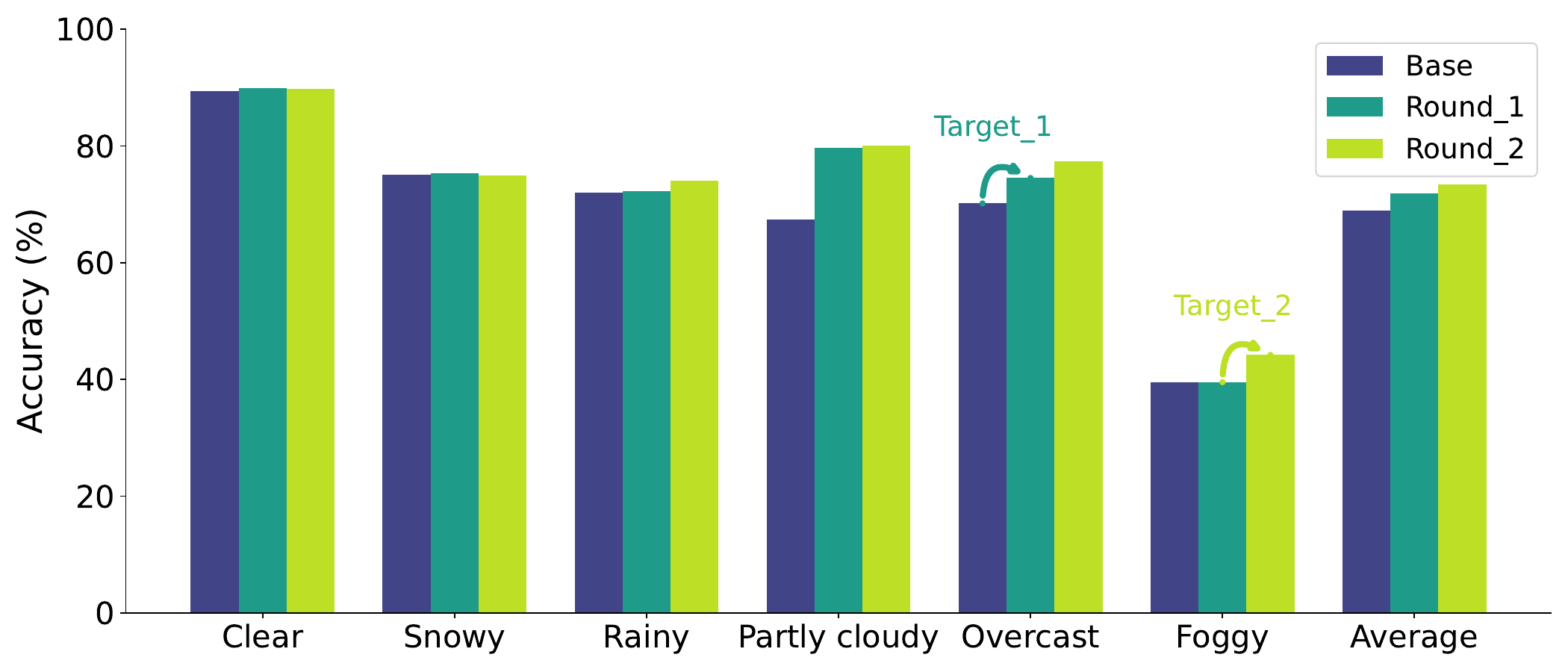}
  \end{minipage}
  \vspace{-0.2in}
\end{figure}

This paper aims to address this critical gap by introducing a novel retention-centric framework to ensure MDS. We propose to formulate the MDS as data-dependent constraints, which offers statistical guarantee for strict preservation of performance for all protected tasks. With this framework, we explore developing a pretrained CLIP model~\citep{DBLP:conf/icml/RadfordKHRGASAM21} for acquiring new capabilities or improving existing ones in image classification. We propose an efficient constrained optimization algorithm with theoretical guarantee. With insights from theoretical analysis, we finetune the CLIP model with task-dependent heads to facilitate MDS. Finally, we demonstrate the efficacy of our approach on autonomous driving dataset and scene recognition dataset, highlighting the practical importance of MDS in real-world scenarios. Our contributions are summarized below:
\begin{itemize}[leftmargin=*]
    \item We introduce a retention-centric framework by formulating the MDS as data-dependent constraints, {which offer statistical guarantee for strictly preserving performance for every protected task.}
    \item We propose an efficient constrained optimization algorithm with theoretical guarantee to develop a pretrained CLIP model for acquiring new or improving existing capabilities of image classification.
    \item We conduct comprehensive experiments to study the proposed algorithm and compare it with existing baselines to demonstrate its effectiveness. An experimental result for ensuring MDS in improving vision-based perception capabilities of autonomous driving is shown in Figure~\ref{fig:multi-round}.
\end{itemize}

\section{Related Work}
\paragraph{Continual learning.} 
This work is closely related to Continual learning (CL), also known as lifelong learning,  yet it exhibits nuanced differences. Continual learning usually refers to learning a sequence of tasks one by one and accumulating knowledge like human instead of substituting knowledge~\citep{wang2024comprehensive,qu2021recent}. The core issue in CL is known as catastrophic forgetting~\citep{mccloskey1989catastrophic}, i.e., the learning of the later tasks may \textbf{significantly} degrade the performance of the model for the earlier tasks. There is a vast literature of CL of deep neural networks (DNNs)~\citep{DBLP:journals/corr/abs-1812-03596,10.5555/3295222.3295393,DBLP:journals/corr/abs-1910-07104,10.5555/3294996.3295218,DBLP:conf/nips/GuoLYR20,DBLP:journals/corr/abs-1802-07569}.  Different approaches have been investigated to mitigate catastrophic forgetting, including regularization-based approaches~\citep{castro2018end, doi:10.1073/pnas.1611835114,10.5555/3305890.3306093,DBLP:journals/corr/LiH16e}, expansion-based approaches~\citep{ zhou2022model,yan2021dynamically, Li2019LearnTG, DBLP:journals/corr/RusuRDSKKPH16}, and memory-based approaches~\citep{buzzega2020dark, cha2021co2l, DBLP:conf/nips/GuoLYR20,10.5555/3295222.3295393,chaudhry2018efficient}. 
 The framework proposed in this work is similar to conventional memory-based approaches in the sense that both use examples of existing tasks to regulate learning. However, the key difference is that most existing continual learning focuses on the trade-off between learning plasticity and memory stability and aims to find a proper balance between performance on previous tasks and new tasks~\citep{wang2024comprehensive}. Hence, they do not provide a guarantee for MDS. A recent work~\citep{peng2023ideal} has proposed an ideal continual learner that never forgets by assuming that all tasks share the same optimal solution. However, it is not implementable for deep learning problems. Besides, existing continual learning studies usually highlight resource efficiency when accumulating knowledge by reducing the number of samples of previous tasks. In contrast, this work tends to utilize more examples to construct  constraints for protected tasks to facilitate MDS.

\textbf{AI Safety.}
Our notion of model developmental safety should not be confused with AI safety. The latter is a field concerned with mitigating risks associated with AI, whose surge in attention stems from the growing capabilities of AI systems, particularly large foundation models~\cite{kojima2022large,wei2022emergent,bubeck2023sparks,radford2021learning}. As these models become more adept at complex tasks, concerns around potential misuse, bias, and unintended consequences rise proportionally. \cite{DBLP:journals/corr/AmodeiOSCSM16} presents several practical research problems related to AI safety, including avoiding side effects, avoiding reward hacking, scalable oversight, safe exploration, and robustness to distributional shift. More recently, \cite{wang2023decodingtrust} elaborate on eight different perspectives to evaluate the trustworthiness of LLMs, including toxicity, stereotype bias, adversarial robustness, out-of-distribution robustness, robustness on adversarial demonstrations, privacy, machine ethics, and fairness. 
These AI safety issues arise in the usage of AI models, and they are distinctive from model developmental safety studied in this work, which arises in the development of AI models. Note that the term "safety" in model developmental safety is to underline that it is important and must be enforced in practice. Therefore, this work provides another dimension for consideration in AI safety, i.e., retention of safety. Any safety features of an AI system that have been acquired and validated should be preserved safely in continuous development.

\textbf{SafeRL.} This work is partially related to SafeRL (Safe Reinforcement Learning), which focuses on developing algorithms and techniques to ensure safety (avoid harmful actions) of RL agents, such as in autonomous driving~\cite{shalev2016safe}, robotics areas~\cite{pham2018optlayer}. Many studies have been conducted in SafeRL domain. A popular approach in SafeRL is to maximize the expected cumulative reward subject to specific safety constraints~\cite{wachi2024survey}, such as expected cumulative safety constraint~\cite{ding2021provably,bura2022dope,tessler2018reward,achiam2017constrained}, state constraint~\cite{thomas2021safe,turchetta2020safe,wang2023enforcing,thananjeyan2021recovery}, joint chance constraint~\cite{ono2015chance,pfrommer2022safe}, etc. However, as SafeRL heavily relies on the special structure of policy optimization for RL, it is different from our work that study a generic developmental safety in model development process. Hence, although sharing the similarity of solving a constrained problem, the algorithms for SafeRL are not applicable to our problem.

\textbf{Constrained Learning.}  
Constrained learning has attracted significant attention in the literature.
Traditional works for constrained optimization include three primary categories: 1) primal methods which do not involve the Lagrange multipliers, e.g., cooperative subgradient methods~\cite{lan2016algorithms, polyak1973method} and level-set methods~\cite{aravkin2019level, lin2018level,lin2018level2};  2) primal-dual methods which reformulate constrained optimization problems as saddle point problems~\cite{hamedani2021primal,nemirovski2004prox}; 3) penalty-based approaches which incorporate constraints by adding a penalty term to the objective function~\cite{xu2021iteration,lan2013iteration,lan2016iteration}. However, most of these works are limited to convex objectives or convex constraints. In recent years, due to its increasing importance in modern machine learning problems, 
such as in applications concerned with fairness~\cite{cotter2019optimization,DBLP:journals/corr/abs-1803-02453}, robustness~\cite{
robey2021adversarial,madry2017towards}, and safety~\cite{paternain2019constrained,paternain2019learning} problems, the research interest has been directed to developing efficient algorithms for non-convex optimization(non-convex objective and non-convex constraint)~\cite{boob2023stochastic,facchinei2021ghost,ma2020quadratically,li2024stochastic,chamon2022constrained, alacaoglu2024complexity}. Among these, \cite{chamon2022constrained} studies how to solve  constrained learning learning with expected non-convex loss and expected non-convex constraints by using empirical data to ensure the PAC learnability,  and proposed a primal-dual algorithm to solve constrained optimization problems in the empirical dual domain. However, their algorithm requires solving the primal problem up to a certain accuracy, which is theoretically not feasible for general non-convex problems.  \cite{boob2023stochastic} introduces a new proximal point method that transforms a non-convex problem into a sequence of convex problems by adding quadratic terms to both the objective and constraints. For solving non-convex optimization problems with  equality constraints, \cite{alacaoglu2024complexity} propose single-loop quadratic penalty and augmented Lagrangian algorithms with variance reduction techniques to
improve the complexity. Nevertheless, none of these algorithms can be directly applied to our large-scale deep learning problem~\eqref{eq:CLIP_train}, due to either prohibitive running cost or failure to handle biased stochastic gradients caused by compositional structure.

\section{Notations and Preliminaries}

{\bf Notations.} We consider developing a model $\w$ to improve its capabilities on a target task $\T_{o}$ while preserving its performance on a set of protected tasks denoted by $\T_1, \ldots, \T_{m}$. A task can be as simple as predicting a class for multi-class classification. 
In the paper, we focus on classification using CLIP models and each task refers to one class. For example, we can consider tasks of detecting different weather conditions in autonomous driving, e.g., foggy, overcast, cloudy, clear, rainy, etc. We assume that each task is associated with a data distribution denoted by $\mathfrak{D}_k$. 
Let $(\x, y)\sim \mathfrak{D}_k$ denote random data of task $\T_k$ with input $\x\in\mathcal X$ (e.g., an image) and output $y\in\mathcal Y$ (e.g., its class label). We assume that each protected task has a set of examples denoted by $\cD_k=\{(\x_i,y_i)\}_{i=1}^{n_k}$,  sampled from $\fD_k$. Let $\ell_k(\w, \x, y) = \ell_k(s(\x; \w), y)$ denote a loss function that measures the loss of the model's prediction $s(\x; \w)$  with respect to the groundtruth $y$ for task $k$. For classification, the loss could be zero-one loss $\ell_{0-1}$ that measures the classification error or the cross-entropy loss $\ell_{ce}$ that is differentiable for learning. We will define these losses shortly for using CLIP models. We denote by $\mathcal L_k(\w, \mathfrak{D}_k) = \E_{\x, y\sim \mathfrak{D}_k}\ell_k(\w, \x, y)$ as the expected loss, and  by $\mathcal L(\w, \cD_k) = \frac{1}{n_k}\sum_{(\x_i, y_i)\sim \cD_k}\ell_k(\w, \x_i, y_i)$ as the empirical loss for task $k$. 

{\bf The CLIP model and Contrastive Loss.} 
We consider optimizing a two-way contrastive loss for each image-text pair $(\x_i, \t_i)$ following \cite{yuan2022provable}:
\begin{align}
\label{eq:contrasive_loss}
   L_{\text{ctr}}&(\w; \x_i, \t_i, \mathcal T_{i}^-,\mathcal I_{i}^- ): =    \notag \\
   & -\tau \log\frac{\exp(E_1(\w, \x_i)^{\top}E_2(\w, \t_i)/\tau)}{\sum_{\t_j\in\mathcal T_i^{-}}\exp(E_1(\w, \x_i)^{\top}E_2(\w, \t_j)/\tau)}  \\ 
    & - \tau \log\frac{\exp(E_2(\w, \t_i)^{\top}E_1(\w, \x_i)/\tau)}{\sum_{\x_j\in\mathcal I_i^-}\exp(E_2(\w, \t_i)^{\top}E_1(\w, \x_j) /\tau)},\notag
    \end{align}
where $E_1(\w, \x)$ and $E_2(\w, \t)$ denotes a (normalized) encoded representation of a image $\x$, and a text $\t$, respectively,  $\mathcal T_{i}^-$ denotes the set of all texts to be contrasted with respect to (w.r.t) $\x_i$ (including itself) and $\mathcal I_{i}^-$ denotes the set of all images to be contrasted w.r.t $\t_i$ (including itself). 

To utilize a CLIP model for multi-class classification with classes $\mathcal C=\{c_1, \ldots, c_K\}$, we will convert a class $c_k$, e.g., "rainy", into a text description of $c_k$, denoted by $\hat\t_k$, e.g., "the weather is rainy", similar to the zero-shot classification scheme of the  CLIP model. Hence, a prediction score (i.e., a logit) for an image $\x$ and a text description $\hat\t_k$ of class $c_k$ is calculated by $s_k(\x; \w)=  E_1(\w, \x)^{\top}  E_2(\w, \hat\t_k)$. The predicted class label is given by $\hat y =  \arg\max_{c_k\in\mathcal C}s_k(\x; \w)$. Hence, given the true class $y\in\mathcal C$, the zero-one loss is given by $\ell_{0,1}(\w, \x, y) = \mathbb I(\hat y\neq y)$, and the cross-entropy loss is given by 
$\ell_{ce}(\w, \x, y) = -\log\frac{\exp(s_y(\x; \w)/\tau_0)}{\sum_{\ell=1}^K\exp(s_l(\x; \w)/\tau_0)}$,
where $\tau_0>0$ is a temperature parameter that controls the balance between the approximation error of the zero-one loss and the smoothness of the function. In particular, a smaller $\tau_0$ gives a smaller approximation error and a larger $\tau_0$ indicates a smaller gradient Lipschitz constant of the loss function in terms of logits.

\section{A Retention-Centric Framework}

\subsection{Model Developmental Safety}


To measure the model developmental safety, it is necessary to evaluate how the performance of the model changes in protected tasks from the old model $\w_{\text{old}}$ to a new model $\w_{\text{new}}$. 
{We introduce the formal definition of model developmental safety (MDS) in Definition~\ref{def:mds}, which ensures the new model strictly preserves performance on each individual protected task.}

\begin{definition}[Model Developmental Safety (MDS)]\label{def:mds}
    In model development process, 
    the model developmental safety is satisfied if   $\mathcal L_k(\w_{\text{new}}, \mathfrak{D}_k)\leq \mathcal L_k(\w_{\text{old}}, \mathfrak{D}_k), \forall k\in\{1,\ldots, m\}$ , where  $\mathcal L_k(\w, \mathfrak{D}_k) = \E_{\x, y\sim \mathfrak{D}_k}\ell_k(\w, \x, y)$. 
\end{definition}

In practice, the developmental safety will be measured using a set of examples $\mathcal S_j\sim\mathfrak{D}_j$ for each protected task. Hence, we define the empirical developmental safety metric, corresponding to Definition~\ref{def:mds}, for evaluation:
\begin{equation}\label{eqn:dev_safety}
\begin{aligned}
    &\text{DevSafety} = \min_{k \in \{1,\cdots, m\}} \left(\mathcal L_k(\w_{\text{old}}, \mathcal{S}_k) -  \mathcal L_k(\w_{\text{new}}, \mathcal{S}_k)\right).
\end{aligned}
\end{equation}
When we use the zero-one loss $\ell_{0-1}$ in the above definitions, we refer to the above developmental safety metric as DevSafety(acc). 

\subsection{A Retention-Centric Approach for Model Developmental Safety}
\label{sec:generalization}

 The key to our retention-centric framework is to utilize examples of protected tasks to define empirical retention constraints when updating the model on a target task.
 In order to develop the model for improving the performance on a target task $\T_o$, we assume that a set of data $\mathcal D$ for $\T_o$ is constructed and a proper objective is given based on application, denoted by $F(\w, \mathcal D)$. Then, our retention-centric approach for model development is imposed by solving the following problem: 
\begin{equation}\label{eqn:cons_population}
\begin{aligned}
\w_{\text{new}}= \arg\min_{\w} &  ~~F(\w, \mathcal D)   \\
\text{s.t.}  &~~\L_k(\w, \cD_k)- \L_k(\w_{\text{old}}, \cD_k)\leq 0,\ k =1, \cdots, m.  
\end{aligned}
\end{equation}
We will propose an algorithm to directly solve this data-dependent constrained optimization problem with a contrastive objective in the context of developing a CLIP model in next section.

\noindent{\bf Generalization Analysis.} Since we can only use empirical data $\cD_1, \dots, \cD_m$  in \eqref{eqn:cons_population}, there exist generalization errors between the {retention} constraints in \eqref{eqn:cons_population} and the MDS we want to ensure in Definition~\ref{def:mds}. The lemma below uses a standard tool of statistical error analysis to bound the generalization error of retention. For simplicity, we assume each protected task is associated with the same loss function, namely, $\ell_k=\ell$ for $k=1,\dots,m$. In the analysis,  we use the Rademacher complexity of the loss class $\mathcal H=\{\ell(\w, \cdot, \cdot): \mathcal X\times\mathcal Y \rightarrow [0,1]|\w\in\mathbb{R}^d\}$ induced by the model $\w$ on $n$ data points, which is denoted by $R_n(\mathcal H)$.
We assume that $R_n(\mathcal H)\leq Cn^{-\alpha}$ for some $C\geq 0$ and $\alpha\leq 0.5$. We note that $\alpha = 0.5$ in the vast majority of model and loss families, including linear models~\citep{kakade2008complexity}, deep neural
networks~\citep{bartlett2002rademacher}, and model families with bounded VC dimension~\citep{bartlett2002rademacher}.

\begin{lemma}[Generalization Error of Rentention]
\label{thm:safety_generalization}
Suppose that $R_n(\mathcal H)\leq Cn^{-\alpha}$ for some $C\geq 0$ and $\alpha\leq 0.5$. Then, with probability at least $1-\delta$, it holds that
    \begin{align*}
        \mathcal L_k(\w_{\text{new}}, \mathfrak{D}_k) -  \mathcal L_k(\w_{\text{old}}, \mathfrak{D}_k)\leq \mathcal L_k(\w_{\text{new}}, \cD_k)-  \mathcal L_k(\w_{\text{old}}, \cD_k)  +\frac{4C}{n_k^\alpha} +2\sqrt{\frac{\ln(2m/\delta)}{2n_k}}, \forall k.
    \end{align*}
\end{lemma}
{\bf Remark:} The lemma indicates that as long as the empirical retention constraints are satisfied, i.e., $\mathcal L_k(\w_{\text{new}}, \cD_k)-  \mathcal L_k(\w_{\text{old}}, \cD_k)\leq 0$, the model developmental safety is ensured up to a statistical error in the order of $O(n^{-\alpha})$, where $n=\min_k n_k$. Hence, the more examples used to construct the retention constraints, the more likely the new model meets MDS requirement. The proof is given in Appendix~\ref{sec:safety_generalization}.

\section{Retention-Centric Development of CLIP Models}

Based on the proposed framework above, in this section, we present an efficient algorithm for improving a pretrained CLIP model on a target task while ensuring MDS on a set of protected tasks.  The CLIP model is of particular interest because (i) it is a foundation model that has been used extensively in many applications; and (ii) can 
adapt to the open-world for handling new classes using languages. However, existing studies have shown that directly applying a pretrained CLIP model (e.g., OpenAI's CLIP model) to a certain downstream application yields varying performance across different classes~\cite{parashar2024neglected}. Rare concepts (e.g., foggy) usually has worse performance than frequent concepts (e.g., clear), making it necessary to continuously update. 

Suppose a CLIP model $\w_{\text{old}}$ has been trained. We aim to improve it for a target task $\T_o$ (e.g., classifying foggy). To this end, we collect a set of image-text pairs related to the target task, denoted by $\cD= \{(\x_i, \t_i)\}_{i=1}^{n_o}$. 
As labeled data for rare scenarios (e.g., \textit{foggy}) are usually limited in practice, we consider augmenting the dataset $\cD$ by using a query prompt to search for target-related image-text pairs from the internet (detailed in Appendix~\ref{sec:laion}). For each image-text pair, a set of negative texts has been collected to be contrasted w.r.t. $\x_i$, which together with $\t_i$ form $\mathcal T_{i}^-$, and  a set of negative images has been also collected to be contrasted w.r.t. $\t_i$, which together with $\x_i$ form $\mathcal I_{i}^-$. 

To develop the CLIP model in our retention-centric framework, we instantiate~\eqref{eqn:cons_population} as:
\begin{equation}
\label{eq:CLIP_train}
\begin{aligned}
&\min_{\w}~~~  F(\w, \cD): = \frac{1}{n_o}\sum\nolimits_{(\x_i, \t_i)\in\cD} L_{\text{ctr}}(\w; \x_i, \t_i,\mathcal T_{i}^-,\mathcal I_{i}^-)\\
& \text{s.t.} ~~~ h_k(\w) :=\L_k(\w, \cD_k)- \L_k(\w_{\text{old}}, \cD_k)\leq 0,\ k =1, \cdots, m.
\end{aligned}
\end{equation}

\subsection{Efficient Optimization}
The optimization problem in \eqref{eq:CLIP_train} is challenging for multiple reasons. First, this problem involves a non-convex objective  and non-convex constraints, so finding a global optimal solution is intractable in general. Second, the objective and constraint functions are formulated using a large dataset, so we need to sample from the dataset in order to construct stochastic gradients of the functions to update the solution. 
Lastly, \eqref{eq:CLIP_train} may contain a large number of constraints, so updating the solutions using the gradients of all constraints may be prohibited. Given these challenges, we need to develop a stochastic optimization for \eqref{eq:CLIP_train} based on advanced techniques and constraint sampling. 

Our method is motivated by the stochastic quadratic penalty method in \cite{alacaoglu2024complexity},  which first converts  \eqref{eq:CLIP_train} into an unconstrained problem by adding a quadratic penalty on the constraints violation to the objective function and then solves the unconstrained problem using a variance-reduced stochastic gradient method. Unfortunately, their method can not be directly applied to \eqref{eq:CLIP_train} because (i) they only consider equality constraints while \eqref{eq:CLIP_train} involves inequality constraints and (ii) they require an unbiased stochastic gradients for each update while the stochastic gradients for \eqref{eq:CLIP_train} will be biased due to the compositional structure. Note that an augmented Lagrangian algorithm (ALA) is also studied by \cite{alacaoglu2024complexity}, which has the same issue as their penalty method. We only consider quadratic penalty method for \eqref{eq:CLIP_train} because it has the same  complexity as the  ALA but is more intuitive and easier to implement. 

A quadratic penalty method converts \eqref{eq:CLIP_train} into the following unconstrained problem:
\begin{equation}\label{eqn:uncons}
\begin{aligned}
 \min_{\w }  \Phi(\w ):= F(\w, \cD) + \frac{1}{m}\sum\nolimits_{k=1}^m  \frac{\beta}{2}([h_k(\w ) ]_+)^2
\end{aligned}
\end{equation}
where $[\cdot]_+=\max\{\cdot,0\}$ and $\beta\geq0$ is the penalty parameter. Under mild conditions\cite{bertsekas2014constrained}, a large enough $\beta$ will ensure the optimal solution to \eqref{eqn:uncons} is also an optimal solution to \eqref{eq:CLIP_train}. In the following, we introduce an efficient stochastic algorithm to solve~\eqref{eqn:uncons}. It is notable that both terms are of finite-sum coupled compositional structure~\cite{DBLP:conf/icml/WangY22}, i.e., $\sum_i f(g_i(\w))$, where $f$ is non-linear.

We discuss how to approximate the gradient of two terms of the objective using mini-batch samples below. Define $g_{1i}(\w) = \frac{1}{|\cT_i^-| }\sum_{\t_j\in\mathcal T_i^{-}}\exp\left(E_1(\w, \x_i)^{\top}E_2(\w, \t_j)-E_1(\w, \x_i)^{\top}E_2(\w, \t_i)/\tau\right) $
and $g_{2i}(\w) = \frac{1}{|\cI_i^-| }\sum_{\x_j\in\mathcal I_i^{-}}\exp\left(E_2(\w, \t_i)^{\top}E_1(\w, \x_j)-E_2(\w, \t_i)^{\top}E_1(\w, \x_i)/\tau\right)$. Then, $F(\w,\cD)= \frac{1}{n_o}\sum_{(\x_i, \t_i)\in\cD}\\ \tau \log g_{1i}(\w) + \tau \log g_{2i}(\w) $ and its gradient is given by
\begin{small}
\begin{align*}
    \nabla F(\w, \cD) = \frac{\tau}{n_o} \sum_{(\x_{i}, \t_i) \in \mathcal{D}} \left( \frac{\nabla g_{1i}(\w)}{g_{1i}(\w)}   + \frac{\nabla g_{2i}(\w)}{g_{2i}(\w)}    \right).
\end{align*}
\end{small}
The major cost of computing $\nabla F(\w; \cD)$ lies on calculating $g_{1i}(\w)$ and $g_{2i}(\w)$ and their gradient for each pair, as it involves all the samples in $\mathcal{T}_i^{-}$ and $\mathcal{I}_i^{-}$. Directly approximating $g_{1i}$ and $g_{2i}$ by a mini-batch  of samples from $\mathcal{T}_i^{-}$ and $\mathcal{I}_i^{-}$ will reduce the computational cost but lead to a biased stochastic gradient of $\nabla F(\w, \cD)$ due to the non-linear dependence of $\nabla F(\w; \cD)$ on $g_{1i}$ and $g_{2i}$, which will cause the issue of requiring a large batch size in order to converge. 

To address this issue, we employ the moving average estimators for estimating $g_{1i}$ and $g_{2i}$ which gradually reduces the aforementioned biases to zero~\cite{yuan2022provable}. More specifically, let $\w^t$ be the solution at iteration $t$. We randomly sample a mini batch $\cB\subset\mathcal D$, and construct mini-batch negatives $\cB_{1,i}\subset \cT_i^-$,  $\cB_{2,i}\subset \cI_i^-$ for each data $(\x_i, \t_i)\in\cB$ and construct the following stochastic estimations of $g_{1i}(\w^t)$ and $g_{2i}(\w^t)$:
\begin{small}
\begin{equation*}
\begin{aligned}
    \hat g_{1i}(\w^t) &:= \frac{1}{|\cB_{1,i}|}\sum\nolimits_{\t_j\in \cB_{1,i}}\exp((E_1(\w, \x_i)^\top E_2(\w, \t_j)- E_1(\w, \x_i)^\top E_2(\w, \t_i)) /\tau) \\
     \hat g_{2i}(\w^t) &:=\frac{1}{|\cB_{2,i}|}\sum\nolimits_{\x_j\in \cB_{2,i}}\exp((E_2(\w, \t_i)^\top E_1(\w, \x_j)-E_2(\w, \t_i)^\top E_1(\w, \x_i))/\tau).
\end{aligned}
\end{equation*}
\end{small}

The moving averaging estimators of $g_{1i}\left(\w^t\right)$ and  $g_{2i}\left(\w^t\right)$ denoted by  $u_{1i}^{t}$ and $u_{2i}^{t}$ are updated by:
\begin{equation}
\begin{aligned}
    u_{1i}^{t+1} = (1-\gamma_1)u_{1i}^{t}+\gamma_1 \hat g_{1i}\left(\w^t\right),~ u_{2i}^{t+1} = (1-\gamma_1)u_{2i}^{t}+\gamma_1 \hat g_{2i}\left(\w^t\right),
\end{aligned}\label{eqn:u_posi}
\end{equation}
where $\gamma_1\in(0,1)$ is a hyper-parameter. The gradient estimator of $F(\w^t, \cD)$ is computed by
\begin{equation}\label{eqn:G1}
    G_1^t = \frac{\tau}{|\cB|}\sum\nolimits_{i\in\cB}\left(\nabla\hat g_{1i}\left(\w^t\right)/u_{1i}^{t}+ \nabla \hat g_{2i}\left(\w^t\right)/u_{2i}^{t} \right).
\end{equation}
The gradient of the quadratic penalized term at $\w^t$ can be approximated similarly by  
\begin{equation}\label{eqn:G2}
G_2^t = \frac{1}{|\cB_{c}|}\sum\nolimits_{k\in\cB_c}\beta[u_k^{t}]_+\nabla \hat h_k(\w^t),
\end{equation}
where $\cB_c$ denotes a sampled subset of protected tasks,  $\hat h_k(\w^t)$ denotes a mini-batch estimator of $h_k(\w^t)$ using mini-batch $\mathcal B_k\subset\cD_k$, and $u_{k}^{t}$ is the moving average estimator of $h_k(\w^t)$ computed by 
\begin{equation}
\begin{aligned}\label{eqn:u_class}
 u_{k}^{t+1}  = (1-\gamma_2)u_{k}^{t}+\gamma_2 \hat h_k(\w^t), \: \hat h_k(\w^t)&= \frac{1}{|\cB_k|}\sum\nolimits_{j\in \cB_k} \ell_{ce}(\w,\x_j,y_j) - \ell_{ce}(\w_{\text{old}},\x_j,y_j) .
\end{aligned}
\end{equation}
We emphasize that the gradient estimator in~(\ref{eqn:G2}) related to the protected tasks, where each proected task has an effective weight $\beta[u_k^{t}]_+$ that is dynamically changing in the learning process, is the key difference from the native weighting method mentioned at the beginning.

\setlength{\textfloatsep}{3pt}
\begin{algorithm}[t!]
\caption{Algorithm for solving~(\ref{eq:CLIP_train})}
\label{alg:clip_class}
\begin{algorithmic}[1]
\State \textbf{Initialization:} choose $\w^0$, $\beta$, $\gamma_1$, $\gamma_2$, $\theta$ and step sizes $\eta$.
\For{$t=0, 1,\cdots,T-1$}
\State{Sample image-text pairs $\cB$ from $\cD$ and protected tasks $\cB_c$ from $\{1,\cdots, m\}$. 
}
\For{ each $(\x_i,\t_i) \in \cB$}
\State{Update $u_{1i}^t$ and $u_{2i}^t$ by Eqn.~(\ref{eqn:u_posi})

}
\EndFor
\State{Update the estimator of gradient $\nabla F(\w^t, \mathcal{D})$ by $G_1^t$ as in Eqn.~(\ref{eqn:G1})
}
\For{ each $k \in \cB_c$}
\State{Sample a minibatch of data from $\cD_k$ denoted by $\cB_k$.
}
\State{Update the estimators of $h_k$  by Eqn.~(\ref{eqn:u_class}).}
\EndFor


\State{Compute the stochastic gradient estimator $G_2^t$ as in Eqn.~(\ref{eqn:G2})} 
\State{Update Gradient Estimator $v^{t+1} = (1-\theta)v^{t}+\theta  (G_1^t +G_2^t)$}
\State Update $\w$ by $\w^{t+1}=\w^{t}-\eta v^{t+1}$. 

\EndFor
\end{algorithmic}
\end{algorithm}

The key steps are presented in Algorithm~\ref{alg:clip_class}. We would like to emphasize  Algorithm~\ref{alg:clip_class} can be easily generalized to handling a general objective function $F(\w)$ whose unbiased stochastic gradient is available by just replacing $G_1^t$ with the stochastic gradient estimator of $F(\w_t)$. 

\subsection{Convergence Analysis}

Since our considered constrained optimization problem is non-convex for both objectives and constraints, a critical concern is whether Algorithm~\ref{alg:clip_class} has some convergence guarantee as standard learning algorithms such as SGD/Adam. We address this in this section. To the best of our knowlege, this is the first convergence analysis of a penalty method for solving non-convex inequliaty constrained optimization.  For analysis, we make the following assumptions.  

\begin{ass}\label{assumption1}
(a) $g_1(\cdot)$ and $g_2(\cdot)$ are $L_g$-Lipschitz continuous and $L_{\nabla g}$-smooth. (b) There exist $C_{g}>0$ and $c_{g}>0$ such that $c_g \leq \min\{g_1(\cdot), g_2(\cdot)\}$ and $\max\{g_1(\cdot), g_2(\cdot)\} \leq C_g$. (c) $h_k(\cdot)$ is $L_h$-Lipschitz continuous and $L_{\nabla h}$-smooth for $k=1,\cdots,m$.
\end{ass} 
\begin{ass}\label{assumption2}
    There exists $\w^0$ such that $h_k(\w^0)\leq 0$ for $k=1,\cdots,m$.
\end{ass}
\begin{ass}\label{assumption3}
    (a) $\mathbb{E}[\|\hat g_{1i}(\w)-g_{1i}(\w)\|^2]\leq \sigma_g^2/|\cB_{1i}|$, $\mathbb{E}[\|\hat g_{2i}(\w)-g_{2i}(\w)\|^2]\leq \sigma_g^2/|\cB_{2i}|$; (b) $\mathbb{E}[\|\nabla \hat g_1(\w)-\nabla g_{1i}(\w)\|^2]\leq \sigma_{\nabla g}^2/|\cB_{1i}|$, $\mathbb{E}[\|\nabla \hat g_{2i}(\w)-\nabla g_{2i}(\w)\|^2]\leq \sigma_{\nabla g}^2/|\cB_{2i}|$; (c) $\mathbb{E}[\|\nabla \hat h_k(\w)-\nabla h_k(\w)\|^2]\leq \sigma_{\nabla h}^2/|\cB_k|$; (d) $\mathbb{E}[\|\hat h_k(\w)-h_k(\w)\|^2]\leq \sigma_{h}^2/|\cB_k|$ for $k=1,\cdots,m$.
\end{ass}
\begin{ass}\label{ass:full_rank}
    There exists a constant $\delta > 0$ such that $\|\nabla \textbf{h}(\vw^t)[\textbf{h}(\vw^t)]_+\|\geq \delta \|[\textbf{h}(\vw^t)]_+\|$ for $t = 0,\cdots,T$, where $\textbf{h}(\w) = [h_1(\w), \ldots, h_m(\w)]^{\top}$ and $\nabla\textbf{h}(\w) = [\nabla h_1(\w), \ldots, \nabla h_m(\w)]$.
\end{ass}
{\bf Remark:} Assumption 1  has been justified in the earlier work~\cite{yuan2022provable,DBLP:conf/icml/QiuHYZ0Y23} for optimizing a global contrastive loss. Assumption 2 is easily satisfied with $\w^0=\w_{\text{old}}$.  Assumption 3 is a standard one that bounds the variance of mini-batch estimators. Assumption 4 is also made in many existing works on optimization with non-convex constraints~\cite{sahin2019inexact,xie2021complexity,alacaoglu2024complexity,lin2022complexity,li2024stochastic}. This assumption is equivalent to that the quadratic penalty term $H(\w):=\frac{\beta}{2m}\| [\textbf{h}(\w)]_+\|^2$ satisfies the Polyak-Lojasiewicz inequality at $\w=\w^t$, meaning that there exists $\delta\geq0$ such that $\|\nabla H(\w^t)\|^2\geq \frac{2\delta^2\beta}{m} H(\w^t)$. Without this assumption, \eqref{eq:CLIP_train} may be intractable because there may exist an iterate $\w^t$ such that $H(\w^t)>0$ but $\nabla H(\w^t)=0$, meaning that $\w^t$ is infeasible but at a flat location of $H(\w)$ so $\w^t$ may get trapped at this location forever. We will show later that a small $\delta$ in Assumption 4 will increase the complexity of our algorithm. Hence, we will present an approach in next subsection to increase $\delta$.

For a non-convex optimization problem like \eqref{eq:CLIP_train}, finding a globally optimal solution is intractable, so almost all numerical algorithms for non-convex problems can only guarantee a Karush-Kuhn-Tucker (KKT) solution defined below.
\begin{definition}
\label{def:KKT}
A solution $\vw$  is a KKT solution to \eqref{eq:CLIP_train} if there exist $\boldsymbol{\lambda}=(\lambda_1,\dots,\lambda_m)^\top\in\mathbb{R}_+^m$ such that $\nabla F(\vw, \cD)+ \nabla \textbf{h}(\vw)\boldsymbol{\lambda}=\mathbf{0}$, $\textbf{h}(\vw)\leq \mathbf{0}$ and $\lambda_kh_k(\w)=0$ for $k=1,\dots,m$.
\end{definition}
We present the convergence theorem of Algorithm~\ref{alg:clip_class} as follows, which shows the iteration complexity of Algorithm~\ref{alg:clip_class} for finding an $\epsilon$-KKT solution, i.e., a solution satisfying the three  conditions in Definition~\ref{def:KKT} up to $\epsilon$ precision.  The proof of the theorem is presented in Appendix~\ref{proof_thm}.
\begin{thm}\label{theorem}
    Suppose Assumptions~\ref{assumption1}, \ref{assumption2}, \ref{assumption3} and \ref{ass:full_rank} hold. Also, suppose, in Algorithm~\ref{alg:clip_class}, set $\beta=\frac{1}{\epsilon\delta}$,\\ $\theta =\min\{\frac{\epsilon^4\delta^2\min\{|\mathcal{B}_c|,|\mathcal{B}_{k}|\}}{672(\sigma_{\nabla h}^2+L_h^2)} ,\frac{\epsilon^2\min\{|\mathcal{B}|,|\mathcal{B}_{1i}|,|\mathcal{B}_{2i}|\}}{1344L_f^2(\sigma_{\nabla g}^2+L_g^2)}\}$, $\gamma_1=\gamma_2 = \min\{\frac{5n_0\theta}{3|\mathcal{B}|}, \frac{5m\theta}{3|\mathcal{B}_c|},\frac{\epsilon^4\delta^2|\mathcal{B}_k|}{26880\sigma_h^2\Tilde{C}_{\nabla h}^2}\}$ and \\$\eta = \min \left\{ \frac{1}{12(L_F+\beta L_H)}, \frac{\theta}{8\sqrt{3}L_F}, \frac{\theta}{8\sqrt{3}L_H \beta},\frac{\gamma_1 |\mathcal{B}|}{40\sqrt{6}L_g L_f \Tilde{C}_{\nabla g} n_0},\frac{\gamma_2 |\mathcal{B}_c|}{40\sqrt{6}\beta L_h\Tilde{C}_{\nabla h} m} \right\}$, where  $\Tilde{C}_{\nabla g}:=\sigma_{\nabla g}+L_g$, $\Tilde{C}_{\nabla h}:=\sigma_{\nabla h}+L_h$, $L_f:=\frac{\tau}{c_g}$, $L_{\nabla f}:=\frac{\tau}{c_g^2}$, $L_F:=2(L_{\nabla g}L_f+L_{\nabla f}L_g^2)$ and $L_H := 2L_{\nabla h} + L_h^2$.     
    Then there exists $\boldsymbol{\lambda}\in\mathbb{R}_+^m$ such that after $T=O(\epsilon^{-7}\delta^{-3})$ iterations Algorithm~\ref{alg:clip_class} satisfies
    \begin{equation*}
           \mathbb{E}\left[ \|\nabla F(\vw^{\hat{t}}, \cD)+ \nabla \textbf{h}(\vw^{\hat{t}})\boldsymbol{\lambda}\| \right]\leq \epsilon,\quad
    \mathbb{E} [\|[\textbf{h}(\vw^{\hat{t}})]_+\|] \leq \epsilon,\quad \mathbb{E} [\boldsymbol{\lambda}^\top [\textbf{h}(\vw^{\hat{t}})]_+]\leq \epsilon
    \end{equation*}
 where $\hat{t}$ selected uniformly at random from $\{1,\cdots,T\}$.
\end{thm}
{\bf Remark:} It is notable that the order of complexity in terms of $\epsilon$ is higher than that of standard learning (i.e., $O(\epsilon^{-4})$). While the complexity for a stochastic constrained optimization could be inherently higher than unconstrained optimization~\cite{alacaoglu2024complexity}, we note that the above complexity is also weaker than the state-of-the-art complexity of stochastic constrained optimization~\cite{alacaoglu2024complexity}. We remark that this is a limitation of the present work due to two reasons:  (i) we use the moving average gradient estimator for sake of implementation; in contrast, they use the advanced variance reduced gradient estimator (STORM), which incurs additional overhead; (ii) we use a constant $\beta$ and they use an increasing $\beta$. In our experiments shown in ablation studies, we find that using a constant $\beta$ is generally better than using an increasing $\beta$. Additionally, the dependence on $\delta$ could also slow down the convergence. We mitigate this issue by utilizing task-dependent heads for CLIP models justified below. 

\subsection{Promoting developmental safety via Task-dependent heads}\label{sec:task-heads}

Below, we present a way to design the text encoder of the CLIP model  such that the value of $\delta$ could be larger. Without causing confusion, we denote by $\w$ the parameter of the text encoder, which consists of two components $\u$ and  $W$ such that the text embedding  $E_2(\w,\t)\in\mathbb{R}^{d_2}$ can be represented as $E_2(\w,\t) =W \cdot \bar E_2(\u,\t)$, where 
 $\bar E_2(\u,\cdot)\in\mathbb{R}^{d_1}$ is a backbone encoder while $W\in\mathbb{R}^{d_2\times d_1}$ is called the head. The idea of task-dependent heads is to let each task $k$ have its own head $W_k= W+U_kV_k^\top$ using  low rank matrices $U_k \in \mathbb{R}^{d_2 \times r}$ and $V_k \in \mathbb{R}^{ d_1 \times r}$, where $r<\min(d_1, d_2)$ is the rank chosen as a hyper-parameter.  
 The output of this class-specific text encoder for task $k$ is $E_2(\u,W, U_k,V_k,\t_k)=(W+U_kV_k^\top)\cdot \bar E_2(\u,\t_k)$. Note that $\|\nabla \textbf{h}(\vw^t)^\top [\textbf{h}(\vw^t)]_+\|^2\geq \lambda_{\min}(\nabla \textbf{h}(\vw^t)^\top\nabla \textbf{h}(\vw^t)) \|[\textbf{h}(\vw^t)]_+\|^2$, where $ \lambda_{\min}(\cdot)$ represents the smallest eigenvalue of a matrix. This means $\min_{t}\lambda_{\min}(\nabla \textbf{h}(\vw^t)^\top\nabla \textbf{h}(\vw^t))$ is a lower bound of $\delta$ in Assumption 4. 
 The following lemma shows that, after expanding $\w$ with $U_k$ and $V_k$, $\lambda_{\min}(\nabla \textbf{h}(\vw^t)^\top\nabla \textbf{h}(\vw^t))$ may increase at some $U_k$ and $V_k$, providing some insight on why the task-dependent heads help to increase the parameter $\delta$ in Assumption 4, reducing the total complexity of our algorithm according to Theorem~\ref{theorem}. 

\begin{lemma}
\label{thm:regularity}
Let $\mathbf{U}=(U_1,\dots,U_m)$ and $\mathbf{V}=(V_1,\dots,V_m)$. Let $\w=(W,\u)$, $\hat\w=(W,\u,\mathbf{U},\mathbf{V})$, $h_k(\w)=h_k(W,\u)$, and $\hat h_k(\hat\w)=h_k(W+U_kV_k^\top,\u)$. 
Suppose $U_kV_k^\top=\mathbf{0}$ for all $k$'s. We have
\small
\begin{eqnarray*}
 \lambda_{\min}\left(\nabla\widehat{\textbf{h}}(\hat\w)^\top\nabla\widehat{\textbf{h}}(\hat\w)\right)  \geq\lambda_{\min}\left(\nabla\textbf{h}(\w)^\top\nabla\textbf{h}(\w)\right)+\min_k\left\{\left\|\nabla_{W} h_k(\w)V_k\right\|_F^2,\left\|\nabla_{W} h_k(\w)^\top U_k\right\|_F^2\right\},
\end{eqnarray*}
\normalsize
 where $\widehat{\textbf{h}}(\hat\w) = [\hat h_1(\hat\w), \ldots,\hat h_m(\hat\w)]^{\top}$ and $\nabla\widehat{\textbf{h}}(\hat \w) = [\nabla \hat h_1(\hat\w), \ldots, \nabla \hat h_m(\hat\w)]$.
\end{lemma}
Following this lemma, in our experiments, we employ the task-dependent heads by setting the initial value of $U_k$ to zero so $U_kV_k^\top=\mathbf{0}$. The proof of the above lemma is given in Appendix~\ref{sec:regularity}

\section{Experiments}

In this section, we conduct extensive experiments to understand our proposed method comprehensively. { We start by presenting a visualization of the learning process of the proposed method in Section \ref{sec:visu} to provide an overview of how our approach works. In Section~\ref{sec:compare_baseline}, we demonstrate the effectiveness of the proposed method in achieving model developmental safety compared with other strong baselines. We present a detailed ablation study to help understand each design choice of the proposed method in Section~\ref{sec:ablation}, including the effect of using external data for contrastive learning on improving the target task, the importance of task-dependent heads, etc. Lastly, we show the potential of our method in multiple rounds of model development in Section~\ref{sec:multiround}}



\textbf{Dataset.}  We experiment on the large-scale diverse driving image dataset, namely BDD100K~\cite{seita2018bdd100k}. This dataset involves classification of six weather conditions, i.e., {\it clear, overcast, snowy, rainy, partly cloudy, foggy}, and of six scene types, i.e., {\it highway, residential area, city street, parking lot, gas station, tunnel}. Since the labels of the official testing dataset are not released, we utilize the official validation set for testing and partition the training dataset into training and validation sets using an 80\%/20\% ratio.  Moreover, we experiment on the scene recognition dataset, Places365 which has 365 classes~\cite{zhou2017places}, 
to verify the effectiveness of the proposed method in handling a large number of constraints. We utilize the standard version of the dataset (i.e., Places365-Standard), with 1.8 million training and 36500 validation images from 365 scene classes. The number of examples for each class varies between 3,068 and 5,000 in the training set. We merge the training dataset and validation dataset and randomly split the whole set into training set, validation set and test set with an 60\%/20\%/20\% ratio.

\textbf{Evaluation Metrics.}  Since the focus of this paper is to ensure model developmental safety during model development process, we separate the performance on the new task and that on protected tasks in our evaluation. 
We measure improvement on target task with $\Delta\text{Acc(Target)} = \text{Acc}(\text{Target}, \w_{\text{new}}) - \text{Acc}(\text{Target}, \w_{\text{old}})$.
Besides, we utilize "DevSafety(acc)" (i.e., Eqn.~\ref{eqn:dev_safety}) to measure the empirical MDS. As optimization involves randomness, we run all the experiments with five different random seeds then calculate the average target accuracy and the percentage of times that DevSafety(acc) is non-negative, denoted as {\bf retention ratio}, to measure the possibility of strictly preserving the performance on protected tasks. For example, the retention ratio is 60\% if 3 out of 5 runs of the method preserve previous performance for all protected tasks.)

\textbf{Experimental Settings.} 
We employ the CLIP ViT-B/16~\cite{radford2021learning} as the backbone network in all our experiments. For BDD100K dataset, we obtain a base model by fine-tuning the pretrained CILP model, following the method proposed in \cite{yuan2022provable}, on the BDD100K training dataset without foggy and tunnel data.  Subsequently, we undertake secondary development to improve the performance of a target class separately. We consider three settings with \textit{foggy}, \textit{overcast} and \textit{tunnel} as the target class. For targeting \textit{foggy}, we consider other weather conditions as protected tasks, for targeting \textit{overcast} we consider other weather conditions except for foggy as protected tasks due to that there is a lack of foggy data in BDD100k for defining a significant constraint. For the same reason, we consider other scence types except \textit{gas station} as protected tasks for targeting \textit{tunnel}. The image-text pairs for the objective function are from the training set of BDD100K and the  external LAION400M \cite{schuhmann2021laion} dataset. Specifically, for each target class, we use a query prompt (detailed in Appendix~\ref{sec:laion}) to search for target-related image-text pairs in LAION400M to augment the set $\cD$. Additionally, we randomly sample a set of image-text pairs from LAION400M that is 10 times larger than target-related pairs 
as negative data for contrasting.  The data of protected tasks used for developmental safety constraints are sampled from the BDD100K training set with varying sizes. Statistics for BDD100K in our experiments are shown in Appendix Table~\ref{tab:bdd100k}. The text templates used for BDD100K dataset are "\textit{the weather is [Weather]}" and "\textit{the scene is a [Scene]}".

For Places365 dataset, we directly utilize the pretrained CLIP model released by \cite{radford2021learning} as the base mode. Then we conduct continual development to improve the performance of \textit{dressing room} class, which has the fewest samples in the dataset, and consider all the other 364 classes as protected tasks. Similar to the setting for BDD100K dataset, we also use a query prompt (detailed in Appendix~\ref{sec:laion}) to search for target-related image-text pairs in LAION400M to augment the set $\cD$. The data of protected tasks used for developmental safety constraints are sampled from the Places365 training set. The text templates used for Places365 dataset are "\textit{the scene is a(n) [Scene]}".

\textbf{Baselines.} To verify the effectiveness of our algorithms, we compare our proposed algorithm with the following baseline methods: (1) FLYP~\cite{goyal2023finetune}, a state-of-the-art CLIP finetuning method that optimizes a contrastive loss on all available data including those used in our objective and constraints.
In our experiments, we utilize the same global contrastive loss (GCL)~\cite{yuan2022provable} instead of mini-batch contrastive loss;  (2) Weighted Combination of Contrastive Losses (WCCL), which utilizes a weight to combine GCL losses on protected tasks and the target task to control the tradeoff between them to achieve model developmental safety; 
(3) GEM~\cite{lopez2017gradient}, which is a strong CL baseline motivated by a similar idea utilizing data of previous tasks for constraints; (4) Co$^2$L\cite{cha2021co2l}, which is a recent SOTA contrastive continual learning baseline; (5) DER\cite{buzzega2020dark}, a strong memory-based continual learning baseline, which mixes rehearsal with knowledge distillation and regularization; (6) Regularization Method (RM), as commonly adopted in continual learning literature~\cite{rebuffi2017icarl,castro2018end}, directly takes the constraints in Eqn.~\eqref{eq:CLIP_train} as a regularization term by adding it to the objective function with a regularization weight $\alpha$. All methods start from the same  CLIP model. More details about baselines are presented in Appendix~\ref{sec:baselines}.

\textbf{Hyperparameter tuning.}
For all methods in our experiments, we tune the learning rate in \{1e-5, 1e-6\} with Cosine scheduler and AdamW optimizer, using a weight decay of 0.1. For BDD100K dataset, we set temperature $\tau_0$ as 0.05. We run each method for a total of 40 epochs with a batch size of 256 and 600 iterations per epoch, except for GEM whose total epochs are tuned in \{1,2,5\} with a batch size of 64 since more iterations lead to exacerbated catastrophic forgetting problems as shown in their paper. For our method, we tune $\beta$ in \{100, 200, 400\}, $\gamma_2$ in \{0.4, 0.6, 0.8\} and set $r=32, |\cB_c|=m, |\cB_k|=10$.  We set $\gamma_1$ to 0.8, $\tau$ to 0.05 in FLYP, WCCL, RM, and our method. For WCCL, we vary the weight parameter $\alpha$ in \{0.5,0.9,0.99\}. For GEM, we tune their small constant $\gamma$ in \{0.5, 1.0\}. For Co$^2$L, we tune their $\tau$ in \{0.05, 0.1\}, $\kappa$ in \{0.1, 0.2\}, $\kappa^*$ in \{0.01, 0.1\}, $\lambda$ in \{0.1, 1, 10\}. For DER, we tune their $\alpha$ in \{0.1, 1, 10\} and $\beta$ in \{0, 1\}. For RM, we tune regularization weight $\alpha$ in \{0.1, 1, 10\}. In hyper-parameters selection for all methods, we prioritize larger {\bf retention ratio} first and consider larger $\Delta\text{Acc (Target)}$ if there is a tie in terms of retention ratio, as we look for models that maximize $\Delta\text{Acc (Target)}$ while satisfying  $\text{DevSafety}\geq 0$.  For Places365 dataset, the temperature $\tau_0$ is set as 0.01. Since there are as many as 364 constraints, we set $|\cB_c|=240, |\cB_k|=2$. We tune $\beta$ in \{600, 1000, 4000\} for our method and regularization weight $\alpha$ in \{1, 10, 100, 1000, 10000\} for RM. 
We run each method five times for a total of 40 epochs with 1400 iterations per epoch, with a batch size of 64.

\begin{figure*}[!t]
\centering

    \includegraphics[width=0.95\linewidth]{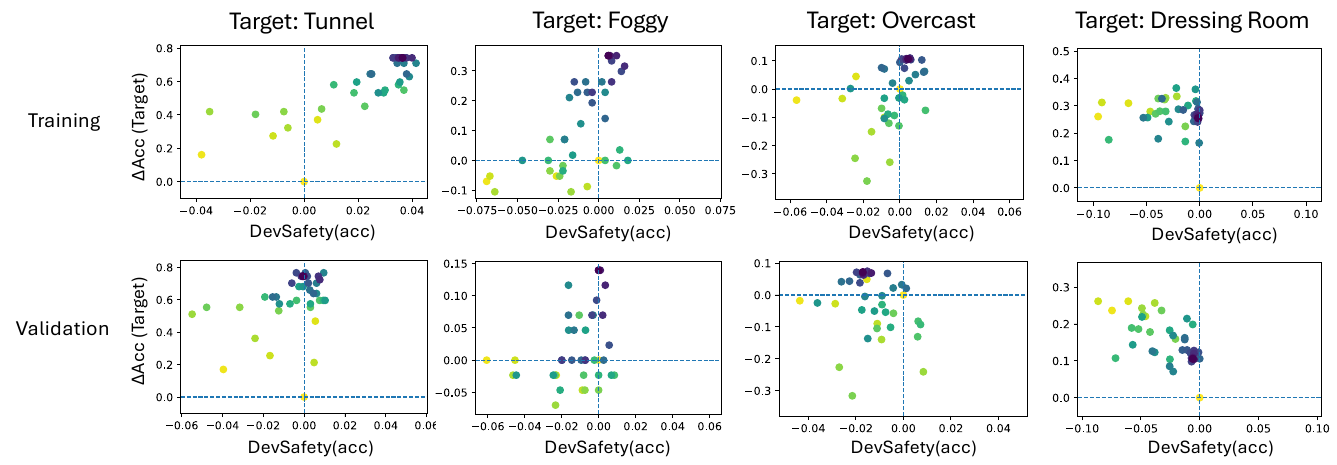}

\centering
\caption{Visualization of the learning trajectory. Each dot denotes a solution with lighter color being earlier iterations and darker being later iterations.}
\label{fig:trajectory}
\end{figure*}

\subsection{Visualization of Learning Process}
\label{sec:visu}
To provide a direct understanding of why and how the proposed algorithm works, we present the learning trajectory of the algorithm in Figure~\ref{fig:trajectory}. {Each dot in this figure represents a solution during the learning processing, with lighter colors indicating earlier stages and darker colors representing later stages. 
From the top four figures for training sets, we can observe a common trend that solutions start from the lower left and move toward the upper right,} indicating the algorithm endeavors to enhance the performance of the targeted task while ensuring developmental safety on protected tasks.  Similarly,  this trend extends to the validation sets, shown in the bottom row, demonstrating the generalization capability of the proposed algorithm. It is striking to see that, when targeting \textit{Dressing Room} in Places365 dataset with all other 364 classes as protected tasks, our method are still able to achieve developmental safety  in training set and generalize to validation set. 
These observations can also be found in separate views of DevSafety vs epochs and $\Delta\text{Acc(Target)}$ vs epochs shown in Figure~\ref{fig:trajectory2}.

\begin{figure*}[tb!]
\centering
    \includegraphics[width=0.33\linewidth]{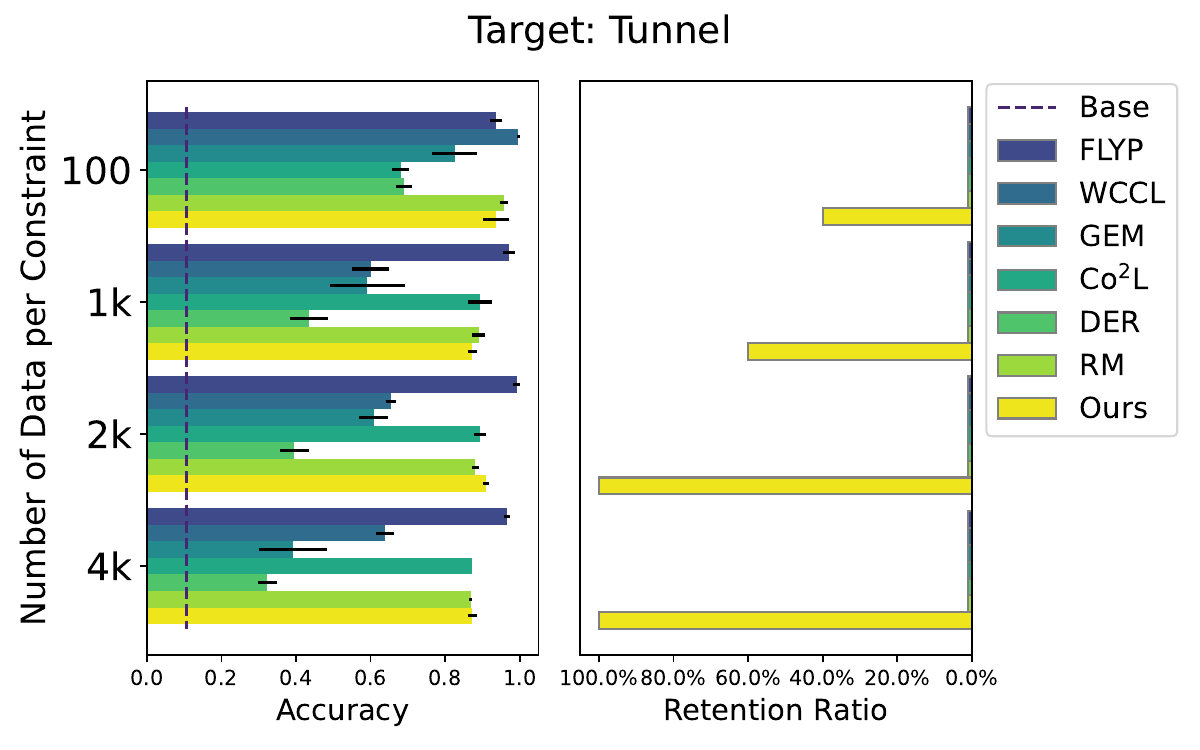}
    \includegraphics[width=0.33\linewidth]{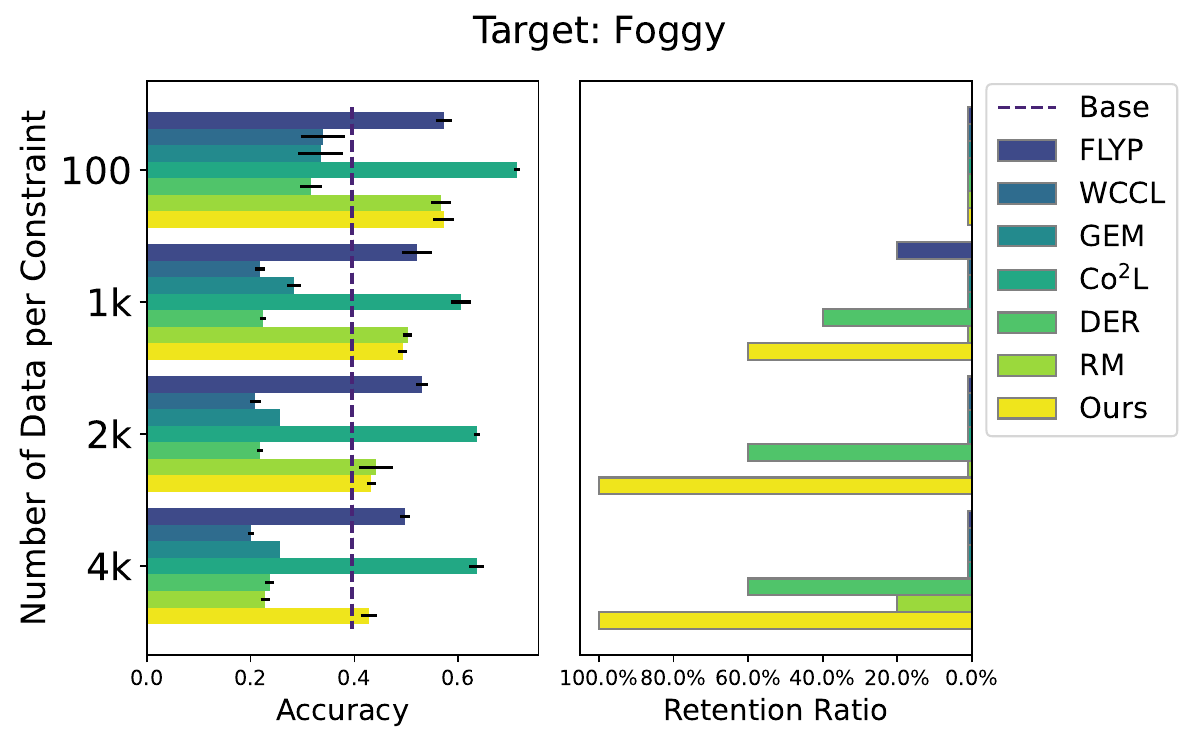}
    \includegraphics[width=0.33\linewidth]{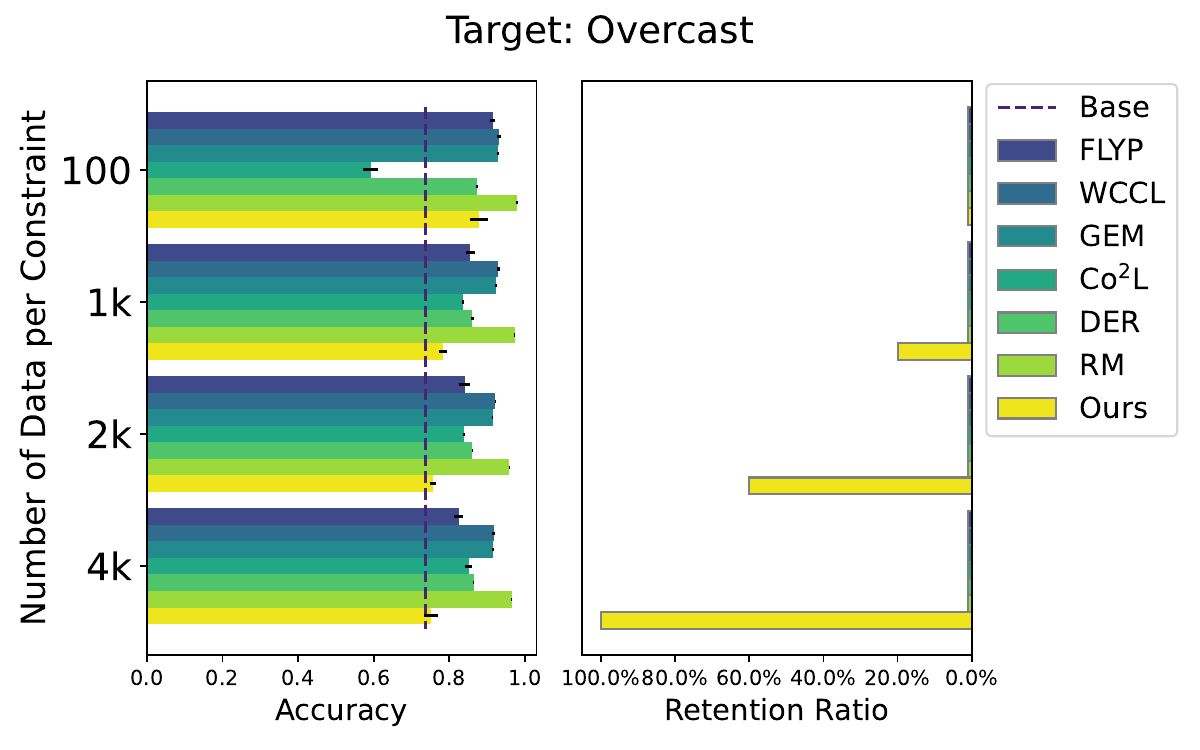} \\
\centering
\caption{Performance Comparison with Baselines. Dot lines represent the performance of the base model on the target task. Detailed numbers are presented in Table~\ref{tab:d_tunnel},~\ref{tab:d_foggy},~\ref{tab:d_overcast}.} 
\label{fig:vsbaseline}
\end{figure*}

\subsection{Comparison with Baselines for Model Developmental Safety}
\label{sec:compare_baseline}

In this part, we compare the proposed method with other baselines to demonstrate the superiority. Specifically, we focus on two metrics, i.e., retention ratio for measuring the possibility of strictly preserving the performance on all protected tasks and accuracy on the target task.  On autonomous driving BDD100K dataset, we conduct experiments with different numbers of data for constraints, i,e., 100, 1k, 2k, 4k from each task.
The comparison results are presented in Figure~\ref{fig:vsbaseline}. The figure illustrates that improving the base model on the target tasks is not challenging, as nearly all methods accomplish this effortlessly. However, all baselines, including the strong continual learning baselines GEM and Co$^2$L, exhibit a zero retention ratio across almost all settings, showing the insufficiency of existing methods for ensuring MDS on protected tasks. Although, when targeting \textit{Foggy}, DER may achieve a retention ratio up to 60\% with more data from protected tasks and large regularization weights, it fails to improve target performance at the same time. In contrast, our method begins to ensure developmental safety with 1k samples per protected task and even 100 samples for the target class \textit{tunnel}. Besides, the retention ratio increases when using more data for constraints, consistent with the result obtained in Lemma~\ref{thm:safety_generalization} (Refer to Table~\ref{tab:res} for more results). Notably, our method achieves a 100\% retention ratio with 4k samples per protected task in all three settings, while improving accuracy on the target class. We also see that the target class \textit{overcast} is most difficult to improve as the base model already has 73.6\% accuracy.

\begin{figure*}[t]
\centering

     \includegraphics[width=0.37\linewidth]{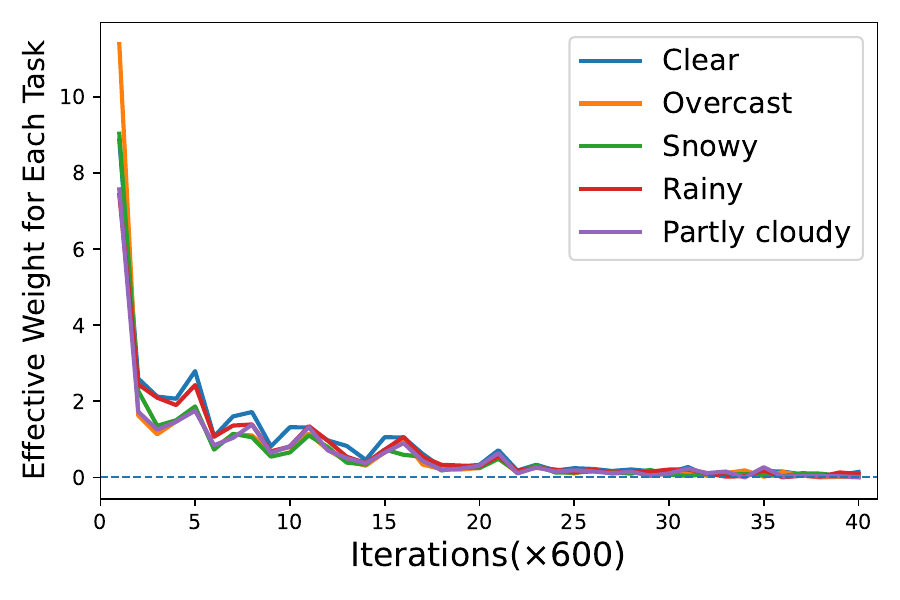}
    \includegraphics[width=0.55\linewidth]{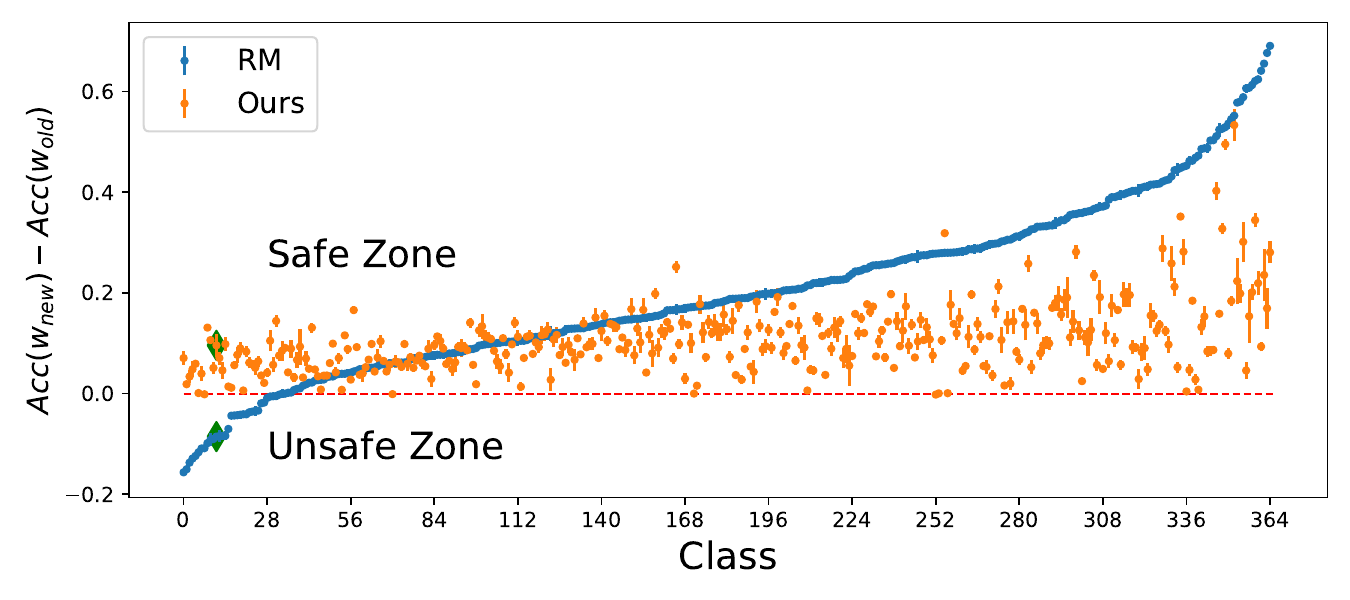}
\centering
\caption{(Left) Adaptive weight adjustments for each protected task during training (Targeting \textit{Foggy}). Weights shown are averaged over every 600 iterations for visualization. (Right) Performance comparison with baseline RM when targeting \textit{Dresssing Room} on Places365 Dataset, with 2k samples per constraint. Red line denotes base model's performance, green diamonds denote the target class. RM baseline shown is for weight $\alpha=10000$ and more plots for other weights are presented in Figure~\ref{fig:places365}.   } 
\label{fig:weight}
\end{figure*}

From Figure~\ref{fig:vsbaseline}, we notice that the baseline RM fails to achieve MDS, even though it has a tunable weight parameter $\alpha$ for protected tasks. Comparison with RM can directly verify the advantage of our method as the only difference between the two methods is how to handle the protected tasks.  From Eqn.~(\ref{eqn:G2}), we can see that our algorithm has an effective weight $\beta[u^t_k]_+$ for each protected task. It is adaptively adjusted during learning, and depends on the degree of violation of constraints, i.e., the larger the violation, the larger the weight. Figure~\ref{fig:weight} (left) shows that these effective weights gradually decrease to zero during the learning of our algorithm, which allows the model to learn from the target task while satisfying constraints.  This mechanism plays a big role in not only achieving MDS but also improving the performance on the target task. In contrast, RM uses a constant weight $\alpha$ for every protected task. Simply increasing $\alpha$ may not ensure MDS, due to varied learning difficulty between protected tasks. Besides, too large $\alpha$ will also harm the performance of the target task.  We further investigate the phenomenon in Appendix~\ref{sec:weight} and find that, with a uniform weight for all the protected tasks, it might preserve previous performance on some of the protected tasks but fail to achieve MDS for all the protected tasks, even with a very high weight $\alpha$. 

To further verify the effectiveness of our method in handling a large number of constraints, we experiment on the Place365 dataset, compared with RM, targeting \textit{Dressing room} class and protecting the other 364 tasks in Figure~\ref{fig:weight} (right). It shows that even with hundreds of protected tasks, our method is still effective in preserving their performance. In contrast, RM even with a large weight $\alpha$ not only causes performance drops in around 30 protected classes failing to ensure MDS but also fails to improve the performance of the target task.  


\subsection{Performance with Multiple Rounds of Model Development}
\label{sec:multiround}
To demonstrate the effectiveness of the proposed retention-centric framework in iterative model development process, we conduct two consecutive rounds of development on recognizing weather conditions on BDD100k data. Specifically, we first target at \textit{overcast} task, taking all the other five weather conditions as protected tasks, then with one selected improved model, we successively improve the model, targeting at improving the performance of the \textit{foggy} task. As shown in Fig.~\ref{fig:multi-round}, our method notably improves the performance of the \textit{overcast} task in the first round while ensuring the performance of other tasks does not decrease. In the second round, it continues to enhance the performance of the \textit{foggy} task. Simultaneously, it preserves the performance, if not boosts it, across other tasks, with only a slight decrease on the \textit{snowy} task, showing the effectiveness of the proposed framework for maintaining the model developmental safety. 

\subsection{Ablation studies}\label{sec:ablation}

\begin{table}[t!]
  \centering
  \caption{Effect of the Number of Samples for Constraints. Numbers in parentheses denote standard deviation.}
    \resizebox{0.98\linewidth}{!}{
    \begin{tabular}{l|l|c|cccc}
    \toprule
    Target&\multicolumn{1}{l|}{Measures} & \multicolumn{1}{l|}{Base model} & 100   & 1k    & 2k    & 4k \\
    \midrule
    &DevSafety(acc)  & 0.00(0.0000) & -0.0050(0.0076) & -0.0001(0.0043) & 0.0105(0.0053) & 0.0186(0.0058) \\
  Tunnel &  Retention Ratio & 100.00\% & 40.00\% & 60.00\% & 100.00\% & 100.00\% \\
    &Target Acc & 0.1064(0.0000) & 0.9362(0.0699) & 0.8723(0.0233) & 0.9106(0.0159) & 0.8723(0.0233) \\
    \midrule
    &DevSafety(acc)  & 0.00(0.0000) & -0.0241(0.0082) & -0.0009(0.0044) & 0.0044(0.0033) & 0.0061(0.0047) \\
     Foggy&Retention Ratio & 100.00\% & 0.00\% & 60.00\% & 100.00\% & 100.00\% \\
    &Target Acc & 0.3953(0.0000) & 0.5721(0.0406) & 0.4930(0.0174) & 0.4326(0.0186) & 0.4279(0.0316) \\
    \midrule
    &DevSafety(acc)  & 0.00(0.0000) & -0.0655(0.0249) & -0.0043(0.0037) & 0.0012(0.0029) & 0.0046(0.0016) \\
   Overcast &Retention Ratio & 100.00\% & 0.00\% & 20.00\% & 60.00\% & 100.00\% \\
    &Target Acc  & 0.7361(0.0000) & 0.8789(0.0464) & 0.7827(0.0225) & 0.7562(0.0167) & 0.7525(0.0366) \\
    \bottomrule
    \end{tabular}%
    }
  \label{tab:res}%
\end{table}%

\subsubsection{Importance of the external data from LAION400M}

We conduct experiments on targeting \textit{foggy} to investigate the benefits of the external data retrieved from LAION400M dataset. In detail, we vary the number of retrieved target-related image-text pairs utilized in the objective function, i.e., \{0, 2k, 5k, 11k\}, with 1k samples from each protected task as constraints. From Tab.~\ref{tab:laion}, we can see that, with only 57 \textit{foggy} samples from BDD100k dataset (i.e., 0 samples from the external data), the model does not improve the target accuracy at all. However, with more and more retrieved image-text pairs utilized to augment the dataset $\cD$, the improvement on the targeted task appears and becomes significant, showing the advantages of incorporating the retrieved target-related image-text pairs for boosting target task accuracy. Regarding retention ratios, we don't observe a clear correlation between the amount of retrieved data and the retention ratios.

\begin{table}[tb]
  \centering
  \caption{The Effect of External Image-text Pairs from LIAON400M. Numbers in parentheses denote std.}
  \resizebox{0.99\linewidth}{!}{
    \begin{tabular}{l|c|cccc}
     & \multicolumn{1}{l|}{Ref(Base model)} & 0     & 2k    & 5k    & 11k \\
    \midrule
    Retention Ratio & 100.00\% & 100.00\% & 80.00\% & 100.00\% & 60.00\% \\
    Target Acc (Foggy) & 0.3953(0.0000) & 0.3674(0.0372) & 0.4047(0.0562) & 0.4186(0.0389) & 0.4930(0.0174) \\
    \end{tabular}%
    }
  \label{tab:laion}%
\end{table}%

\subsubsection{Importance of Task-dependent Heads}\label{sec:important_head}

As introduced in Section~\ref{sec:task-heads}, to reduce the total complexity of our algorithm, we propose task-dependent heads to increase the parameter $\delta$ in Assumption 4, avoiding getting trapped at a flat location where $\w^t$ is infeasible but $\nabla H(\w^t)=0$. To verify the effectiveness of the design, we experiment on targeting the \textit{foggy} task with varying numbers of data for constraints. The results are presented in Figure~\ref{fig:heads}(a). The results show that models equipped with task-dependent heads almost consistently exhibit both higher retention ratio and higher accuracy on the target task. Besides, without task-dependent heads, models may have trouble achieving 100\% developmental safety, demonstrating the effectiveness of task-dependent heads for promoting developmental safety.

\begin{figure*}[t]
\centering
    \subfigure[] {\includegraphics[width=0.48\linewidth]{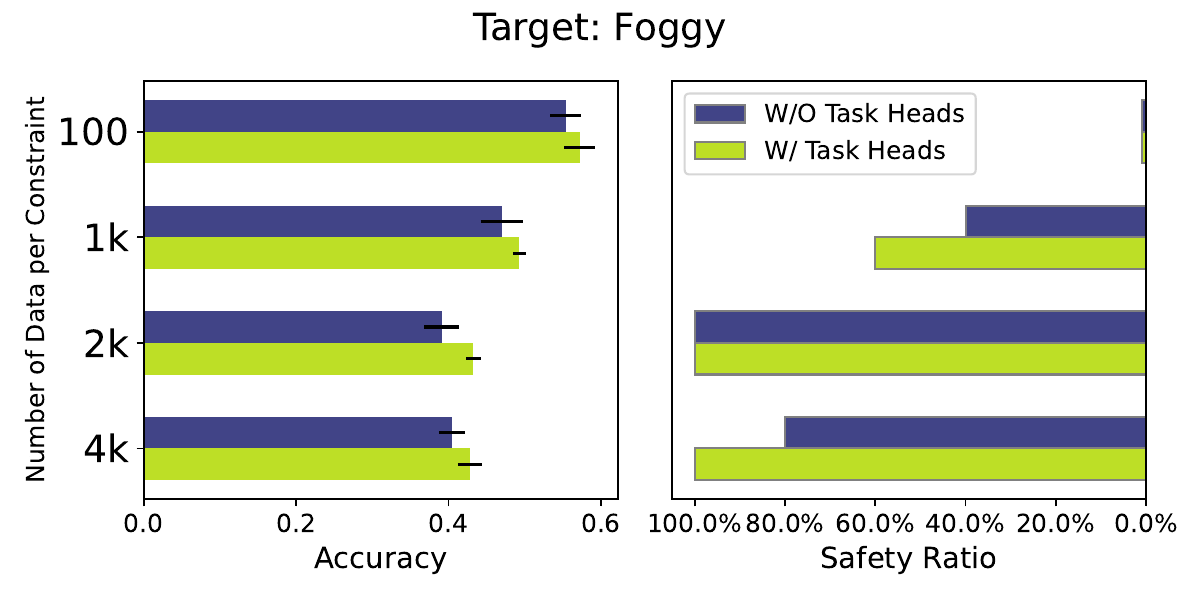}}
    \subfigure[]{\includegraphics[width=0.48\linewidth]{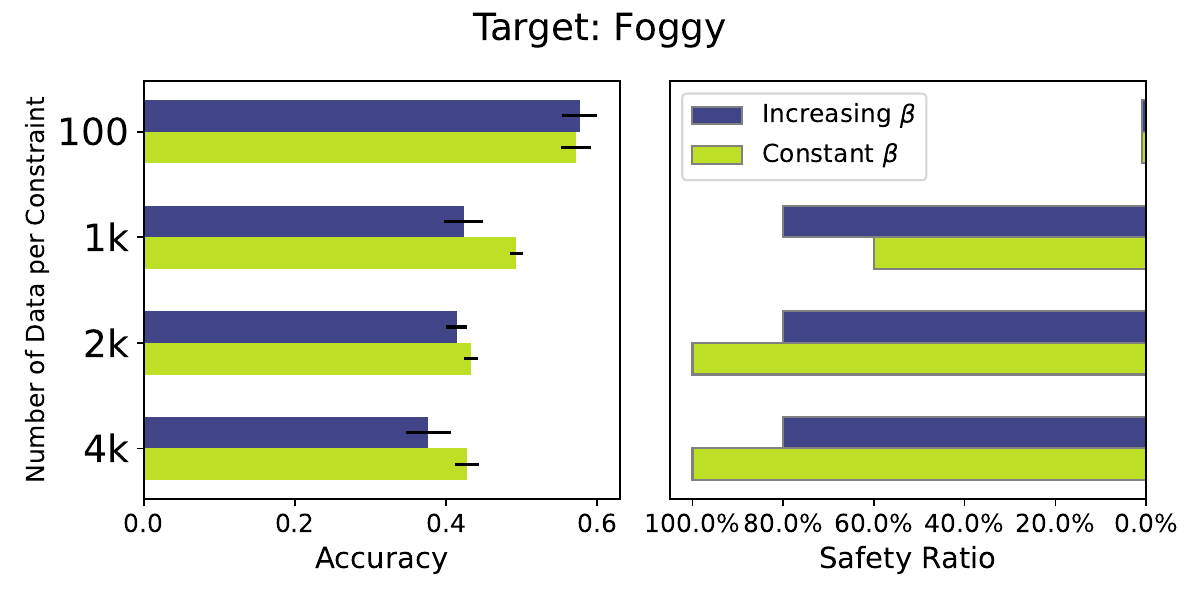}}

\centering
\caption{(a) Task-dependent heads promote developmental safety. (b) Performance Comparison between constant $\beta$ and increasing $\beta$. }
\label{fig:heads}
\end{figure*}


To further verify  the theoretical result in Lemma~\ref{thm:regularity}, we empirically calculate $\nabla \widehat{\textbf{h}}(\hat\w)$ and $ \nabla {\textbf{h}}(\w)$ with CLIP models. Specifically,  we compute the minimal singular values of $\nabla \widehat{\textbf{h}}(\hat\w)$ and $ \nabla {\textbf{h}}(\w)$ on the base model and two trained models, with 1k samples for each protected task. The initial value of $U_k$ is set to zero so $U_kV_k^\top=\mathbf{0}$. From the results presented in Table~\ref{tab:singular}, we can observe that, on the initial model, the minimal singular value of $\nabla \widehat{\textbf{h}}(\hat\w)$ is slightly larger than that of  $ \nabla {\textbf{h}}(\w)$ and the gap become much significant after training, {which is consistent with the theoretical result in Lemma~\ref{thm:regularity} and also provides some insight on the empirical results in Figure~\ref{fig:heads}(a).
}

\begin{table}[t]
  \centering
  \caption{Minimal Singular Values $\delta$ of  $ \nabla {\textbf{h}}(\w)$ and $\nabla \widehat{\textbf{h}}(\hat\w)$ }
    \begin{tabular}{l|l|ll}
         & & Initial Model & Final Model \\
    \midrule
     w/o task-dependent heads&$ \nabla {\textbf{h}}(\w)$ & 23.9183 & 15.0919 \\
     w task-dependent heads& $\nabla \widehat{\textbf{h}}(\hat\w)$ &  24.1038 & 16.3397 \\
     \bottomrule
    \end{tabular}%
  \label{tab:singular}%
\end{table}%

\subsubsection{Constant $\beta$ vs Increasing $\beta$} 

In theory, an increasing penalty parameter $\beta$ may help reduce the complexity of constrained problems as shown in \cite{alacaoglu2024complexity},  but in our empirical experiments, we find that using a constant $\beta$ is generally behave better than using an increasing $\beta$ . As shown in Fig.~\ref{fig:heads}(b) for target task \textit{foggy}, models with a constant $\beta$ are able to achieve 100\% retention ratio with 2k or 4k sampler per constraint. On the contrary, models using a cosine increasing $\beta$ obtain both lower retention ratio and lower accuracy on the target task compared to models with constant $\beta$. We conjecture that this is because models with an increasing $\beta$ might leave the developmental safety region too far in the initial stages as they have a relatively small penalty weight $\beta$ at this time. Given the high non-convexity and complexity of the model space, it becomes increasingly challenging in the later stages to return to a feasible solution that satisfies developmental safety constraints while significantly improving target accuracy.

\section{Conclusion}
In this paper, we introduced the concept of "model developmental safety" to ensure that model development not only acquires new capabilities but also strictly preserves those already owns, addressing the critical developmental safety oversight in existing ML/AI studies. To ensure model developmental safety, we proposed a {retention}-centric framework by formulating the model developmental safety as data-dependent constraints. We proposed an efficient constrained optimization algorithm with theoretical guarantees to develop a pretrained vision-language model (CLIP model) for improving existing image classification capabilities. Comprehensive experiments demonstrate the effectiveness of the algorithm in enhancing vision-based perception capabilities in autonomous driving and scene recognition, showing its practical value in real-world scenarios.  {As the proposed framework in this paper is a generic retention-centric optimization framework, it can be potentially extended to various scenarios or models, such as finetuning LLMs or enhancing object detection systems and motion prediction tasks for autonomous driving. We hope our work can inspire researchers in safety-critical application domain for more exploration.}

\section*{Acknowledgements}

This research is partially supported by the National Artificial Intelligence Research Resource (NAIRR) Pilot and the Delta advanced computing and data resource which is supported by the National Science Foundation (award NSF-OAC 2005572) and the State of Illinois. Delta is a joint effort of the University of Illinois Urbana-Champaign and its National Center for Supercomputing Applications.



\bibliography{ref}

\begin{thebibliography}{10}

\bibitem{achiam2017constrained}
Joshua Achiam, David Held, Aviv Tamar, and Pieter Abbeel.
\newblock Constrained policy optimization.
\newblock In {\em International conference on machine learning}, pages 22--31. PMLR, 2017.

\bibitem{DBLP:journals/corr/abs-1803-02453}
Alekh Agarwal, Alina Beygelzimer, Miroslav Dud{\'{\i}}k, John Langford, and Hanna~M. Wallach.
\newblock A reductions approach to fair classification.
\newblock {\em CoRR}, abs/1803.02453, 2018.

\bibitem{alacaoglu2024complexity}
Ahmet Alacaoglu and Stephen~J Wright.
\newblock Complexity of single loop algorithms for nonlinear programming with stochastic objective and constraints.
\newblock In {\em International Conference on Artificial Intelligence and Statistics}, pages 4627--4635. PMLR, 2024.

\bibitem{DBLP:journals/corr/abs-1812-03596}
Rahaf Aljundi, Klaas Kelchtermans, and Tinne Tuytelaars.
\newblock Task-free continual learning.
\newblock {\em CoRR}, abs/1812.03596, 2018.

\bibitem{DBLP:journals/corr/AmodeiOSCSM16}
Dario Amodei, Chris Olah, Jacob Steinhardt, Paul~F. Christiano, John Schulman, and Dan Man{\'{e}}.
\newblock Concrete problems in {AI} safety.
\newblock {\em CoRR}, abs/1606.06565, 2016.

\bibitem{aravkin2019level}
Aleksandr~Y Aravkin, James~V Burke, Dmitry Drusvyatskiy, Michael~P Friedlander, and Scott Roy.
\newblock Level-set methods for convex optimization.
\newblock {\em Mathematical Programming}, 174:359--390, 2019.

\bibitem{bartlett2002rademacher}
Peter~L Bartlett and Shahar Mendelson.
\newblock Rademacher and gaussian complexities: Risk bounds and structural results.
\newblock {\em Journal of Machine Learning Research}, 3(Nov):463--482, 2002.

\bibitem{bertsekas2014constrained}
Dimitri~P Bertsekas.
\newblock {\em Constrained optimization and Lagrange multiplier methods}.
\newblock Academic press, 2014.

\bibitem{boob2023stochastic}
Digvijay Boob, Qi~Deng, and Guanghui Lan.
\newblock Stochastic first-order methods for convex and nonconvex functional constrained optimization.
\newblock {\em Mathematical Programming}, 197(1):215--279, 2023.

\bibitem{boucheron2005theory}
St{\'e}phane Boucheron, Olivier Bousquet, and G{\'a}bor Lugosi.
\newblock Theory of classification: A survey of some recent advances.
\newblock {\em ESAIM: probability and statistics}, 9:323--375, 2005.

\bibitem{bubeck2023sparks}
S{\'e}bastien Bubeck, Varun Chandrasekaran, Ronen Eldan, Johannes Gehrke, Eric Horvitz, Ece Kamar, Peter Lee, Yin~Tat Lee, Yuanzhi Li, Scott Lundberg, et~al.
\newblock Sparks of artificial general intelligence: Early experiments with gpt-4.
\newblock {\em arXiv preprint arXiv:2303.12712}, 2023.

\bibitem{bura2022dope}
Archana Bura, Aria HasanzadeZonuzy, Dileep Kalathil, Srinivas Shakkottai, and Jean-Francois Chamberland.
\newblock Dope: Doubly optimistic and pessimistic exploration for safe reinforcement learning.
\newblock {\em Advances in neural information processing systems}, 35:1047--1059, 2022.

\bibitem{buzzega2020dark}
Pietro Buzzega, Matteo Boschini, Angelo Porrello, Davide Abati, and Simone Calderara.
\newblock Dark experience for general continual learning: a strong, simple baseline.
\newblock {\em Advances in neural information processing systems}, 33:15920--15930, 2020.

\bibitem{castro2018end}
Francisco~M Castro, Manuel~J Mar{\'\i}n-Jim{\'e}nez, Nicol{\'a}s Guil, Cordelia Schmid, and Karteek Alahari.
\newblock End-to-end incremental learning.
\newblock In {\em Proceedings of the European conference on computer vision (ECCV)}, pages 233--248, 2018.

\bibitem{cha2021co2l}
Hyuntak Cha, Jaeho Lee, and Jinwoo Shin.
\newblock Co2l: Contrastive continual learning.
\newblock In {\em Proceedings of the IEEE/CVF International conference on computer vision}, pages 9516--9525, 2021.

\bibitem{chamon2022constrained}
Luiz~FO Chamon, Santiago Paternain, Miguel Calvo-Fullana, and Alejandro Ribeiro.
\newblock Constrained learning with non-convex losses.
\newblock {\em IEEE Transactions on Information Theory}, 69(3):1739--1760, 2022.

\bibitem{chaudhry2018efficient}
Arslan Chaudhry, Marc’Aurelio Ranzato, Marcus Rohrbach, and Mohamed Elhoseiny.
\newblock Efficient lifelong learning with a-{GEM}.
\newblock In {\em International Conference on Learning Representations}, 2019.

\bibitem{chen2013neil}
Xinlei Chen, Abhinav Shrivastava, and Abhinav Gupta.
\newblock Neil: Extracting visual knowledge from web data.
\newblock In {\em Proceedings of the IEEE international conference on computer vision}, pages 1409--1416, 2013.

\bibitem{McKinseyCompany}
McKinsey~\& Company.
\newblock The future of autonomous driving.
\newblock \url{https://www.mckinsey.com/industries/automotive-and-assembly/our-insights/autonomous-drivings-future-convenient-and-connected}, 2023.

\bibitem{cotter2019optimization}
Andrew Cotter, Heinrich Jiang, Maya Gupta, Serena Wang, Taman Narayan, Seungil You, and Karthik Sridharan.
\newblock Optimization with non-differentiable constraints with applications to fairness, recall, churn, and other goals.
\newblock {\em Journal of Machine Learning Research}, 20(172):1--59, 2019.

\bibitem{ding2021provably}
Dongsheng Ding, Xiaohan Wei, Zhuoran Yang, Zhaoran Wang, and Mihailo Jovanovic.
\newblock Provably efficient safe exploration via primal-dual policy optimization.
\newblock In {\em International conference on artificial intelligence and statistics}, pages 3304--3312. PMLR, 2021.

\bibitem{facchinei2021ghost}
Francisco Facchinei, Vyacheslav Kungurtsev, Lorenzo Lampariello, and Gesualdo Scutari.
\newblock Ghost penalties in nonconvex constrained optimization: Diminishing stepsizes and iteration complexity.
\newblock {\em Mathematics of Operations Research}, 46(2):595--627, 2021.

\bibitem{DBLP:journals/corr/abs-1910-07104}
Mehrdad Farajtabar, Navid Azizan, Alex Mott, and Ang Li.
\newblock Orthogonal gradient descent for continual learning.
\newblock {\em CoRR}, abs/1910.07104, 2019.

\bibitem{goyal2023finetune}
Sachin Goyal, Ananya Kumar, Sankalp Garg, Zico Kolter, and Aditi Raghunathan.
\newblock Finetune like you pretrain: Improved finetuning of zero-shot vision models.
\newblock In {\em Proceedings of the IEEE/CVF Conference on Computer Vision and Pattern Recognition}, pages 19338--19347, 2023.

\bibitem{DBLP:conf/nips/GuoLYR20}
Yunhui Guo, Mingrui Liu, Tianbao Yang, and Tajana Rosing.
\newblock Improved schemes for episodic memory-based lifelong learning.
\newblock In Hugo Larochelle, Marc'Aurelio Ranzato, Raia Hadsell, Maria{-}Florina Balcan, and Hsuan{-}Tien Lin, editors, {\em Advances in Neural Information Processing Systems 33: Annual Conference on Neural Information Processing Systems 2020, NeurIPS 2020, December 6-12, 2020, virtual}, 2020.

\bibitem{hamedani2021primal}
Erfan~Yazdandoost Hamedani and Necdet~Serhat Aybat.
\newblock A primal-dual algorithm with line search for general convex-concave saddle point problems.
\newblock {\em SIAM Journal on Optimization}, 31(2):1299--1329, 2021.

\bibitem{kakade2008complexity}
Sham~M Kakade, Karthik Sridharan, and Ambuj Tewari.
\newblock On the complexity of linear prediction: Risk bounds, margin bounds, and regularization.
\newblock {\em Advances in neural information processing systems}, 21, 2008.

\bibitem{doi:10.1073/pnas.1611835114}
James Kirkpatrick, Razvan Pascanu, Neil Rabinowitz, Joel Veness, Guillaume Desjardins, Andrei~A. Rusu, Kieran Milan, John Quan, Tiago Ramalho, Agnieszka Grabska-Barwinska, Demis Hassabis, Claudia Clopath, Dharshan Kumaran, and Raia Hadsell.
\newblock Overcoming catastrophic forgetting in neural networks.
\newblock {\em Proceedings of the National Academy of Sciences}, 114(13):3521--3526, 2017.

\bibitem{kojima2022large}
Takeshi Kojima, Shixiang~Shane Gu, Machel Reid, Yutaka Matsuo, and Yusuke Iwasawa.
\newblock Large language models are zero-shot reasoners.
\newblock {\em Advances in neural information processing systems}, 35:22199--22213, 2022.

\bibitem{koopman2016challenges}
Philip Koopman and Michael Wagner.
\newblock Challenges in autonomous vehicle testing and validation.
\newblock {\em SAE International Journal of Transportation Safety}, 4(1):15--24, 2016.

\bibitem{lan2013iteration}
Guanghui Lan and Renato~DC Monteiro.
\newblock Iteration-complexity of first-order penalty methods for convex programming.
\newblock {\em Mathematical Programming}, 138(1):115--139, 2013.

\bibitem{lan2016iteration}
Guanghui Lan and Renato~DC Monteiro.
\newblock Iteration-complexity of first-order augmented lagrangian methods for convex programming.
\newblock {\em Mathematical Programming}, 155(1):511--547, 2016.

\bibitem{lan2016algorithms}
Guanghui Lan and Zhiqiang Zhou.
\newblock Algorithms for stochastic optimization with expectation constraints.
\newblock {\em arXiv preprint arXiv:1604.03887}, 2016.

\bibitem{10.5555/3294996.3295218}
Sang-Woo Lee, Jin-Hwa Kim, Jaehyun Jun, Jung-Woo Ha, and Byoung-Tak Zhang.
\newblock Overcoming catastrophic forgetting by incremental moment matching.
\newblock In {\em Proceedings of the 31st International Conference on Neural Information Processing Systems}, NIPS'17, page 4655–4665, Red Hook, NY, USA, 2017. Curran Associates Inc.

\bibitem{Li2019LearnTG}
Xilai Li, Yingbo Zhou, Tianfu Wu, Richard Socher, and Caiming Xiong.
\newblock Learn to grow: A continual structure learning framework for overcoming catastrophic forgetting.
\newblock In {\em International Conference on Machine Learning}, 2019.

\bibitem{DBLP:journals/corr/LiH16e}
Zhizhong Li and Derek Hoiem.
\newblock Learning without forgetting.
\newblock {\em CoRR}, abs/1606.09282, 2016.

\bibitem{li2024stochastic}
Zichong Li, Pin-Yu Chen, Sijia Liu, Songtao Lu, and Yangyang Xu.
\newblock Stochastic inexact augmented lagrangian method for nonconvex expectation constrained optimization.
\newblock {\em Computational Optimization and Applications}, 87(1):117--147, 2024.

\bibitem{liang2024aide}
Mingfu Liang, Jong-Chyi Su, Samuel Schulter, Sparsh Garg, Shiyu Zhao, Ying Wu, and Manmohan Chandraker.
\newblock Aide: An automatic data engine for object detection in autonomous driving.
\newblock In {\em Proceedings of the IEEE/CVF Conference on Computer Vision and Pattern Recognition}, pages 14695--14706, 2024.

\bibitem{lin2022complexity}
Qihang Lin, Runchao Ma, and Yangyang Xu.
\newblock Complexity of an inexact proximal-point penalty method for constrained smooth non-convex optimization.
\newblock {\em Computational optimization and applications}, 82(1):175--224, 2022.

\bibitem{lin2018level}
Qihang Lin, Runchao Ma, and Tianbao Yang.
\newblock Level-set methods for finite-sum constrained convex optimization.
\newblock In {\em International conference on machine learning}, pages 3112--3121. PMLR, 2018.

\bibitem{lin2018level2}
Qihang Lin, Selvaprabu Nadarajah, and Negar Soheili.
\newblock A level-set method for convex optimization with a feasible solution path.
\newblock {\em SIAM Journal on Optimization}, 28(4):3290--3311, 2018.

\bibitem{10.5555/3295222.3295393}
David Lopez-Paz and Marc'Aurelio Ranzato.
\newblock Gradient episodic memory for continual learning.
\newblock In {\em Proceedings of the 31st International Conference on Neural Information Processing Systems}, NIPS'17, page 6470–6479, Red Hook, NY, USA, 2017. Curran Associates Inc.

\bibitem{lopez2017gradient}
David Lopez-Paz and Marc'Aurelio Ranzato.
\newblock Gradient episodic memory for continual learning.
\newblock {\em Advances in neural information processing systems}, 30, 2017.

\bibitem{ma2020quadratically}
Runchao Ma, Qihang Lin, and Tianbao Yang.
\newblock Quadratically regularized subgradient methods for weakly convex optimization with weakly convex constraints.
\newblock In {\em International Conference on Machine Learning}, pages 6554--6564. PMLR, 2020.

\bibitem{madry2017towards}
Aleksander Madry, Aleksandar Makelov, Ludwig Schmidt, Dimitris Tsipras, and Adrian Vladu.
\newblock Towards deep learning models resistant to adversarial attacks.
\newblock {\em arXiv preprint arXiv:1706.06083}, 2017.

\bibitem{mccloskey1989catastrophic}
Michael McCloskey and Neal~J Cohen.
\newblock Catastrophic interference in connectionist networks: The sequential learning problem.
\newblock In {\em Psychology of learning and motivation}, volume~24, pages 109--165. Elsevier, 1989.

\bibitem{mitchell2018never}
Tom Mitchell, William Cohen, Estevam Hruschka, Partha Talukdar, Bishan Yang, Justin Betteridge, Andrew Carlson, Bhavana Dalvi, Matt Gardner, Bryan Kisiel, et~al.
\newblock Never-ending learning.
\newblock {\em Communications of the ACM}, 61(5):103--115, 2018.

\bibitem{nemirovski2004prox}
Arkadi Nemirovski.
\newblock Prox-method with rate of convergence o (1/t) for variational inequalities with lipschitz continuous monotone operators and smooth convex-concave saddle point problems.
\newblock {\em SIAM Journal on Optimization}, 15(1):229--251, 2004.

\bibitem{ono2015chance}
Masahiro Ono, Marco Pavone, Yoshiaki Kuwata, and J~Balaram.
\newblock Chance-constrained dynamic programming with application to risk-aware robotic space exploration.
\newblock {\em Autonomous Robots}, 39:555--571, 2015.

\bibitem{parashar2024neglected}
Shubham Parashar, Zhiqiu Lin, Tian Liu, Xiangjue Dong, Yanan Li, Deva Ramanan, James Caverlee, and Shu Kong.
\newblock The neglected tails of vision-language models, 2024.

\bibitem{DBLP:journals/corr/abs-1802-07569}
German~Ignacio Parisi, Ronald Kemker, Jose~L. Part, Christopher Kanan, and Stefan Wermter.
\newblock Continual lifelong learning with neural networks: {A} review.
\newblock {\em CoRR}, abs/1802.07569, 2018.

\bibitem{paternain2019learning}
Santiago Paternain, Miguel Calvo-Fullana, Luiz~FO Chamon, and Alejandro Ribeiro.
\newblock Learning safe policies via primal-dual methods.
\newblock In {\em 2019 IEEE 58th Conference on Decision and Control (CDC)}, pages 6491--6497. IEEE, 2019.

\bibitem{paternain2019constrained}
Santiago Paternain, Luiz Chamon, Miguel Calvo-Fullana, and Alejandro Ribeiro.
\newblock Constrained reinforcement learning has zero duality gap.
\newblock {\em Advances in Neural Information Processing Systems}, 32, 2019.

\bibitem{peng2023ideal}
Liangzu Peng, Paris Giampouras, and Ren{\'e} Vidal.
\newblock The ideal continual learner: An agent that never forgets.
\newblock In {\em International Conference on Machine Learning}, pages 27585--27610. PMLR, 2023.

\bibitem{pfrommer2022safe}
Samuel Pfrommer, Tanmay Gautam, Alec Zhou, and Somayeh Sojoudi.
\newblock Safe reinforcement learning with chance-constrained model predictive control.
\newblock In {\em Learning for Dynamics and Control Conference}, pages 291--303. PMLR, 2022.

\bibitem{pham2018optlayer}
Tu-Hoa Pham, Giovanni De~Magistris, and Ryuki Tachibana.
\newblock Optlayer-practical constrained optimization for deep reinforcement learning in the real world.
\newblock In {\em 2018 IEEE International Conference on Robotics and Automation (ICRA)}, pages 6236--6243. IEEE, 2018.

\bibitem{polyak1973method}
VT~Polyak and NV~Tret'yakov.
\newblock The method of penalty estimates for conditional extremum problems.
\newblock {\em USSR Computational Mathematics and Mathematical Physics}, 13(1):42--58, 1973.

\bibitem{DBLP:conf/icml/QiuHYZ0Y23}
Zi{-}Hao Qiu, Quanqi Hu, Zhuoning Yuan, Denny Zhou, Lijun Zhang, and Tianbao Yang.
\newblock Not all semantics are created equal: Contrastive self-supervised learning with automatic temperature individualization.
\newblock In Andreas Krause, Emma Brunskill, Kyunghyun Cho, Barbara Engelhardt, Sivan Sabato, and Jonathan Scarlett, editors, {\em International Conference on Machine Learning, {ICML} 2023, 23-29 July 2023, Honolulu, Hawaii, {USA}}, volume 202 of {\em Proceedings of Machine Learning Research}, pages 28389--28421. {PMLR}, 2023.

\bibitem{qu2021recent}
Haoxuan Qu, Hossein Rahmani, Li~Xu, Bryan Williams, and Jun Liu.
\newblock Recent advances of continual learning in computer vision: An overview.
\newblock {\em arXiv preprint arXiv:2109.11369}, 2021.

\bibitem{radford2021learning}
Alec Radford, Jong~Wook Kim, Chris Hallacy, Aditya Ramesh, Gabriel Goh, Sandhini Agarwal, Girish Sastry, Amanda Askell, Pamela Mishkin, Jack Clark, et~al.
\newblock Learning transferable visual models from natural language supervision.
\newblock In {\em International conference on machine learning}, pages 8748--8763. PMLR, 2021.

\bibitem{DBLP:conf/icml/RadfordKHRGASAM21}
Alec Radford, Jong~Wook Kim, Chris Hallacy, Aditya Ramesh, Gabriel Goh, Sandhini Agarwal, Girish Sastry, Amanda Askell, Pamela Mishkin, Jack Clark, Gretchen Krueger, and Ilya Sutskever.
\newblock Learning transferable visual models from natural language supervision.
\newblock In Marina Meila and Tong Zhang, editors, {\em Proceedings of the 38th International Conference on Machine Learning, {ICML} 2021, 18-24 July 2021, Virtual Event}, volume 139 of {\em Proceedings of Machine Learning Research}, pages 8748--8763. {PMLR}, 2021.

\bibitem{rajabli2020software}
Nijat Rajabli, Francesco Flammini, Roberto Nardone, and Valeria Vittorini.
\newblock Software verification and validation of safe autonomous cars: A systematic literature review.
\newblock {\em IEEE Access}, 9:4797--4819, 2020.

\bibitem{rebuffi2017icarl}
Sylvestre-Alvise Rebuffi, Alexander Kolesnikov, Georg Sperl, and Christoph~H Lampert.
\newblock icarl: Incremental classifier and representation learning.
\newblock In {\em Proceedings of the IEEE conference on Computer Vision and Pattern Recognition}, pages 2001--2010, 2017.

\bibitem{robey2021adversarial}
Alexander Robey, Luiz Chamon, George~J Pappas, Hamed Hassani, and Alejandro Ribeiro.
\newblock Adversarial robustness with semi-infinite constrained learning.
\newblock {\em Advances in Neural Information Processing Systems}, 34:6198--6215, 2021.

\bibitem{NEURIPS2019_fa7cdfad}
David Rolnick, Arun Ahuja, Jonathan Schwarz, Timothy Lillicrap, and Gregory Wayne.
\newblock Experience replay for continual learning.
\newblock In H.~Wallach, H.~Larochelle, A.~Beygelzimer, F.~d\textquotesingle Alch\'{e}-Buc, E.~Fox, and R.~Garnett, editors, {\em Advances in Neural Information Processing Systems}, volume~32. Curran Associates, Inc., 2019.

\bibitem{DBLP:journals/corr/RusuRDSKKPH16}
Andrei~A. Rusu, Neil~C. Rabinowitz, Guillaume Desjardins, Hubert Soyer, James Kirkpatrick, Koray Kavukcuoglu, Razvan Pascanu, and Raia Hadsell.
\newblock Progressive neural networks.
\newblock {\em CoRR}, abs/1606.04671, 2016.

\bibitem{sahin2019inexact}
Mehmet~Fatih Sahin, Ahmet Alacaoglu, Fabian Latorre, Volkan Cevher, et~al.
\newblock An inexact augmented lagrangian framework for nonconvex optimization with nonlinear constraints.
\newblock {\em Advances in Neural Information Processing Systems}, 32, 2019.

\bibitem{schuhmann2021laion}
Christoph Schuhmann, Richard Vencu, Romain Beaumont, Robert Kaczmarczyk, Clayton Mullis, Aarush Katta, Theo Coombes, Jenia Jitsev, and Aran Komatsuzaki.
\newblock Laion-400m: Open dataset of clip-filtered 400 million image-text pairs.
\newblock {\em arXiv preprint arXiv:2111.02114}, 2021.

\bibitem{seita2018bdd100k}
Daniel Seita.
\newblock Bdd100k: A large-scale diverse driving video database.
\newblock {\em The Berkeley Artificial Intelligence Research Blog. Version}, 511:41, 2018.

\bibitem{shalev2016safe}
Shai Shalev-Shwartz, Shaked Shammah, and Amnon Shashua.
\newblock Safe, multi-agent, reinforcement learning for autonomous driving.
\newblock {\em arXiv preprint arXiv:1610.03295}, 2016.

\bibitem{NIPS2017_0efbe980}
Hanul Shin, Jung~Kwon Lee, Jaehong Kim, and Jiwon Kim.
\newblock Continual learning with deep generative replay.
\newblock In I.~Guyon, U.~Von Luxburg, S.~Bengio, H.~Wallach, R.~Fergus, S.~Vishwanathan, and R.~Garnett, editors, {\em Advances in Neural Information Processing Systems}, volume~30. Curran Associates, Inc., 2017.

\bibitem{tessler2018reward}
Chen Tessler, Daniel~J Mankowitz, and Shie Mannor.
\newblock Reward constrained policy optimization.
\newblock {\em arXiv preprint arXiv:1805.11074}, 2018.

\bibitem{thananjeyan2021recovery}
Brijen Thananjeyan, Ashwin Balakrishna, Suraj Nair, Michael Luo, Krishnan Srinivasan, Minho Hwang, Joseph~E Gonzalez, Julian Ibarz, Chelsea Finn, and Ken Goldberg.
\newblock Recovery rl: Safe reinforcement learning with learned recovery zones.
\newblock {\em IEEE Robotics and Automation Letters}, 6(3):4915--4922, 2021.

\bibitem{thomas2021safe}
Garrett Thomas, Yuping Luo, and Tengyu Ma.
\newblock Safe reinforcement learning by imagining the near future.
\newblock {\em Advances in Neural Information Processing Systems}, 34:13859--13869, 2021.

\bibitem{turchetta2020safe}
Matteo Turchetta, Andrey Kolobov, Shital Shah, Andreas Krause, and Alekh Agarwal.
\newblock Safe reinforcement learning via curriculum induction.
\newblock {\em Advances in Neural Information Processing Systems}, 33:12151--12162, 2020.

\bibitem{wachi2024survey}
Akifumi Wachi, Xun Shen, and Yanan Sui.
\newblock A survey of constraint formulations in safe reinforcement learning, 2024.

\bibitem{DBLP:conf/icml/WangY22}
Bokun Wang and Tianbao Yang.
\newblock Finite-sum coupled compositional stochastic optimization: Theory and applications.
\newblock In Kamalika Chaudhuri, Stefanie Jegelka, Le~Song, Csaba Szepesv{\'{a}}ri, Gang Niu, and Sivan Sabato, editors, {\em International Conference on Machine Learning, {ICML} 2022, 17-23 July 2022, Baltimore, Maryland, {USA}}, volume 162 of {\em Proceedings of Machine Learning Research}, pages 23292--23317. {PMLR}, 2022.

\bibitem{wang2023decodingtrust}
Boxin Wang, Weixin Chen, Hengzhi Pei, Chulin Xie, Mintong Kang, Chenhui Zhang, Chejian Xu, Zidi Xiong, Ritik Dutta, Rylan Schaeffer, et~al.
\newblock Decodingtrust: A comprehensive assessment of trustworthiness in gpt models.
\newblock In {\em NeurIPS}, 2023.

\bibitem{wang2024comprehensive}
Liyuan Wang, Xingxing Zhang, Hang Su, and Jun Zhu.
\newblock A comprehensive survey of continual learning: theory, method and application.
\newblock {\em IEEE Transactions on Pattern Analysis and Machine Intelligence}, 2024.

\bibitem{wang2023enforcing}
Yixuan Wang, Simon~Sinong Zhan, Ruochen Jiao, Zhilu Wang, Wanxin Jin, Zhuoran Yang, Zhaoran Wang, Chao Huang, and Qi~Zhu.
\newblock Enforcing hard constraints with soft barriers: Safe reinforcement learning in unknown stochastic environments.
\newblock In {\em International Conference on Machine Learning}, pages 36593--36604. PMLR, 2023.

\bibitem{wei2022emergent}
Jason Wei, Yi~Tay, Rishi Bommasani, Colin Raffel, Barret Zoph, Sebastian Borgeaud, Dani Yogatama, Maarten Bosma, Denny Zhou, Donald Metzler, et~al.
\newblock Emergent abilities of large language models.
\newblock {\em arXiv preprint arXiv:2206.07682}, 2022.

\bibitem{xie2021complexity}
Yue Xie and Stephen~J Wright.
\newblock Complexity of proximal augmented lagrangian for nonconvex optimization with nonlinear equality constraints.
\newblock {\em Journal of Scientific Computing}, 86:1--30, 2021.

\bibitem{xu2021iteration}
Yangyang Xu.
\newblock Iteration complexity of inexact augmented lagrangian methods for constrained convex programming.
\newblock {\em Mathematical Programming}, 185:199--244, 2021.

\bibitem{yan2021dynamically}
Shipeng Yan, Jiangwei Xie, and Xuming He.
\newblock Der: Dynamically expandable representation for class incremental learning.
\newblock In {\em Proceedings of the IEEE/CVF conference on computer vision and pattern recognition}, pages 3014--3023, 2021.

\bibitem{yuan2022provable}
Zhuoning Yuan, Yuexin Wu, Zi-Hao Qiu, Xianzhi Du, Lijun Zhang, Denny Zhou, and Tianbao Yang.
\newblock Provable stochastic optimization for global contrastive learning: Small batch does not harm performance.
\newblock In {\em International Conference on Machine Learning}, pages 25760--25782. PMLR, 2022.

\bibitem{10.5555/3305890.3306093}
Friedemann Zenke, Ben Poole, and Surya Ganguli.
\newblock Continual learning through synaptic intelligence.
\newblock In {\em Proceedings of the 34th International Conference on Machine Learning - Volume 70}, ICML'17, page 3987–3995. JMLR.org, 2017.

\bibitem{zhou2017places}
Bolei Zhou, Agata Lapedriza, Aditya Khosla, Aude Oliva, and Antonio Torralba.
\newblock Places: A 10 million image database for scene recognition.
\newblock {\em IEEE Transactions on Pattern Analysis and Machine Intelligence}, 2017.

\bibitem{zhou2022model}
Da-Wei Zhou, Qi-Wei Wang, Han-Jia Ye, and De-Chuan Zhan.
\newblock A model or 603 exemplars: Towards memory-efficient class-incremental learning.
\newblock {\em arXiv preprint arXiv:2205.13218}, 2022.

\end{thebibliography}
\bibliographystyle{plain}

\newpage
\appendix

\section{More Experimental Details and Results}

All experiments in our paper are run on two High Performance Research Computing platforms. One contains 117 GPU nodes, each with two A100 40GB GPUs. Another contains 200 GPU nodes, each with four A100 40GB GPUs or four A40 48GB GPUs.

We present the statistics for BDD100K Dataset in our experiments in Table~\ref{tab:bdd100k}.

\begin{table}[htbp]
  \centering
  \caption{Datasets Statistics for BDD100K Dataset}
    \begin{tabular}{lccc}
    \multicolumn{1}{c}{Weather} & Training & Validation & Testing \\
    \midrule
    Clear & \textcolor[rgb]{ .231,  .231,  .231}{29865} & 7479  & 5346 \\
    Snowy & 4445  & 1104  & 769 \\
    Rainy & 4119  & 951   & 738 \\
    Partly Cloudy & 3992  & 959   & 738 \\
    Overcast & 7043  & 1727  & 1239 \\
    Foggy & 57    & 43    & 43 \\
    \midrule
    \multicolumn{1}{c}{Scene} & Training & Validation & Testing \\
    \midrule
    Hightway & 13952 & 3427  & 2499 \\
    Residential area & 6458  & 1616  & 1253 \\
    City street & 34862 & 8654  & 6112 \\
    Parking lot & 297   & 80    & 49 \\
    Tunnel & 62    & 47    & 47 \\
    \end{tabular}%
  \label{tab:bdd100k}%
\end{table}%


\subsection{Details about Baselines}
\label{sec:baselines}

\textbf{FLYP.} In the original FLYP paper~\cite{goyal2023finetune}, the author 
presents extensive experiments demonstrating the superiority of employing the contrastive loss used during pre-training instead of the typical cross-entropy for finetuning image-text models for zero-shot vision classification. As the local contrastive loss, defined over the mini-batch samples, utilized in their paper requires a very large mini-batch size to converge, we follow \cite{yuan2022provable} to employ a global constrastive loss (GCL) as indicated in Eqn.~\ref{eq:flyp} to address this issue: 
\begin{equation}
\label{eq:flyp}
\begin{aligned}
&\min_{\w}~~~  \frac{1}{n_{all}}\sum\nolimits_{(\x_i, \t_i)\in\cD_{all}} L_{\text{ctr}}(\w; \x_i, \t_i,\cD_{all},\cD_{all})
\end{aligned}
\end{equation}
where $\cD_{all}= \cD \cup \cD_- \cup \cD_1 \cup \dots \cup \cD_m, n_{all} = n_o + 10*n_o + n_1 + \dots + n_m$, $\cD_-$ is the negative data collected form LAION400M as discussed in Appendix\ref{sec:laion}. All available data, including those used in our objective and constraints, are utilized for fine-tuning. The simple text prompts for the labeled BDD100k dataset are the same as those used for our method, i.e., "\textit{the weather is [Weather]}" and "\textit{the scene is a [Scene]}".

\textbf{WCCL.} Weighted Combination of Contrastive Losses(WCCL) is a straightforward baseline that utilizes a weight to combine GCL losses on protected tasks and the target task to balance protected tasks and the target task and achieve model developmental safety.
Specifically, the objective can be formulated as:
\begin{equation}
\label{eq:aclt}
\begin{aligned}
\min_{\w}~~~  \alpha   &\Big(\frac{1}{m}\sum_{k=1}^m  \frac{1}{n_{k}}\sum\nolimits_{(\x_i, \t_i)\in\cD_{k}} L_{\text{ctr}}(\w; \x_i, \t_i,\mathcal T_{ik}^-,\mathcal I_{ik}^-) \Big)  \\
+ & (1-\alpha) \left(\frac{1}{n_{o}}\sum\nolimits_{(\x_i, \t_i)\in\cD_{o}} L_{\text{ctr}}(\w; \x_i, \t_i,\mathcal T_{io}^-,\mathcal I_{io}^-) \right )
\end{aligned}
\end{equation}
where $\mathcal T_{ik}^- = \{\t_j: (\x_j, \t_j) \in \cD_{all} \backslash \cD_{k} \}\cup \{\t_i\}, \mathcal I_{ik}^- = \{\x_j: (\x_j, \t_j) \in \cD_{all}\backslash \cD_{k} \}\cup \{\x_i\}$, $\cD_{all}\backslash \cD_{k}$ denotes all training samples excluding samples from $\cD_k$. Similarly,   $\mathcal T_{io}^- = \{\t_j: (\x_j, \t_j) \in \cD_{all}\backslash \cD_{o} \}\cup \{\t_i\}, \mathcal I_{io}^- = \{\x_j: (\x_j, \t_j) \in \cD_{all}\backslash \cD_{o} \}\cup \{\x_i\}$. Consistent with other methods, the simple text prompts for this baseline are also "\textit{the weather is [Weather]}" and "\textit{the scene is a [Scene]}".

\textbf{GEM.}  GEM~\cite{lopez2017gradient} is a strong continual learning baseline which motivated by a similar idea, utilizing data of previous tasks for constraints. But it doesn't solve the constrained optimization problem directly but project gradients to reduce the increase in the loss of previous tasks. For GEM, we start from pretrained image encoder of the same CLIP model and initialize the linear classification heads $W\in \R^{d \times (m+1) }$ with the representations outputted by the text encoder with input "\textit{the weather is [Weather]}" or "\textit{the scene is a [Scene]}". For each task $k$, cross entropy loss is employed  $\mathcal L_k(\w,W, \cD_k) = \frac{1}{n_k}\sum_{(\x_i, y_i)\sim \cD_k} -\log\frac{\exp( W_k^\top E_1(\w, \x_i)  /\tau_0)}{\sum_{\ell=1}^{m+1}\exp(W_l^\top E_1(\w, \x_i)/\tau_0)}$, where $\tau_0>0$ is a temperature parameter, $W_k, W_l$ denoted the $k_{th}$, $l_{th}$ column vector of $W$ respectively, and $E_1(\w, \x_i)$ is the normalized image representation of $\x_i$. For consistency, $\tau_0$ is fixed to 0.05 as the one used in our method. In each iteration, 10 examples are drawn from each protected task to calculate the corresponding loss gradient vector for each task.

\textbf{RM.} In continual learning literature, adding explicit regularization terms is a widely used approach to balance old and new tasks, exploiting a frozen copy of previously-learned model to help prevent catastrophic forgetting~\cite{rebuffi2017icarl,castro2018end}. Similarly, the Regularization Method(RM) baseline incorporates the constraints from Eqn.~\eqref{eq:CLIP_train} as a regularization term, adding it to the objective function with an associated regularization weight:

\begin{equation}\label{eqn:rm}
\begin{aligned}
 \min_{\w } \frac{1}{n_{o}}\sum\nolimits_{(\x_i, \t_i)\in\cD_{o}} L_{\text{ctr}}(\w; \x_i, \t_i,\mathcal T_{io}^-,\mathcal I_{io}^-)  + \alpha \left(\frac{1}{m}\sum_{k=1}^m  \frac{1}{n_k} \sum_{(\x, y) \in \cD_k}\ell_{ce}(\w, \x, y)\right)
\end{aligned}
\end{equation}

\subsection{Retrieving external data from LIAON400M}\label{sec:laion}
To improve the performance of the CLIP model on the target task, we retrieve related image-text pairs from an external dataset.  Specifically, for each target task, we retrieve task-related image-text pairs from  Laion400M~\cite{schuhmann2021laion} to improve target performance, by going through the dataset and retrieving the image-text pairs with text containing the specific target task names, e.g., 'foggy', 'overcast', 'tunnel', 'dressing room'. 
{Similar approaches have been used  in \cite{liang2024aide,mitchell2018never,chen2013neil}, where  \cite{liang2024aide} used this approach to improve the detection of rare or unseen categories in object detection for autonomous driving systems. However, their study is different from ours in the sense that they do not provide guarantee on the model developmental safety.} 

Moreover, we refine the retrieved datasets. Let's take the task 'tunnel' as an example. For task ‘tunnel’, the retrieved data contained excessive noise, including numerous image-text pairs unrelated to tunnels, but contained 'tunnel' in the text. Therefore, we employed the GPT-4o API to filter the retrieved data with prompt "\textit{Determine whether the following caption mentions a tunnel or related context. First provide reasoning for your answer, and then respond with 'True' if it mentions a tunnel, or 'False' if it does not.}", thereby decreasing the noise of our retrieved data. The statistics of obtained task-related image-text pairs are presented in the Table~\ref{tab:data}. Additionally, for each target class, we randomly sample a set of image-text pairs from LAION400M that is 10 times larger than the positive set as negative data for contrasting.

\begin{table}[htbp]
  \centering
  \caption{Statistics of Data Collected from LIAON400M}
    \begin{tabular}{l|cccc}
    Task  & Foggy & Overcast & Tunnel & Dressing room\\
    \midrule
    Size  & 11415 & 4134  & 23484 & 6786 \\
    \end{tabular}%
  \label{tab:data}%
\end{table}%

\subsection{Visualization of Models' Learning Curves}\label{sec:vislc}

Along with the learning trajectory in the main paper, we present the training and validation curves in Fig.~\ref{fig:trajectory2} to further illustrate the learning process of the algorithm. From the figure, we can see that the DevSafety(acc) fluctuates along the model developmental safety line while $\Delta Acc (Target)$ continues to increase, which shows the model is striving to improve the model's performance while satisfying the MDS requirements.

\begin{figure*}[h!]
\centering
    \includegraphics[width=\linewidth]{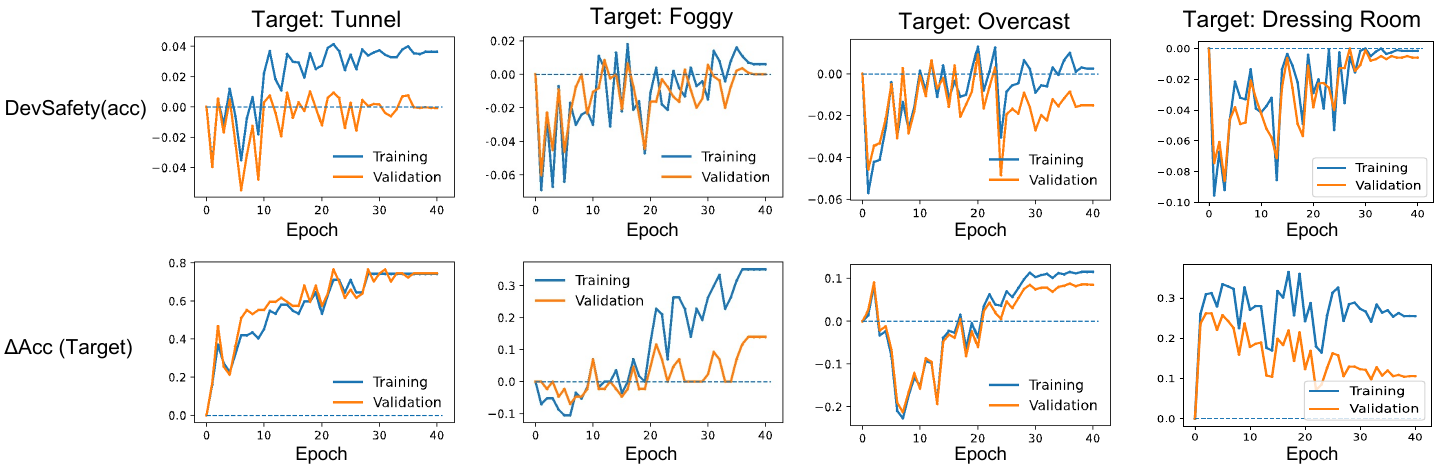}
\centering
\caption{ Models' Training and Validation Curves}
\label{fig:trajectory2}
\end{figure*}

\subsection{Deficiency of Weighting Methods}\label{sec:weight} 

As observed in Figure~\ref{fig:vsbaseline}, the naive weighting approach RM fail to achieve model developmental safety, even though they tradeoff the performance on the target task and protected tasks with weight parameter $\alpha$. To have a close look at why this happens, we show the detailed performance RM when targeting \textit{foggy}  with 4k samples for each protected
task in Table~\ref{tab:rm}. We find that, with a uniform weight for all the protected tasks, the method might preserve previous performance on some of the protected tasks but fail to achieve MDS for all the protected tasks, even with a very high $\alpha$. Moreover, with the weight $\alpha$ getting larger, the performance on the target task drops dramatically while the decrease gap goes smaller, e.g.,  \textit{Clear} tasks for RM. In contrast, our proposed method is able to preserve all the protected tasks' performance and improve the target task, as the mechanism of our algorithm is very different from using the uniform weight. In our method, weights for constraints depend on the loss of those tasks, i.e., the larger the violation, the larger the weight. As shown in Figure~\ref{fig:weight}, the weight for each protected task is adaptively adjusted during learning and once one protected task constraint is satisfied, it will not be penalized (weight becomes zero). This mechanism plays a big role in enabling the model to find feasible solutions to ensure zero-forgetting on all the protected tasks.

To further demonstrate the deficiency of the weighting method, we compare RM with our method on the Place365 dataset, targeting \textit{Dressing room} class and protecting the other 364 tasks in Figure~\ref{fig:places365}. With $\alpha=1,10,100,1000, 10000$, RM causes performance drops in 50, 35, 33, 32, and 35 classes, respectively. Although larger weights reduce the number of classes where performance drops, RM still cannot ensure MDS for all protected tasks and excessively high weights lead to performance decrease on the target task instead of improvement. In contrast, we can see that even with hundreds of protected tasks, our method is still effective in preserving their performance whiling improving the target task.

\begin{table}[tb!]
  \centering
  \caption{Detailed performance comparison between our method and baseline RM on targeting \textit{Foggy} with 4k samples for each protected task. Bold numbers highlight the performance decrease over the base model.}
    \resizebox{0.98\linewidth}{!}{
        \begin{tabular}{l|lllll|l|l}
                    & \multicolumn{5}{c|}{Protected Tasks}  & Target Task & \multicolumn{1}{c}{\multirow{2}[1]{*}{Average}} \\
                     & Clear & Overcast & Snowy & Rainy & Partly cloudy & Foggy &  \\
        \midrule
        \multicolumn{1}{l|}{Base} & 0.8938 & 0.7014 & 0.7503 & 0.7195 & 0.6734 & 0.3953 & 0.6889 \\
        \midrule
        \multicolumn{1}{l|}{Ours} & +0.0115(0.0054) & +0.0831(0.0228) & +0.0120(0.0079) & +0.0230(0.0081) & +0.1047(0.0168) & 0.0326(0.0316) & +0.0430(0.0027) \\
        \midrule
         RM $\alpha=0.1$ & \textbf{-0.0189(0.0039)} & +0.0667(0.0392) & +0.0328(0.0113) & +0.0081(0.0074) & +0.1253(0.0227) & +0.0559(0.0617) & +0.0450(0.0071) \\
            RM $\alpha=1$ & \textbf{-0.0129(0.0055)} & +0.0910(0.0102) & +0.0666(0.0139) & +0.0217(0.0215) & +0.1168(0.0112) & \textbf{-0.0604(0.0634)} & +0.0372(0.0114) \\
              RM $\alpha=10$ & \textbf{-0.0106(0.0085)} & +0.1131(0.0068) & +0.0656(0.0302) & +0.0163(0.0182) & +0.0830(0.0201) & \textbf{-0.1674(0.0174)} & +0.0167(0.0050) \\
        \end{tabular}%
    
    }
  \label{tab:rm}%
\end{table}%

\begin{figure*}[!tb]
\centering
    
    \subfigure[$\alpha=1$]{\includegraphics[width=0.45\linewidth]{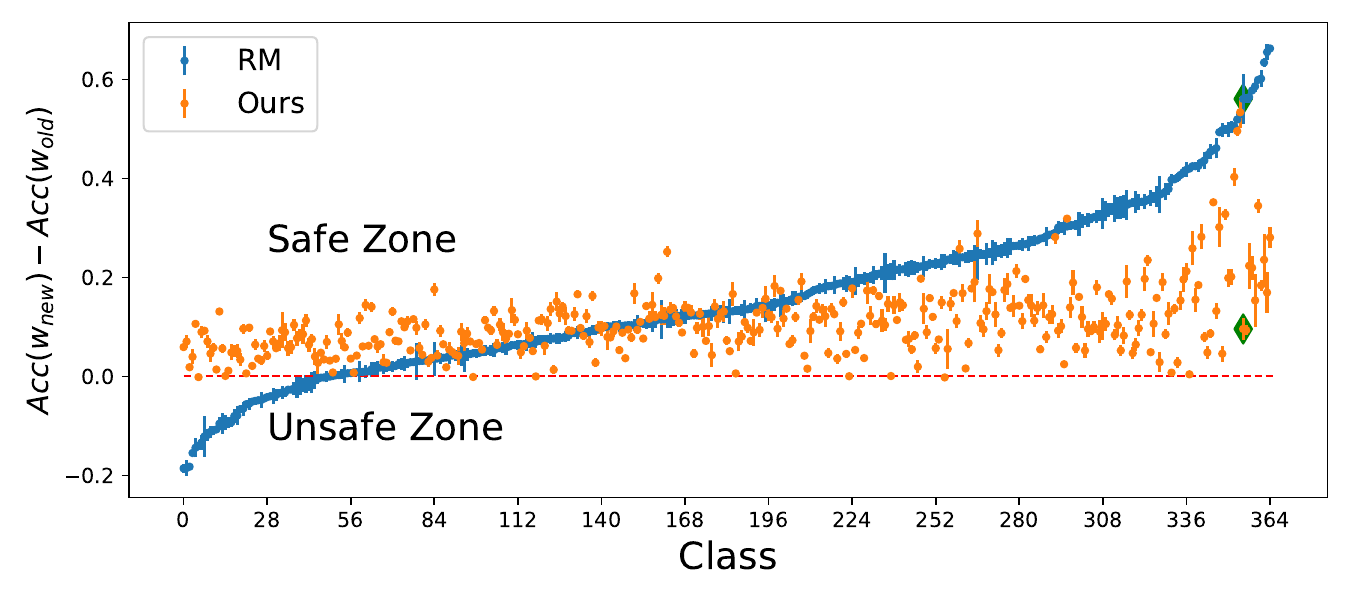} }
     \subfigure[$\alpha=10$]{\includegraphics[width=0.45\linewidth]{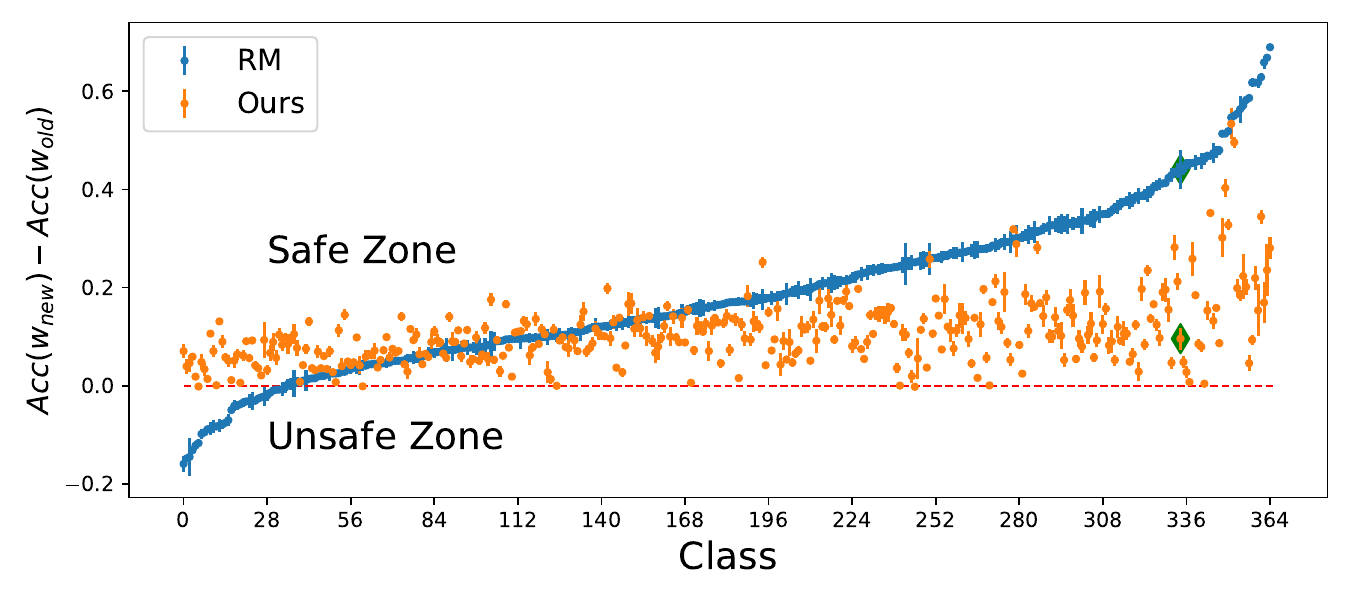} }
    \subfigure[$\alpha=100$]{\includegraphics[width=0.45\linewidth]{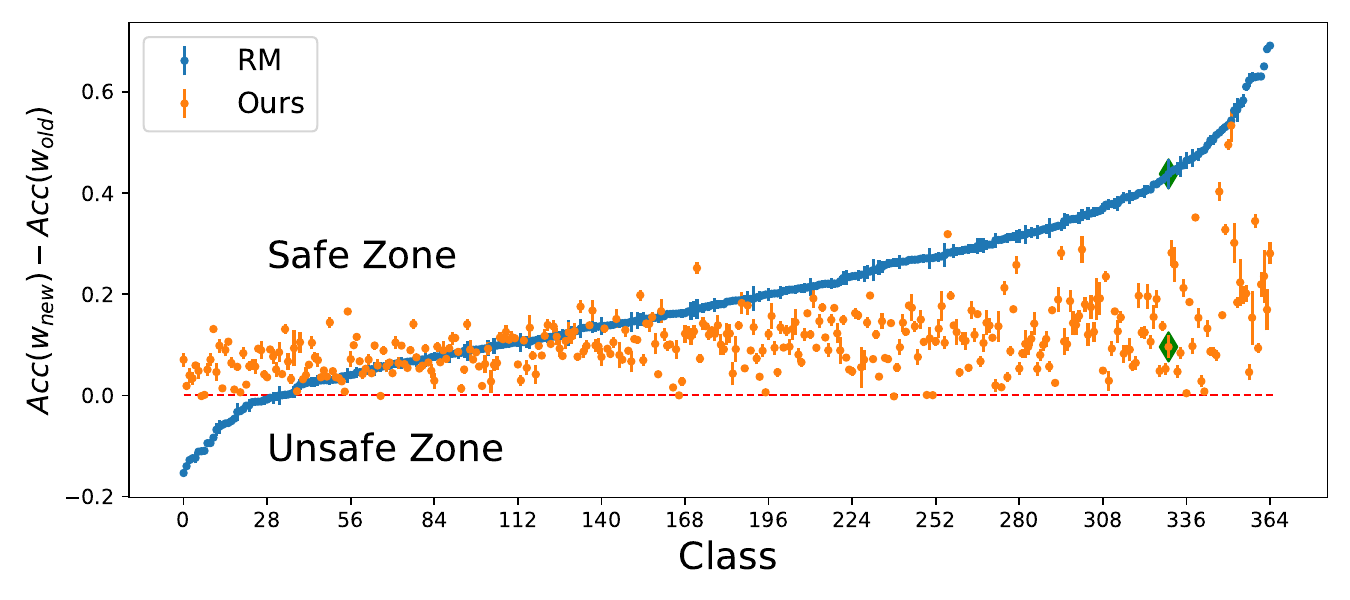} }
    \subfigure[$\alpha=1000$]{\includegraphics[width=0.45\linewidth]{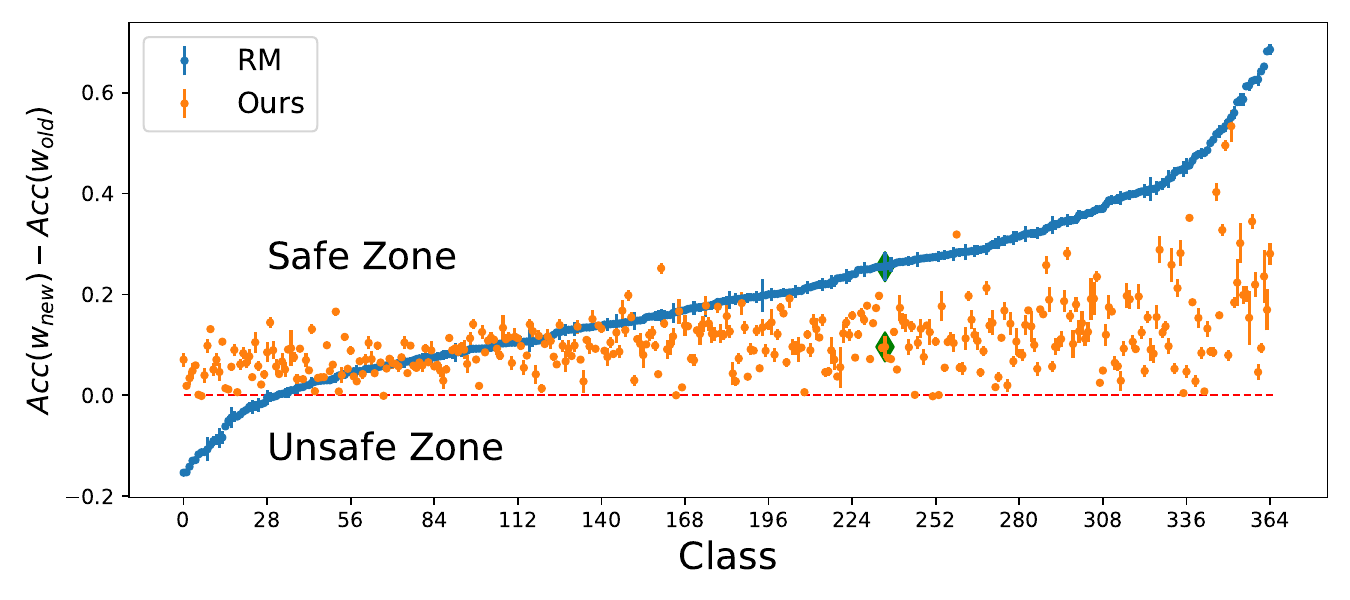}}
    \subfigure[$\alpha=10000$]{\includegraphics[width=0.45\linewidth]{./dressing_room_2k_test_w10000.pdf}}
\centering
\caption{Performance comparison between our method and baseline RM with different regularization weight $\alpha$ when targeting \textit{Dresssing Room} on Places365 Dataset, with 2k samples per constraint. Red line denotes base model's performance, green diamonds denote the target class. 
}
\label{fig:places365}
\end{figure*}

\subsection{Detailed Performance Comparison with Baselines}

In this part, we present a detailed performance comparison with baselines. Specifically, we include the DevSafety(acc) numbers for each method in Table~\ref{tab:d_tunnel},~\ref{tab:d_foggy},~\ref{tab:d_overcast}, which directly show the largest decrease over all the protected tasks. We can see that baselines usually lead to 1-14 percent decrease when targeting Tunnel, 0-13 percent decrease when targeting Foggy, 1-30 percent decrease when targeting Overcast. In contrast, our method demonstrates a smaller performance drop when there is insufficient data for constraints and ensures zero forgetting on the protected task when sufficient constraint data is available.

\begin{table}[h!]
  \centering
  \caption{Detailed Performance Comparison on Targeting Tunnel}
  \resizebox{0.98\linewidth}{!}{
    \begin{tabular}{r|l|cccc}
    \multicolumn{1}{l|}{Method} & Measures & 100   & 1k    & 2k    & 4k \\
    \midrule
    \multicolumn{1}{l|}{Base} & Retention Ratio//DevSafety(acc) & 100\%//0.00(0.0000) & 100\%//0.00(0.0000) & 100\%//0.00(0.0000) & 100\%//0.00(0.0000) \\
          & Target Tunnel & 0.1064(0.0000) & 0.1064(0.0000) & 0.1064(0.0000) & 0.1064(0.0000) \\
    \midrule
    \multicolumn{1}{l|}{FLYP} & Retention Ratio//DevSafety(acc) & 0.00\%//-0.0398(0.0067 & 0.00\%//-0.0660(0.0126) & 0.00\%//-0.0647(0.0123) & 0.00\%//-0.0774(0.0069) \\
          & Target Tunnel & 0.9361(0.0330) & 0.9702(0.0318) & 0.9915(0.0170) & 0.9659(0.0170) \\
    \midrule
    \multicolumn{1}{l|}{WCCL} & Retention Ratio//DevSafety(acc) & 0.00\%//-0.0836(0.0164) & 0.00\%//-0.0756(0.0090) & 0.00\%//-0.0673(0.0103) & 0.00\%//-0.0893(0.0089) \\
          & Target Tunnel & 0.9957(0.0085) & 0.6000(0.1002) & 0.6553(0.0282) & 0.6383(0.0485) \\
    \midrule
    \multicolumn{1}{l|}{GEM} & Retention Ratio//DevSafety(acc) & 0.00\%//-0.1019(0.0267) & 0.00\%//-0.1034(0.0153) & 0.00\%//-0.1301(0.0169) & 0.00\%//-0.0873(0.0231) \\
          & Target Tunnel & 0.8255(0.1214) & 0.5915(0.2020) & 0.6085(0.0768) & 0.3915(0.1819) \\
    \midrule
    \multicolumn{1}{l|}{Co$^2$L} & Retention Ratio//DevSafety(acc) & 0.00\%//-0.1407(0.0043) & 0.00\%//-0.1252(0.0061) & 0.00\%//-0.0821(0.0029) & 0.00\%//-0.0479(0.0039) \\
          & Target Tunnel & 0.6808(0.0460) & 0.8936(0.0626) & 0.8936(0.0301) & 0.8723(0.0000) \\
    \midrule
    \multicolumn{1}{l|}{DER} & Retention Ratio//DevSafety(acc) & 0.00\%//-0.0219(0.0083) & 0.00\%//-0.0156(0.0026) & 0.00\%//-0.0155(0.0048) & 0.00\%//-0.0111(0.0029) \\ 
          & Target Tunnel & 0.6894(0.0438) & 0.4340(0.1030) & 0.3957(0.0768) & 0.3234(0.0493) \\
    \midrule
    \multicolumn{1}{l|}{RM} & Retention Ratio//DevSafety(acc) & 0.00\%//-0.1021(0.0022) & 0.00\%//-0.0969(0.0036) & 0.00\%//-0.0955(0.0057) & 0.00\%//-0.0897(0.0068) \\
          & Target Tunnel & 0.9574(0.0233) & 0.8894(0.0340) & 0.8808(0.0170) & 0.8681(0.0085) \\
    \midrule
    \multicolumn{1}{l|}{Ours} & Retention Ratio//DevSafety(acc) & 40.00\%//-0.0050(0.0076) & 60.00\%//-0.0001(0.0043) & 100.00\%//0.0105(0.0053) & 100.00\%//0.0186(0.0058) \\
          & Target Tunnel & 0.9362(0.0699) & 0.8723(0.0233) & 0.9106(0.0159) & 0.8723(0.0233) \\
    \end{tabular}%
    }
  \label{tab:d_tunnel}%
\end{table}%

\begin{table}[h!]
  \centering
  \caption{Detailed Performance Comparison on Targeting Foggy}
    \resizebox{0.98\linewidth}{!}{
    \begin{tabular}{r|l|cccc}
    \multicolumn{1}{l|}{Method} & Measures & 100   & 1k    & 2k    & 4k \\
    \midrule
    \multicolumn{1}{l|}{Base} & Retention Ratio//DevSafety(acc) & 100\%//0.00(0.0000) & 100\%//0.00(0.0000) & 100\%//0.00(0.0000) & 100\%//0.00(0.0000) \\
          & Target Foggy & 0.3953(0.0000) & 0.3953(0.0000) & 0.3953(0.0000) & 0.3953(0.0000) \\
    \midrule
    \multicolumn{1}{l|}{FLYP} & Retention Ratio//DevSafety(acc) & 0.00\%//-0.0590(0.0140) & 20.00\%//-0.0281(0.0167) & 0.00\%//-0.0254(0.0101) & 0.00\%//-0.0201(0.0105) \\
          & Target Foggy & 0.5721(0.0315) & 0.5209(0.0581) & 0.5302(0.0228) & 0.4977(0.0186) \\
    \midrule
    \multicolumn{1}{l|}{WCCL} & Retention Ratio//DevSafety(acc) & 0.00\%//-0.0504(0.0123) & 0.00\%//-0.0259(0.0080) & 20.00\%//-0.0141(0.0111) & 0.00\%//-0.0132(0.0076) \\
          & Target Foggy & 0.3395(0.0865) & 0.2186(0.0186) & 0.2093(0.0208) & 0.2000(0.0114) \\
    \midrule
    \multicolumn{1}{l|}{GEM} & Retention Ratio//DevSafety(acc) & 0.00\%//-0.0695(0.0099) & 0.00\%//-0.0339(0.0053) & 0.00\%//-0.0424(0.0060) & 0.00\%//-0.0424(0.0060) \\
          & Target Foggy & 0.3349(0.0865) & 0.2837(0.0271) & 0.2558(0.0000) & 0.2558(0.0000) \\
    \midrule
    \multicolumn{1}{l|}{Co$^2$L} & Retention Ratio//DevSafety(acc) & 0.00\%//-0.0686(0.0064) & 0.00\%//-0.1217(0.0383) & 0.00\%//-0.1305(0.0183) & 0.00\%//-0.0721(0.0154) \\
          & Target Foggy & 0.7132(0.0109) & 0.6047(0.0380) & 0.6357(0.0110) & 0.6357(0.0290) \\
    \midrule
    \multicolumn{1}{l|}{DER} & Retention Ratio//DevSafety(acc) & 0.00\%//-0.0252(0.0033) & 40.00\%//-0.0029(0.0027) & 40.00\%//-0.0010(0.0010) & 60.00\%//0.0017(0.0038) \\ 
          & Target Foggy & 0.3163(0.0405) & 0.2233(0.0114) & 0.2186(0.0114) & 0.2372(0.0174) \\
    \midrule
    \multicolumn{1}{l|}{RM} & Retention Ratio//DevSafety(acc) & 0.00\%//-0.0418(0.0062) & 0.00\%//-0.0173(0.0054) & 0.00\%//-0.0159(0.0034) & 20.00\%//-0.0124(0.0091) \\
          & Target Foggy & 0.5674(0.0378) & 0.5023(0.0186) & 0.4419(0.0658) & 0.2279(0.0174) \\
    \midrule
    \multicolumn{1}{l|}{Ours} & Retention Ratio//DevSafety(acc) & 0.00\%//-0.0241(0.0082) & 60.00\%//-0.0009(0.0044) & 100.00\%//0.0044(0.0033) & 100.00\%//0.0061(0.0047) \\
          & Target Foggy & 0.5721(0.0406) & 0.4930(0.0174) & 0.4326(0.0186) & 0.4279(0.0316) \\
    \end{tabular}%
    }
  \label{tab:d_foggy}%
\end{table}%

\begin{table}[h!]
  \centering
  \caption{Detailed Performance Comparison on Targeting Overcast}
  \resizebox{0.98\linewidth}{!}{
    \begin{tabular}{r|l|cccc}
    \multicolumn{1}{l|}{Method} & Measures & 100   & 1k    & 2k    & 4k \\
    \midrule
    \multicolumn{1}{l|}{Base} & Retention Ratio//DevSafety(acc) & 100\%//0.00(0.0000) & 100\%//0.00(0.0000) & 100\%//0.00(0.0000) & 100\%//0.00(0.0000) \\
          & Target Overcast & 0.7361(0.0000) & 0.7361(0.0000) & 0.7361(0.0000) & 0.7361(0.0000) \\
    \midrule
    \multicolumn{1}{l|}{FLYP} & Retention Ratio//DevSafety(acc) & 0.00\%//-0.0749(0.0049) & 0.00\%//-0.0449(0.0140) & 0.00\%//-0.0434(0.0095) & 0.00\%//-0.0314(0.0113) \\
          & Target Overcast & 0.9143(0.0111) & 0.8559(0.0241) & 0.8412(0.0294) & 0.8247(0.0255) \\
    \midrule
    \multicolumn{1}{l|}{WCCL} & Retention Ratio//DevSafety(acc) & 0.00\%//-0.1192(0.0294) & 0.00\%//-0.0716(0.0053) & 0.00\%//-0.0424(0.0091) & 0.00\%//-0.0414(0.0102) \\
          & Target Overcast & 0.9315(0.0112) & 0.9296(0.0092) & 0.9207(0.0022) & 0.9172(0.0064) \\
    \midrule
    \multicolumn{1}{l|}{GEM} & Retention Ratio//DevSafety(acc) & 0.00\%//-0.0677(0.0042) & 0.00\%//-0.0711(0.0050) & 0.00\%//-0.0807(0.0128) & 0.00\%//-0.0634(0.0042) \\
          & Target Overcast & 0.9282(0.0051) & 0.9233(0.0037) & 0.9149(0.0088) & 0.9165(0.0049) \\
    \midrule
    \multicolumn{1}{l|}{Co$^2$L} & Retention Ratio//DevSafety(acc) & 0.00\%//-0.0138(0.0099) & 0.00\%//-0.0072(0.0032) & 0.00\%//-0.0095(0.0043) & 0.00\%//-0.0137(0.0052) \\
          & Target Overcast & 0.5916(0.0417) & 0.8369(0.0049) & 0.8396(0.0055) & 0.8507(0.0172) \\
    \midrule
    \multicolumn{1}{l|}{DER} & Retention Ratio//DevSafety(acc) & 0.00\%//-0.0435(0.0037) & 0.00\%//-0.0241(0.0064) & 0.00\%//-0.0182(0.0032) & 0.00\%//-0.0166(0.0077) \\ 
          & Target Overcast  & 0.8731(0.0066) & 0.8604(0.0075) & 0.8602(0.0023) & 0.8651(0.0027) \\
    \midrule
    \multicolumn{1}{l|}{RM} & Retention Ratio//DevSafety(acc) & 0.00\%//-0.2932(0.0365) & 0.00\%//-0.3016(0.0228) & 0.00\%//-0.2444(0.0120) & 0.00\%//-0.2634(0.0105) \\
          & Target Overcast & 0.9787(0.0050) & 0.9730(0.0028) & 0.9588(0.0041) & 0.9647(0.0023) \\
    \midrule
    \multicolumn{1}{l|}{Ours} & Retention Ratio//DevSafety(acc) & 0.00\%//-0.0655(0.0249) & 20.00\%//-0.0043(0.0037) & 60.00\%//0.0012(0.0029) & 100.00\%//0.0046(0.0016) \\
          & Target Overcast & 0.8789(0.0464) & 0.7827(0.0225) & 0.7562(0.0167) & 0.7525(0.0366) \\
    \end{tabular}%
    }
  \label{tab:d_overcast}%
\end{table}%

\clearpage

\section{Proofs}
\subsection{Proof of Lemma~\ref{thm:safety_generalization}}\label{sec:safety_generalization}
\begin{proof}
Consider task $k$. Recall that $\cD_k$ contains $n_k$ data points.
According to Theorem 3.2 in \cite{boucheron2005theory}, we have with probability at least $1-\delta/m$, for all $\w$,
$$
\left|\mathcal L_k(\w, \mathfrak{D}_k) - \mathcal L_k(\w, \cD_k)\right|\leq 
2R_{n_k}(\mathcal H) +\sqrt{\frac{\ln(2m/\delta)}{2n_k}}
\leq\frac{2C}{n_k^\alpha} +\sqrt{\frac{\ln(2m/\delta)}{2n_k}},
$$
where the second inequality is by the assumption on $R_n(\mathcal H)$.
Combining the inequalities above with  $\w=\w_{\text{new}}$ and $\w=\w_{\text{old}}$, we have with probability at least $1-\delta/m$ 
    \begin{align*}
        \mathcal L_k(\w_{\text{new}}, \mathfrak{D}_k) -  \mathcal L_k(\w_{\text{old}}, \mathfrak{D}_k)\leq \mathcal L_k(\w_{\text{new}}, \cD_k)-  \mathcal L_k(\w_{\text{old}}, \cD_k)  +\frac{4C}{n_k^\alpha} +2\sqrt{\frac{\ln(2m/\delta)}{2n_k}}.
    \end{align*}
Applying the union bound with the events above for $k=1,\dots,m$ leads to the conclusion of this lemma. 
\end{proof}

\subsection{Proof of Lemma~\ref{thm:regularity}}
\label{sec:regularity}
\begin{proof}
Recall that $\w$ has two component $\u$ and $W$. 
The gradient of $h_k(\w)$ with respect to $W$ and $\u$ are denoted by 
$\nabla_{W}h_k(\w)$ and $\nabla_{\u}h_k(\w)$, respectively. Hence,
$$
\nabla h_k(\w)=\left(\nabla_{\u}h_k(\w), \nabla_{W}h_k(\w)\right)
$$
for $k=1,\dots,m$. Similarly, after adding the task-dependent heads, $\hat\w$ has four component $\u$, $W$, $\mathbf{U}$ and $\mathbf{V}$. The gradients  $\nabla_{\u}\hat h_k(\hat\w)$, $\nabla_{W}\hat h_k(\hat\w)$ $\nabla_{\mathbf{U}}\hat h_k(\hat\w)$ and $\nabla_{\mathbf{V}}\hat h_k(\hat\w)$ are defined correspondingly, and
$$
\nabla \hat h_k(\hat\w)=\left(\nabla_{\u}\hat h_k(\hat\w), \nabla_{W}\hat h_k(\hat\w), \nabla_{\mathbf{U}}\hat h_k(\hat\w), \nabla_{\mathbf{V}}\hat h_k(\hat\w)\right).
$$
Recall that
$$
\hat h_k(\hat\w)=h_k(W+U_kV_k^\top,\u)\text{ for }k=1,\dots,m.
$$
Therefore, 
\small
$$
\nabla_{\u} \hat h_k(\hat \w)=\nabla_{\u} h_k(W+U_kV_k^\top,\u),\quad\quad
\nabla_{W} \hat h_k(\hat \w)=\nabla_{W} h_k(W+U_kV_k^\top,\u)
$$
\normalsize
and
\small
\begin{eqnarray*}
\nabla_\mathbf{U} \hat h_k(\hat\w)
&=&\bigg(
\mathbf{0}, \dots, \mathbf{0}, 
\underbrace{\nabla_{W} h_k(\mathbf{W}+U_kV_k^\top,\u)V_k}_{\text{The } k\text{th  block} },
\mathbf{0}, \dots, \mathbf{0}
\bigg)^\top
\end{eqnarray*}
\normalsize
\small
\begin{eqnarray*}
\nabla_\mathbf{V} \hat h_k(\hat\w)
&=&\bigg(
\mathbf{0}, \dots, \mathbf{0}, 
\underbrace{\nabla_{W} h_k(\mathbf{W}+U_kV_k^\top,\u)^\top U_k}_{\text{The } k\text{th  block} },
\mathbf{0}, \dots, \mathbf{0}
\bigg)^\top,
\end{eqnarray*}
\normalsize
where the sparsity patterns of $\nabla_\mathbf{U} \hat h_k(\hat\w)$ and $\nabla_\mathbf{V} \hat h_k(\hat\w)$ are because $\hat h_k$ does not depend on $U_j$ and $V_j$ with $j\neq k$.

Suppose $U_kV_k^\top=\mathbf{0}$ for all $k$. It holds that $h_k(\w)=\hat h_k(\hat \w)$ and
$$
\nabla h_k(\w)=\left(\nabla_{\u}h_k(\w), \nabla_{W}h_k(\w)\right)=\left(\nabla_\u \hat h_k(\hat\w),\nabla_W \hat h_k(\hat\w)\right).
$$
Consider any $\boldsymbol{\alpha}=(\alpha_1,\dots,\alpha_m)\in\mathbb{R}^m$. We have
\small
\begin{align*}
&\lambda_{\min}\left(\left[\nabla \hat h_1(\hat\w),\dots,\nabla \hat h_m(\hat\w)\right]^\top\left[\nabla \hat h_1(\hat\w),\dots,\nabla \hat h_m(\hat\w)\right]\right)\\
=&\min_{\boldsymbol{\alpha},\text{s.t.}\|\boldsymbol{\alpha}\|=1}\left\|\sum_{k=1}^m\alpha_k\nabla \hat h_k(\hat\w)\right\|^2\\
=&\min_{\boldsymbol{\alpha},\text{s.t.}\|\boldsymbol{\alpha}\|=1}\left(\left\|\sum_{k=1}^m\alpha_k\nabla_\u \hat h_k(\hat\w)\right\|^2+\left\|\sum_{k=1}^m\alpha_k\nabla_W \hat h_k(\hat\w)\right\|^2+\left\|\sum_{k=1}^m\alpha_k\nabla_{\mathbf{U}} \hat h_k(\hat\w)\right\|^2+\left\|\sum_{k=1}^m\alpha_k\nabla_{\mathbf{V}} \hat h_k(\hat\w)\right\|^2\right)\\
=&\min_{\boldsymbol{\alpha},\text{s.t.}\|\boldsymbol{\alpha}\|=1}\left(\left\|\sum_{k=1}^m\alpha_k\nabla h_k(\w)\right\|^2+\sum_{k=1}^m\alpha_k^2\left\|\nabla_{W} h_k(\w)V_k\right\|_F^2
+\sum_{k=1}^m\alpha_k^2\left\|\nabla_{W} h_k(\w)^\top U_k\right\|_F^2\right)\\
\geq&\lambda_{\min}\left(\left[\nabla  h_1( \w),\dots,\nabla h_m( \w)\right]^\top\left[\nabla   h_1( \w),\dots,\nabla h_m( \w)\right]\right)\\
&+\min_k\left\|\nabla_{W} h_k(\w)V_k\right\|_F^2+\min_k\left\|\nabla_{W} h_k(\w)^\top U_k\right\|_F^2,
\end{align*}
\normalsize
where the first two equalities are by definitions and the third equality is because $U_kV_k^\top=\mathbf{0}$ for all $k$.

\end{proof}

\subsection{Proof of Theorem~\ref{theorem}}
\label{proof_thm}
In this section, we present the proof of the Theorem~\ref{theorem}. Recall that the problem is formulated as
\begin{equation}\label{eqn:pprob}
    \min_{\vw}\ F(\vw,\cD):=\frac{1}{n_0}\sum^{n_0}_{i=1}\left(f(g_{1i}(\vw))+ f(g_{2i}(\vw))\right) \quad \text{s.t.} \quad \frac{1}{m}h_k(\vw;\mathcal{D}_k)\leq 0,\ k=1,\cdots,m.
\end{equation}
with $f(\cdot)=\tau \log(\cdot)$. With the quadratic penalty method, the problem is converted to 
\begin{equation}
    \min_{\vw}\ \Phi(\vw) := F(\vw,\cD) + \underbrace{\frac{1}{m}\sum_{k=1}^m \frac{\beta}{2}([h_k(\vw;\mathcal{D}_k)]_+)^2}_{H(\vw)}.
\end{equation}
By Assumptions \ref{assumption1}, we can get $f$ is $L_f$-Lipschitz continuous and $L_{\nabla f}$-smooth with $L_f=\frac{\tau}{c_g}$ and $L_{\nabla f}=\frac{\tau}{c_g^2}$. By noticing that $\ell_{ce}$ is a cross entropy loss, we find that $|h_k(\cdot)|$ can be bounded by a constant $C_h$ with $C_h=2$. Then, we can get $\Phi(\vw)$ is $L_{\beta}$-smooth with $L_{\beta}:= L_F + \beta L_H$ where $L_F :=2( L_{\nabla g}L_f+L_{\nabla f}L_g^2)$ and $L_H := L_{\nabla h} C_h + L_h^2$. We also define $\Tilde{C}_{\nabla g}:=\sigma_{\nabla g}+L_g$ and $\Tilde{C}_{\nabla h}:=\sigma_{\nabla h}+L_h$. To facilitate our discussion, we let
\begin{equation*}
    \begin{split}
        v_1^{t} &= (1-\theta)v_1^{t-1}+\theta G_1^t,\\
        v_2^{t} &= (1-\theta)v_2^{t-1}+\theta G_2^t,\\
        v^{t} & = v_1^{t} + v_2^{t}.
    \end{split}
\end{equation*}
To prove our main theorem, we need following lemmas.
\begin{lemma}
    If $\theta \leq \frac{1}{3}$, the gradient variance $\Delta_1^{t}:= \|v_1^{t}-\nabla F(\vw^t,\cD)\|^2$ can be bounded as
    \begin{equation}\label{eqn:lemma1}
    \begin{split}
        \mathbb{E}[\Delta_1^{t+1}]\leq & (1-\theta)\mathbb{E}[\Delta_1^{t}]+\frac{2L_{F}^2}{\theta}\E[\|\vw^{t+1}-\vw^t\|^2]+5\theta L_f^2\Tilde{C}_{\nabla g}^2\mathbb{E}[\Xi^{t+1}_1+\Xi^{t+1}_2]\\
        &+3L_f^2\Tilde{C}_{\nabla g}^2\mathbb{E}\left[\frac{1}{n_0}\sum_{i\in \mathcal{B}^{t+1}}\left(\left\|u_{1i}^{t+1}-u_{1i}^t\right\|^2+\left\|u_{2i}^{t+1}-u_{2i}^t\right\|^2\right)\right]+\frac{2\theta^2 L_f^2(\sigma_{\nabla g}^2+L_g^2)}{\text{min}\{|\cB|,|\cB_{1i}|,|\cB_{2i}|\}},
        \end{split}
    \end{equation}
with $\Xi^{t+1}_1:= \frac{1}{n_0}\|\textbf{u}_1^{t+1}-\textbf{g}_1(\vw^{t+1})\|^2=\frac{1}{n_0}\sum_{i=1}^{n_0}\| u_{1i}^{t+1}-g_{1i}(\vw^{t+1})\|^2$ and $\Xi^{t+1}_2:= \frac{1}{n_0}\|\textbf{u}_2^{t+1}-\textbf{g}_2(\vw^{t+1})\|^2=\frac{1}{n_0}\sum_{i=1}^{n_0}\| u_{2i}^{t+1}-g_{2i}(\vw^{t+1})\|^2$.
\end{lemma}

\begin{proof}
    \begin{equation*}
        \begin{split}
            \Delta_1^{t+1}&=\|v_1^{t+1}-\nabla F(\vw^{t+1})\|^2=\|(1-\theta)v_1^{t}+\theta G_1^{t}- \nabla F(\vw^{t+1})\|^2\\
            & = \left\|\circled{1}+\circled{2}+\circled{3}+\circled{4}\right\|^2,
        \end{split}
    \end{equation*}
where $\circled{1}$, $\circled{2}$, $\circled{3}$, $\circled{4}$ are defined as
\small
\begin{equation*}
    \begin{split}
        \circled{1}&=(1-\theta)(v_1^{t}-\nabla F(\vw^{t})), \quad \circled{2}=(1-\theta)(\nabla F(\vw^{t})-\nabla F(\vw^{t+1})),\\
        \circled{3}&=\frac{\theta}{|\cB|}\sum_{i\in\mathcal{B}^{t+1}}\nabla \hat{g}_{1i}(\vw^{t+1})\left(\nabla f(u_{1i}^{t})-\nabla f(g_{1i}(\vw^{t+1}))\right)+\nabla \hat{g}_{2i}(\vw^{t+1})\left(\nabla f(u_{2i}^{t})-\nabla f(g_{2i}(\vw^{t+1}))\right),\\
        \circled{4}&=\frac{\theta}{|\cB|}\sum_{i\in\mathcal{B}^{t+1}}\nabla \hat{g}_{1i}(\vw^{t+1})\nabla f(g_{1i}(\vw^{t+1}))+\nabla \hat{g}_{2i}(\vw^{t+1})\nabla f(g_{2i}(\vw^{t+1}))-\nabla F(\vw^{t+1}).
    \end{split}
\end{equation*}
\normalsize
Note that $\E_t[\langle\circled{1},\circled{4}\rangle]=\E_t[\langle\circled{2},\circled{4}\rangle]=0$. Then, by the Young's inequality, we can get
\small
\begin{equation*}
    \begin{split} &\E_t\left[\left\|\circled{1}+\circled{2}+\circled{3}+\circled{4}\right\|^2\right]\\
    = & \left\|\circled{1}\right\|^2+\left\|\circled{2}\right\|^2+\E_t\left\|\circled{3}\right\|^2+\E_t\left\|\circled{4}\right\|^2+2\left\langle \circled{1},\circled{2}\right\rangle+2\E_t[\left\langle \circled{1},\circled{3}\right\rangle]+2\E_t[\left\langle\circled{2},\circled{3}\right\rangle]+2\E_t[\left\langle \circled{3},\circled{4}\right\rangle]\\
    \leq & (1+\theta)\left\|\circled{1}\right\|^2+2\left(1+\frac{1}{\theta}\right)\left\|\circled{2}\right\|^2+\frac{2+3\theta}{\theta}\E_t\left\|\circled{3}\right\|^2+2\E_t\left\|\circled{4}\right\|^2.
    \end{split}
\end{equation*}
\normalsize
We can also get
\begin{small}

\begin{equation*}
    \begin{split}
    &(1+\theta)\|\circled{1}\|^2= (1+\theta)(1-\theta)^2\|v_1^{t}-\nabla F(\vw^t)\|^2 \leq (1-\theta)\|v_1^{t}-\nabla F(\vw^t)\|^2\\
    &2\left(1+\frac{1}{\theta}\right)\left\|\circled{2}\right\|^2 = 2\left(1+\frac{1}{\theta}\right)(1-\theta)^2\|\nabla F(\vw^t)-\nabla F(\vw^{t+1})\|^2\leq \frac{2L_{F}^2}{\theta}\|\vw^{t+1}-\vw^t\|^2
    \end{split}
\end{equation*}

\begin{equation*}
    \begin{split}
         &\frac{2+3\theta}{\theta}\mathbb{E}_t\left[\|\circled{3}\|^2\right]
         \\&=\frac{2+3\theta}{\theta}\frac{\theta^2}{|\cB|}\mathbb{E}_t\sum_{i\in\mathcal{B}^{t+1}}\left(\left\|\nabla \hat{g}_{1i}(\vw^{t+1})\right\|^2\left\|\nabla f(u_{1i}^{t})-\nabla f(g_{1i}(\vw^{t+1})\right\|^2+\left\|\nabla \hat{g}_{2i}(\vw^{t+1})\right\|^2\left\|\nabla f(u_{2i}^{t})-\nabla f(g_{2i}(\vw^{t+1})\right\|^2\right)
    \end{split}
\end{equation*}
\end{small}
We first bound the first term
\begin{small}
\begin{equation*}
    \begin{split}
    &\frac{(2+3\theta)\theta}{|\cB|}\mathbb{E}_t\sum_{i\in\mathcal{B}^{t+1}}\left\|\nabla \hat{g}_{1i}(\vw^{t+1})\right\|^2\left\|\nabla f(u_{1i}^{t})-\nabla f(g_{1i}(\vw^{t+1})\right\|^2\\
    &\leq \frac{(2+3\theta)\theta L_f^2}{|\cB|}\mathbb{E}_t\left[\sum_{i\in\mathcal{B}^{k+1}}\left\|\nabla \hat{g}_{1i}(\vw^{t+1})\right\|^2\left\|u_{1i}^{t}-g_{1i}(\vw^{t+1})\right\|^2\right]\\
    &=(2+3\theta)\theta L_f^2\mathbb{E}_t\left[\frac{1}{|\cB|}\sum_{i\in\mathcal{B}^{t+1}}\mathbb{E}_t\left[\left\|\nabla \hat{g}_{1i}(\vw^{k+1})\right\|^2|i\in\mathcal{B}^{t+1}\right]\left\|u_{1i}^{t}-g_{1i}(\vw^{k+1})\right\|^2\right]\\
    &\leq (2+3\theta)\theta L_f^2\Tilde{C}_{\nabla g}^2\mathbb{E}_t\left[\frac{1}{|\cB|}\sum_{i\in\mathcal{B}^{t+1}}\left\|u_{1i}^{t}-g_{1i}(\vw^{t+1})\right\|^2\right]\\
    &\leq (2+3\theta)\theta L_f^2\Tilde{C}_{\nabla g}^2\left((1+\delta)\mathbb{E}_t\left[\frac{1}{n_0}\sum_{i=1}^{n_0}\left\|u_{1i}^{t+1}-g_{1i}(\vw^{t+1})\right\|^2\right]+(1+1/\delta)\mathbb{E}_t\left[\frac{1}{n_0}\sum_{i=1}^{n_0}\left\|u_{1i}^{t+1}-u_{1i}^t\right\|^2\right]\right)\\
    &= (2+3\theta)\theta  L_f^2\Tilde{C}_{\nabla g}^2\left((1+\delta)\mathbb{E}_t\left[\frac{1}{n_0}\sum_{i=1}^{n_0}\left\|u_{1i}^{t+1}-g_i(\vw^{t+1})\right\|^2\right]+(1+1/\delta)\mathbb{E}_t\left[\frac{1}{n_0}\sum_{i\in \mathcal{B}^{t+1}}\left\|u_{1i}^{t+1}-u_{1i}^t\right\|^2\right]\right)
    \end{split}
\end{equation*}
\end{small}
If $\theta\leq \frac{1}{3}$ and $\delta=\frac{3\theta}{2}$, we have $(2+3\theta)\theta(1+\delta)\leq 5\theta$ and $(2+3\theta)\theta(1+1/\delta)\leq 3$. And similarly, we can get the bound for the second term. Then, by combining them, we can get
\begin{equation*}
    \frac{2+3\theta}{\theta}\mathbb{E}\left[\|\circled{3}\|^2\right]\leq 5\theta L_f^2\Tilde{C}_{\nabla g}^2\mathbb{E}[\Xi_1^{t+1}+\Xi_2^{t+1}]+3L_f^2\Tilde{C}_{\nabla g}^2\mathbb{E}\left[\frac{1}{n_0}\sum_{i\in \mathcal{B}^{t+1}}\left(\left\|u_{1i}^{t+1}-u_{1i}^t\right\|^2+\left\|u_{2i}^{t+1}-u_{2i}^t\right\|^2\right)\right].
\end{equation*}
\begin{equation*}
    \begin{split}
       & \mathbb{E}_t\left[\|\circled{4}\|^2\right]\\
       =&\theta^2\mathbb{E}_t\left[ \left\| \frac{1}{|\cB|}\sum_{i\in\mathcal{B}^{t+1}}\nabla \hat{g}_{1i}(\vw^{k+1})\nabla f(g_{1i}(\vw^{k+1}))-\frac{1}{n_0}\sum_{i=1}^{n_0}\nabla g_{1i}(\vw^{t+1})\nabla f(g_{1i}(\vw^{t+1}))\right\|^2 \right]\\
        & + \theta^2\mathbb{E}_t\left[ \left\| \frac{1}{|\cB|}\sum_{i\in\mathcal{B}^{t+1}}\nabla \hat{g}_{2i}(\vw^{k+1})\nabla f(g_{2i}(\vw^{t+1}))-\frac{1}{n_0}\sum_{i=1}^{n_0}\nabla g_{2i}(\vw^{t+1})\nabla f(g_{2i}(\vw^{t+1}))\right\|^2 \right]\\
        = & \theta^2\mathbb{E}_t\left[ \left\| \frac{1}{|\cB|}\sum_{i\in\mathcal{B}^{t+1}}\nabla \hat{g}_{1i}(\vw^{t+1})\nabla f(g_{1i}(\vw^{t+1}))-\frac{1}{|\cB|}\sum_{i\in\mathcal{B}^{t+1}}\nabla g_{1i}(\vw^{t+1})\nabla f(g_{1i}(\vw^{t+1}))\right\|^2 \right]\\
        & + \theta^2\mathbb{E}_t\left[ \left\| \frac{1}{|\cB|}\sum_{i\in\mathcal{B}^{t+1}}\nabla g_{1i}(\vw^{t+1})\nabla f(g_{1i}(\vw^{t+1}))-\frac{1}{n_0}\sum_{i=1}^{n_0}\nabla g_{1i}(\vw^{t+1})\nabla f(g_{1i}(\vw^{t+1}))\right\|^2 \right]\\
        & + \theta^2\mathbb{E}_t\left[ \left\| \frac{1}{|\cB|}\sum_{i\in\mathcal{B}^{t+1}}\nabla \hat{g}_{2i}(\vw^{t+1})\nabla f(g_{2i}(\vw^{t+1}))-\frac{1}{|\cB|}\sum_{i\in\mathcal{B}^{t+1}}\nabla g_{2i}(\vw^{t+1})\nabla f(g_{2i}(\vw^{t+1}))\right\|^2 \right]\\
        & + \theta^2\mathbb{E}_t\left[ \left\| \frac{1}{|\cB|}\sum_{i\in\mathcal{B}^{t+1}}\nabla g_{2i}(\vw^{t+1})\nabla f(g_{2i}(\vw^{t+1}))-\frac{1}{n_0}\sum_{i=1}^{n_0}\nabla g_{2i}(\vw^{t+1})\nabla f(g_{2i}(\vw^{t+1}))\right\|^2 \right]\\
        \leq & \frac{2\theta^2 L_f^2(\sigma_{\nabla g}^2+L_g^2)}{\text{min}\{|\cB|,|\cB_{1i}|,|\cB_{2i}|\}}.
    \end{split}
\end{equation*}
\normalsize
Therefore, we can get
\begin{equation*}
    \begin{split}
        \mathbb{E}[\Delta_1^{t+1}]\leq & (1-\theta)\mathbb{E}[\Delta_1^{t}]+\frac{2L_{F}^2}{\theta}\E[\|\vw^{t+1}-\vw^t\|^2]+5\theta L_f^2\Tilde{C}_{\nabla g}^2\mathbb{E}[\Xi^{t+1}_1+\Xi^{t+1}_2]\\
        &+3L_f^2\Tilde{C}_{\nabla g}^2\mathbb{E}\left[\frac{1}{n_0}\sum_{i\in \mathcal{B}^{t+1}}\left(\left\|u_{1i}^{t+1}-u_{1i}^t\right\|^2+\left\|u_{2i}^{t+1}-u_{2i}^t\right\|^2\right)\right]+\frac{2\theta^2 L_f^2(\sigma_{\nabla g}^2+L_g^2)}{\text{min}\{|\cB|,|\cB_{1i}|,|\cB_{2i}|\}}.
    \end{split}
 \end{equation*}

\end{proof}

\begin{lemma}\label{lem:g1}
    If $\gamma_1\leq 1/5$, function value variance $\Xi^t_1:=\frac{1}{n_0}\|\textbf{u}_1^t-\textbf{g}_1(\vw^t)\|^2$ can be bounded as
    \begin{equation}\label{eqn:lemma2}
        \E [\Xi^{t+1}_1]\leq \left(1-\frac{\gamma_1 |\cB|}{4n_0} \right)\E\left[\Xi^t_1\right]+\frac{5n_0L_g^2\E[\|\vw^{t+1}-\vw^t\|^2]}{\gamma_1 |\cB|}+\frac{2\gamma_1^2\sigma_g^2 |\cB|}{n_0|\cB_{1i}|}-\frac{1}{4n_0}\E\left[\sum_{i\in\mathcal{B}^{t+1}}\|u_{1i}^{t+1}-u_{1i}^t\|^2\right].
    \end{equation}
\end{lemma}

\begin{lemma}\label{lem:g2}
    If $\gamma_1\leq 1/5$, function value variance $\Xi^t_2:=\frac{1}{n_0}\|\textbf{u}_2^t-\textbf{g}_2(\vw^t)\|^2$ can be bounded as
    \begin{equation}\label{eqn:lemma_g2}
        \E [\Xi^{t+1}_2]\leq \left(1-\frac{\gamma_1 |\cB|}{4n_0} \right)\E\left[\Xi^t_2\right]+\frac{5n_0L_g^2\E[\|\vw^{t+1}-\vw^t\|^2]}{\gamma_1 |\cB|}+\frac{2\gamma_1^2\sigma_g^2 |\cB|}{n_0|\cB_{2i}|}-\frac{1}{4n_0}\E\left[\sum_{i\in\mathcal{B}^{t+1}}\|u_{2i}^{t+1}-u_{2i}^t\|^2\right].
    \end{equation}
\end{lemma}
Since the proof of Lemma~\ref{lem:g1} and Lemma~\ref{lem:g2} are almost the same, we only presents the proof of Lemma~\ref{lem:g1} as follows.

\begin{proof}
Define $\phi_1^t(\textbf{u}_1)=\frac{1}{2}\|\textbf{u}_1-\textbf{g}_1(\vw^k)\|^2=\frac{1}{2}\sum^{n_0}_{i=1}\|u_{1i}-g_{1i}(\vw^k)\|^2$, which is $1$-strongly convex.
\begin{equation}
\begin{split}
    \phi^{t+1}_1(\textbf{u}_1^{t+1})&=\frac{1}{2}\|\textbf{u}_1^{t+1}-\textbf{g}_1(\vw^{t+1})\|^2 = \frac{1}{2}\|\textbf{u}_1^{t}-\textbf{g}_1(\vw^{t+1})\|^2+\langle \textbf{u}_1^k-\textbf{g}_1(\vw^{t+1}),\textbf{u}_1^{t+1}-\textbf{u}_1^{t}\rangle+\frac{1}{2}\|\textbf{u}_1^{t+1}-\textbf{u}_1^{t} \|^2\\
    & = \frac{1}{2}\|\textbf{u}_1^{t}-\textbf{g}_1(\vw^{t+1})\|^2+\sum_{i\in\mathcal{B}^{t+1}}\langle u_{1i}^t - \hat{g}_{1i}(\vw^{t+1}),u_{1i}^{t+1}-u_{1i}^t\rangle + \frac{1}{2}\sum_{i\in\mathcal{B}^{t+1}}\|u_{1i}^{t+1}-u_{1i}^t\|^2\\
    & \quad + \sum_{i\in\mathcal{B}^{t+1}}\langle \hat{g}_{1i}(\vw^{t+1}) - g_{1i}(\vw^{k+1}),u_{1i}^{t+1}-u_{1i}^t\rangle
\end{split}
\end{equation}
Note that $u_{1i}^t-\hat{g}_{1i}(\vw^{t+1})=(u_{1i}^t-u_{1i}^{t+1})/\gamma_1$ and $2\langle b-a, a-c\rangle\leq \|b-c\|^2-\|a-b\|^2-\|a-c\|^2$.
\begin{equation*}
    \begin{split}
        &\sum_{i\in\mathcal{B}^{t+1}}\langle u_{1i}^k - \hat{g}_{1i}(\vw^{t+1}),u_{1i}^{t+1}-u_{1i}^t\rangle\\
        & = \sum_{i\in\mathcal{B}^{t+1}}\langle u_{1i}^t - \hat{g}_{1i}(\vw^{t+1}),g_{1i}(\vw^{k+1})-u_{1i}^t\rangle+\sum_{i\in\mathcal{B}^{t+1}}\langle u_{1i}^t - \hat{g}_{1i}(\vw^{t+1}),u_{1i}^{t+1}-g_{1i}(\vw^{t+1})\rangle\\
        & = \sum_{i\in\mathcal{B}^{t+1}}\langle u_{1i}^t - \hat{g}_{1i}(\vw^{t+1}),g_{1i}(\vw^{t+1})-u_{1i}^t\rangle+\frac{1}{\gamma_1}\sum_{i\in\mathcal{B}^{t+1}}\langle u_{1i}^t - u_{1i}^{t+1},u_{1i}^{t+1}-g_{1i}(\vw^{t+1})\rangle\\
        &\leq \sum_{i\in\mathcal{B}^{t+1}}\langle u_{1i}^t - \hat{g}_{1i}(\vw^{t+1}),g_{1i}(\vw^{t+1})-u_{1i}^t\rangle\\
        &\quad +\frac{1}{2\gamma_1}\sum_{i\in\mathcal{B}^{t+1}}\left(\|u_{1i}^t-g_{1i}(\vw^{t+1})\|^2-\|u_{1i}^{t+1}-u_{1i}^t\|^2-\|u_{1i}^{t+1}-g_{1i}(\vw^{t+1})\|^2\right)
    \end{split}
\end{equation*}
If $\gamma_1 \leq \frac{1}{5}$, we have
\begin{equation*}
    \begin{split}
        &-\frac{1}{2}\left( \frac{1}{\gamma_1}-1-\frac{\gamma_1+1}{4\gamma_1} \right)\sum_{i\in\mathcal{B}^{t+1}}\|u_{1i}^{t+1}-u_{1i}^t\|^2 + \sum_{i\in\mathcal{B}^{t+1}}\langle \hat{g}_{1i}(\vw^{t+1}) - g_{1i}(\vw^{t+1}),u_{1i}^{t+1}-u_{1i}^t\rangle\\
       \leq & -\frac{1}{4\gamma_1} \sum_{i\in\mathcal{B}^{t+1}}\|u_{1i}^{t+1}-u_{1i}^t\|^2+\gamma_1 \sum_{i\in\mathcal{B}^{t+1}}\left\|\hat{g}_{1i}(\vw^{t+1}) - g_{1i}(\vw^{t+1})\right\|^2+\frac{1}{4\gamma_1}\sum_{i\in\mathcal{B}^{t+1}}\|u_{1i}^{t+1}-u_{1i}^t\|^2\\
       = & \gamma_1 \sum_{i\in\mathcal{B}^{t+1}}\left\|\hat{g}_{1i}(\vw^{t+1}) - g_{1i}(\vw^{t+1})\right\|^2.
    \end{split}
\end{equation*}
Then we can get
\begin{equation*}
    \begin{split}
        \frac{1}{2}\|\textbf{u}_1^{t+1}-\textbf{g}_1(\vw^{t+1})\|^2 \leq & \frac{1}{2}\|\textbf{u}_1^{t}-\textbf{g}_1(\vw^{t+1})\|^2+\frac{1}{2\gamma_1}\sum_{i\in\mathcal{B}^{t+1}}\|u_{1i}^t-g_{1i}(\vw^{t+1})\|^2 \\
        & -\frac{1}{2\gamma_1}\sum_{i\in\mathcal{B}^{t+1}}\|u_{1i}^{t+1}-g_{1i}(\vw^{t+1})\|^2\\
        & +\gamma_1 \sum_{i\in\mathcal{B}^{t+1}}\left\|\hat{g}_{1i}(\vw^{t+1}) - g_{1i}(\vw^{t+1})\right\|^2-\frac{\gamma_1+1}{8\gamma_1}\sum_{i\in\mathcal{B}^{t+1}}\|u_{1i}^{t+1}-u_{1i}^t\|^2\\
        & +  \sum_{i\in\mathcal{B}^{t+1}}\langle u_{1i}^t - \hat{g}_{1i}(\vw^{t+1}),g_{1i}(\vw^{t+1})-u_{1i}^t\rangle.
    \end{split}
\end{equation*}
Note that $\frac{1}{2\gamma_1}\sum_{i\notin\mathcal{B}^{t+1}}\|u_{1i}^t-g_{1i}(\vw^{t+1})\|^2=\frac{1}{2\gamma_1}\sum_{i\notin\mathcal{B}^{t+1}}\|u_{1i}^{t+1}-g_{1i}(\vw^{t+1})\|^2$, which implies that 
\begin{equation*}
    \frac{1}{2\gamma_1}\sum_{i\in\mathcal{B}^{t+1}}\left(\|u_{1i}^t-g_{1i}(\vw^{t+1})\|^2 - \|u_{1i}^{t+1}-g_{1i}(\vw^{t+1})\|^2\right)=\frac{1}{2\gamma_1}\left(\|\textbf{u}_1^t-\textbf{g}_1(\vw^{t+1})\|^2-\|\textbf{u}_1^{t+1}-\textbf{g}_1(\vw^{t+1})\|^2\right).
\end{equation*}
Besides, we also have $\E\left[\sum_{i\in\mathcal{B}^{t+1}}\left\|\hat{g}_{1i}(\vw^{t+1}) - g_{1i}(\vw^{t+1})\right\|^2  \right]\leq \frac{|\cB|\sigma_g^2}{|\cB_{1i}|}$ and
\begin{equation*}
    \begin{split}
        \E\left[ \sum_{i\in\mathcal{B}^{t+1}}\langle u_{1i}^t - \hat{g}_{1i}(\vw^{t+1}),g_{1i}(\vw^{t+1})-u_{1i}^t\rangle \right]&=\frac{|\cB|}{n_0}\sum_{i=1}^{n_0}\langle u_{1i}^t - g_{1i}(\vw^{t+1}),g_{1i}(\vw^{t+1})-u_{1i}^t\rangle\\
        & = -\frac{|\cB|}{n_0}\|\textbf{u}_1^{t}-\textbf{g}_1(\vw^{t+1})\|^2.
    \end{split}
\end{equation*}
Then we can obtain
\begin{equation*}
    \begin{split}
        &\left(\frac{1}{2}+\frac{1}{2\gamma_1}\right)\E\left[ \|\textbf{u}_1^{t+1}-\textbf{g}_1(\vw^{t+1})\|^2 \right]\\
        \leq & \left(\frac{1}{2}+\frac{1}{2\gamma_1}-\frac{|\cB|}{n_0}\right)\E\left[ \|\textbf{u}_1^{t}-\textbf{g}_1(\vw^{t+1})\|^2 \right]+\frac{\gamma_1 |\cB|\sigma_g^2}{|\cB_{1i}|}-\frac{\gamma_1+1}{8\gamma_1}\E\left[\sum_{i\in\mathcal{B}^{t+1}}\|u_{1i}^{t+1}-u_{1i}^t\|^2\right].
    \end{split}
\end{equation*}
Divide both sides by $\frac{\gamma_1+1}{2\gamma_1}$ we can get
\begin{equation*}
    \begin{split}
        \E\left[ \|\textbf{u}_1^{t+1}-\textbf{g}_1(\vw^{t+1})\|^2 \right]
        \leq  \frac{\gamma_1+1-2\gamma_1\frac{|\cB|}{n_0}}{\gamma_1+1}\E\left[ \|\textbf{u}_1^{t}-\textbf{g}_1(\vw^{t+1})\|^2 \right]&+\frac{2}{\gamma_1+1}\frac{\gamma_1^2 |\cB|\sigma_g^2}{|\cB_{1i}|} \\
        &-\frac{1}{4}\E\left[\sum_{i\in\mathcal{B}^{t+1}}\|u_{1i}^{t+1}-u_{1i}^t\|^2\right].
    \end{split}
\end{equation*}
Note that $\frac{\gamma_1+1-2\gamma_1\frac{|\cB|}{n_0}}{\gamma_1+1}\leq \frac{\gamma_1(1-\frac{|\cB|}{n_0})+1}{\gamma_1+1}=1-\frac{\gamma_1 |\cB|}{(\gamma_1+1)n_0}\leq 1-\frac{\gamma_1 |\cB|}{2n_0} $ and $\frac{1}{\gamma_1+1}\leq 1$ for $\gamma_1\in (0,1]$. Besides, we have $\|\textbf{u}_1^{t}-\textbf{g}_1(\vw^{t+1})\|^2\leq (1+\frac{\gamma_1 |\cB|}{4n_0})\|\textbf{u}_1^{t}-\textbf{g}_1(\vw^{t})\|^2+(1+\frac{4n_0}{\gamma_1 |\cB|})\|\textbf{g}_1(\vw^{t+1})-\textbf{g}_1(\vw^{t})\|^2$ due to Young's inequality, $(1+\frac{\gamma_1 |\cB|}{4n_0})(1-\frac{\gamma_1 |\cB|}{2n_0})\leq (1-\frac{\gamma_1 |\cB|}{4n_0})$ and $(1+\frac{4n_0}{\gamma_1 |\cB|})(1-\frac{\gamma_1 |\cB|}{2n_0})\leq \frac{5n_0}{\gamma_1 |\cB|}$.
\begin{equation*}
    \begin{split}
        &\E\left[ \Xi^{t+1}_1 \right]=\E\left[\frac{1}{n_0} \|\textbf{u}_1^{t+1}-\textbf{g}_1(\vw^{t+1})\|^2 \right]\\
        &\leq \left(1-\frac{\gamma_1 |\cB|}{4n_0} \right)\E\left[\frac{1}{n_0} \|\textbf{u}_1^{t}-\textbf{g}_1(\vw^{t})\|^2 \right]+\frac{5n_0L_g^2\|\vw^{t+1}-\vw^t\|^2}{\gamma_1 |\cB|}+\frac{2\gamma_1^2\sigma_g^2 |\cB|}{n_0|\cB_{1i}|}-\frac{1}{4n_0}\E\left[\sum_{i\in\mathcal{B}^{t+1}}\|u_{1i}^{t+1}-u_{1i}^t\|^2\right]\\
        & = \left(1-\frac{\gamma_1 |\cB|}{4n_0} \right)\E\left[\Xi^t_1\right]+\frac{5n_0L_g^2\E[\|\vw^{t+1}-\vw^t\|^2]}{\gamma_1 |\cB|}+\frac{2\gamma_1^2\sigma_g^2 |\cB|}{n_0|\cB_{1i}|}-\frac{1}{4n_0}\E\left[\sum_{i\in\mathcal{B}^{t+1}}\|u_{1i}^{t+1}-u_{1i}^t\|^2\right]
    \end{split}
\end{equation*}

\end{proof}

\begin{lemma}
 The gradient variance $\Delta_2^t:=\|v_2^{t} - \nabla H(\vw^t)\|^2$ can be bounded as
 \begin{equation}\label{eqn:lemma3}
 \begin{split}
      \E [\Delta_2^{t+1}] \leq & (1-\theta)\E [\Delta_2^{t}] + \frac{2 \beta^2L_H^2}{\theta}\E\left[\|\vw^{t+1}-\vw^t\|^2\right]+5\theta\beta^2\Tilde{C}_{\nabla h}^2\E[\Gamma_{t+1}]\\
    &+\frac{3\beta^2\Tilde{C}_{\nabla h}^2}{m}\E\left[ \sum_{k\in\mathcal{B}_c^{t+1}}\|u_k^{t+1}-u_k^t\|^2 \right]
    +\frac{\theta^2\beta^2 C_h^2(\sigma_{\nabla h}^2+L_h^2)}{\min\{|\cB_c|,|\cB_k|\}}
 \end{split}
 \end{equation}
 with $\Gamma_{t+1}:=\frac{1}{m}\|\textbf{u}^{t+1}-\textbf{h}(\vw^{t+1})\|^2$.
\end{lemma}

\begin{proof}
\begin{equation*}
    \begin{split}
        \Delta_2^{t+1}&=\|v_2^{t+1}-\nabla H(\vw^{t+1})\|^2=\|(1-\theta)v_2^{t}+\theta G_2^{t}-\nabla H(\vw^{t+1})\|^2\\
        & \|\circled{1}+\circled{2}+\circled{3}+\circled{4}\|^2,
    \end{split}
\end{equation*}
where $\circled{1}$, $\circled{2}$, $\circled{3}$ and $\circled{4}$ are defined as
\begin{equation*}
    \begin{split}
        &\circled{1} = (1-\theta)(v_2^{t}-\nabla H(\vw^t)), \quad \circled{2} = (1-\theta)(\nabla H(\vw^t)-\nabla H(\vw^{t+1})),\\
        & \circled{3} = \frac{\theta}{|\cB_c|}\beta\sum_{k\in\mathcal{B}_c^{t+1}}\left( [u_k^t]_+\nabla \hat{h}_k(\vw^{t+1})-[h_k(\vw^{t+1})]_+\nabla \hat{h}_k(\vw^{t+1})\right)\\
        & \circled{4} = \theta\left( \frac{1}{|\cB_c|}\beta\sum_{k\in\mathcal{B}_c^{t+1}} [h_k(\vw^{t+1})]_+\nabla \hat{h}_k(\vw^{t+1}) - \nabla H(\vw^{t+1})\right)
    \end{split}
\end{equation*}
Note that $\E_t[\langle \circled{1},  \circled{4}\rangle]=\E_t[\langle \circled{2},  \circled{4}\rangle]=0$. Then, by the Young's inequality, we can get
\begin{equation*}
    \begin{split}
        &\E_t\left[ \| \circled{1}+\circled{2}+\circled{3}+\circled{4}\|^2 \right]\\
        =& \| \circled{1}\|^2+\|\circled{2}\|^2+\E_t \|\circled{3}\|^2+\E_t \|\circled{4}\|^2+2\langle\circled{1},\circled{2} \rangle+2\E_t[\langle\circled{1},\circled{3}  \rangle]+2\E_t[\langle\circled{2},\circled{3}  \rangle]+2\E_t[\langle\circled{3},\circled{4}  \rangle]\\
        \leq & (1+\theta)\| \circled{1}\|^2+2\left(1+\frac{1}{\theta}\right)\|\circled{2}\|^2+\frac{2+3\theta}{\theta}\E_t \|\circled{3}\|^2+2\E_t \|\circled{4}\|^2.
    \end{split}
\end{equation*}
We can also get
\begin{equation*}
    \begin{split}
        (1+\theta)\| \circled{1}\|^2=(1+\theta)(1-\theta)^2\|v_2^t-\nabla H(\vw^t)\|^2\leq (1-\theta)\|v_2^t-\nabla H(\vw^t)\|^2
    \end{split}
\end{equation*}
\begin{equation*}
    \begin{split}
        2\left(1+\frac{1}{\theta}\right)\| \circled{2}\|^2=&2\left(1+\frac{1}{\theta}\right)(1-\theta)^2\|\nabla H(\vw^t)-\nabla H(\vw^{t+1})\|^2\\
        \leq & \frac{2}{\theta}\left\| \frac{1}{m}\sum_{k=1}^m \beta\left( \nabla h_k(\vw^{t+1})^\top [h_k(\vw^{t+1})]_+- \nabla h_k(\vw^{t})^\top [h_k(\vw^{t})]_+\right) \right\|^2 \\
        \leq & \frac{2\beta^2L_H^2}{\theta}\|\vw^{t+1}-\vw^t\|^2
    \end{split}
\end{equation*}

\begin{equation*}
    \begin{split}
        \frac{2+3\theta}{\theta} \|\circled{3}\|^2 &\leq \frac{2+3\theta}{\theta}\frac{\theta^2\beta^2}{|\cB_c|}\sum_{k\in\mathcal{B}_c^{t+1}}\|\nabla \hat{h}_k(\vw^{t+1})\|^2\|[u_k^t]_+-[h_k(\vw^{t+1})]_+\|^2\\
        &\leq \frac{(2+3\theta)\theta\beta^2}{|\cB_c|} \sum_{k\in\mathcal{B}_c^{t+1}}\|\nabla \hat{h}_k(\vw^{t+1})\|^2\|u_k^t-h_k(\vw^{t+1})\|^2
    \end{split}
\end{equation*}
Consider that $\vw^{t+1}$ and $u_k^t$ do not depend on either $\mathcal{B}_c^{t+1}$ or $\mathcal{B}_k$, we have
\begin{small}
\begin{equation*}
    \begin{split}
        &(2+3\theta)\theta\beta^2 \E_t \left[ \frac{1}{|\cB_c|}\sum_{k\in\mathcal{B}_c^{t+1}}\|\nabla \hat{h}_k(\vw^{t+1})\|^2\|u_k^t-h_k(\vw^{t+1})\|^2 \right]\\
        & = (2+3\theta)\theta\beta^2 \E_t \left[ \frac{1}{|\cB_c|}\sum_{k\in\mathcal{B}_c^{t+1}}\E_t\left[\|\nabla \hat{h}_k(\vw^{t+1})\|^2|k\in\mathcal{B}_c^{t+1} \right]\|u_k^t-h_k(\vw^{t+1})\|^2 \right]\\
        &\leq (2+3\theta)\theta\beta^2\Tilde{C}_{\nabla h}^2 \E_t \left[ \frac{1}{|\cB_c|} \sum_{k\in\mathcal{B}_c^{t+1}} \|u_k^t-h_k(\vw^{t+1})\|^2  \right]\\
        & \leq \frac{(2+3\theta)\theta(1+\delta)\beta^2\Tilde{C}_{\nabla h}^2}{m}\sum_{k\in[m]}\E_t\left[ \|u_k^{t+1}-h_k(\vw^{t+1})\|^2\right]+\frac{(2+3\theta)\theta(1+1/\delta)\beta^2\Tilde{C}_{\nabla h}^2}{m}\E_t\left[ \sum_{k\in[m]}\|u_k^{t+1}-u_k^t\|^2 \right]\\
        & = \frac{(2+3\theta)\theta(1+\delta)\beta^2\Tilde{C}_{\nabla h}^2}{m}\sum_{k\in[m]}\E_t\left[ \|u_k^{t+1}-h_k(\vw^{t+1})\|^2\right]+\frac{(2+3\theta)\theta(1+1/\delta)\beta^2\Tilde{C}_{\nabla h}^2}{m}\E_t\left[ \sum_{k\in\mathcal{B}_c^{t+1}}\|u_k^{t+1}-u_k^t\|^2 \right]
    \end{split}
\end{equation*}
\end{small}
where the last equation holds by noting that $u_k^{t+1}=u_k^t$ for all $i\notin \mathcal{B}_c^{t+1}$.

If $\theta\leq \frac{1}{3}$ and $\delta=\frac{3\theta}{2}$, we have $(2+3\beta)\beta(1+\delta)\leq 5\theta$ and $(2+3\beta)\beta(1+1/\delta)\leq 3$. Therefore, we can get
\begin{equation*}
     \E\left[\frac{2+3\theta}{\theta} \|\circled{3}\|^2\right]\leq 5\theta\beta^2\Tilde{C}_{\nabla h}^2\E[\Gamma_{t+1}]+\frac{3\beta^2\Tilde{C}_{\nabla h}^2}{m}\E\left[ \sum_{k\in\mathcal{B}_c^{t+1}}\|u_k^{t+1}-u_k^t\|^2 \right]
\end{equation*}
Next, we give the upper bound of $\E_t\|\circled{4}\|^2$.
\begin{equation*}
    \begin{split}
        \E_t\|\circled{4}\|^2=&\theta^2\beta^2\E_k\left[\left\| \frac{1}{|\cB_c|}\sum_{k\in\mathcal{B}_c^{t+1}}[h_k(\vw^{t+1})]_+\nabla \hat{h}_k(\vw^{t+1})-\frac{1}{m}\sum_{k=1}^m[h_k(\vw^{t+1})]_+\nabla h_k(\vw^{t+1})\right\|^2 \right]\\
        \leq& \theta^2 \beta^2 \E_t\left[\left\| \frac{1}{|\cB_c|}\sum_{k\in\mathcal{B}_c^{t+1}}[h_k(\vw^{t+1})]_+\nabla \hat{h}_k(\vw^{t+1})-\frac{1}{|\cB_c|}\sum_{k\in\mathcal{B}_c^{t+1}}[h_k(\vw^{t+1})]_+\nabla h_k(\vw^{t+1})\right\|^2 \right]\\
        & + \theta^2 \beta^2 \E_t\left[\left\| \frac{1}{|\cB_c|}\sum_{k\in\mathcal{B}_c^{t+1}}[h_k(\vw^{t+1})]_+\nabla h_k(\vw^{t+1})-\frac{1}{m}\sum_{k=1}^m[h_k(\vw^{t+1})]_+\nabla h_k(\vw^{t+1})\right\|^2 \right]\\
        \leq & \frac{\theta^2\beta^2 C_h^2(\sigma_{\nabla h}^2+L_h^2)}{\min\{|\cB_c|,|\cB_k|\}}
    \end{split}
\end{equation*}

Combine above inequalities, we can get
\begin{equation*}
\begin{split}
    \E [\Delta_2^{t+1}] \leq & (1-\theta)\E [\Delta_2^{t}] + \frac{2 \beta^2L_H^2}{\theta}\E\left[\|\vw^{t+1}-\vw^t\|^2\right]+5\theta\beta^2\Tilde{C}_{\nabla h}^2\E[\Gamma_{t+1}]\\
    &+\frac{3\beta^2\Tilde{C}_{\nabla h}^2}{m}\E\left[ \sum_{k\in\mathcal{B}_c^{t+1}}\|u_k^{t+1}-u_k^{t}\|^2 \right]
    +\frac{\theta^2\beta^2 C_h^2(\sigma_{\nabla h}^2+L_h^2)}{\min\{|\cB_c|,|\cB_k|\}}.
\end{split}
\end{equation*}
\end{proof}

\begin{lemma}\label{thm:lemma4}
    If $\gamma_2\leq 1/5$, function value variance $\Gamma_t:=\frac{1}{m}\|\textbf{u}^t-\textbf{h}(\vw^t)\|^2$ can be bounded as
    \begin{equation}\label{eqn:lemma4}
        \E [\Gamma_{t+1}]\leq \left(1-\frac{\gamma_2 |\cB_c|}{4m} \right)\E\left[\Gamma_t\right]+\frac{5mL_h^2\E[\|\vw^{t+1}-\vw^t\|^2]}{\gamma |\cB_c|}+\frac{2\gamma_2^2\sigma_h^2 |\cB_c|}{m|\cB_k|}-\frac{1}{4m}\E\left[\sum_{k\in\mathcal{B}_c^{t+1}}\|u_k^{t+1}-u_k^t\|^2\right].
    \end{equation}
\end{lemma}

\begin{proof}
Define $\psi_k(\textbf{u})=\frac{1}{2}\|\textbf{u}-\textbf{h}(\vw^t)\|^2=\frac{1}{2}\sum^m_{k=1}\|u_k-h_k(\vw^t)\|^2$, which is $1$-strongly convex.
\begin{equation}
\begin{split}
    \psi_{t+1}(\textbf{u}^{t+1})&=\frac{1}{2}\|\textbf{u}^{t+1}-\textbf{h}(\vw^{t+1})\|^2 = \frac{1}{2}\|\textbf{u}^{t}-\textbf{h}(\vw^{t+1})\|^2+\langle \textbf{u}^t-\textbf{h}(\vw^{t+1}),\textbf{u}^{t+1}-\textbf{u}^{t}\rangle+\frac{1}{2}\|\textbf{u}^{t+1}-\textbf{u}^{t} \|^2\\
    & = \frac{1}{2}\|\textbf{u}^{t}-\textbf{h}(\vw^{t+1})\|^2+\sum_{k\in\mathcal{B}_c^{t+1}}\langle u_k^t - \hat{h}_k(\vw^{t+1}),u_k^{t+1}-u_k^t\rangle + \frac{1}{2}\sum_{k\in\mathcal{B}_c^{t+1}}\|u_k^{t+1}-u_k^t\|^2\\
    & \quad + \sum_{k\in\mathcal{B}_c^{t+1}}\langle \hat{h}_k(\vw^{t+1}) - h_k(\vw^{t+1}),u_k^{t+1}-u_k^t\rangle
\end{split}
\end{equation}
Note that $u_k^t-\hat{h}_k(\vw^{t+1})=(q_i^k-q_i^{k+1})/\gamma_2$ and $2\langle b-a, a-c\rangle\leq \|b-c\|^2-\|a-b\|^2-\|a-c\|^2$.
\begin{equation*}
    \begin{split}
        &\sum_{k\in\mathcal{B}_c^{t+1}}\langle u_k^t - \hat{h}_k(\vw^{t+1}),u_k^{t+1}-u_k^t\rangle\\
        & = \sum_{k\in\mathcal{B}_c^{t+1}}\langle u_k^t - \hat{h}_k(\vw^{t+1}),h_k(\vw^{t+1})-u_k^t\rangle+\sum_{k\in\mathcal{B}_c^{t+1}}\langle u_k^t - \hat{h}_k(\vw^{t+1}),u_k^{t+1}-h_k(\vw^{t+1})\rangle\\
        & = \sum_{k\in\mathcal{B}_c^{t+1}}\langle u_k^t - \hat{h}_k(\vw^{t+1}),h_k(\vw^{t+1})-u_k^t\rangle+\frac{1}{\gamma_2}\sum_{k\in\mathcal{B}_c^{t+1}}\langle u_k^t - u_k^{t+1},u_k^{t+1}-h_k(\vw^{t+1})\rangle\\
        &\leq \sum_{k\in\mathcal{B}_c^{t+1}}\langle u_k^t - \hat{h}_k(\vw^{t+1}),h_k(\vw^{t+1})-u_k^t\rangle\\
        &\quad +\frac{1}{2\gamma_2}\sum_{k\in\mathcal{B}_c^{t+1}}\left(\|u_k^t-h_k(\vw^{t+1})\|^2-\|u_k^{t+1}-u_k^t\|^2-\|u_k^{t+1}-h_k(\vw^{t+1})\|^2\right)
    \end{split}
\end{equation*}
If $\gamma_2 \leq \frac{1}{5}$, we have
\begin{equation*}
    \begin{split}
        &-\frac{1}{2}\left( \frac{1}{\gamma_2}-1-\frac{\gamma_2+1}{4\gamma_2} \right)\sum_{k\in\mathcal{B}_c^{t+1}}\|u_k^{t+1}-u_k^t\|^2 + \sum_{k\in\mathcal{B}_c^{t+1}}\langle \hat{h}_k(\vw^{t+1}) - h_k(\vw^{t+1}),u_k^{t+1}-u_k^t\rangle\\
       \leq & -\frac{1}{4\gamma_2} \sum_{k\in\mathcal{B}_c^{t+1}}\|u_k^{t+1}-u_k^t\|^2+\gamma_2 \sum_{k\in\mathcal{B}_c^{t+1}}\left\|\hat{h}_k(\vw^{t+1}) - h_k(\vw^{t+1})\right\|^2+\frac{1}{4\gamma_2}\sum_{k\in\mathcal{B}_c^{t+1}}\|u_k^{t+1}-u_k^t\|^2\\
       = & \gamma_2 \sum_{k\in\mathcal{B}_c^{t+1}}\left\|\hat{h}_k(\vw^{t+1}) - h_k(\vw^{t+1})\right\|^2.
    \end{split}
\end{equation*}
Then we can get
\begin{equation*}
    \begin{split}
        \frac{1}{2}\|\textbf{u}^{t+1}-\textbf{h}(\vw^{t+1})\|^2 \leq & \frac{1}{2}\|\textbf{u}^{t}-\textbf{h}(\vw^{t+1})\|^2+\frac{1}{2\gamma_2}\sum_{k\in\mathcal{B}_c^{t+1}}\|u_k^t-h_k(\vw^{t+1})\|^2-\frac{1}{2\gamma_2}\sum_{k\in\mathcal{B}_c^{t+1}}\|u_k^{t+1}-h_k(\vw^{t+1})\|^2\\
        & +\gamma_2 \sum_{k\in\mathcal{B}_c^{t+1}}\left\|\hat{h}_k(\vw^{t+1}) - h_k(\vw^{t+1})\right\|^2-\frac{\gamma_2+1}{8\gamma_2}\sum_{k\in\mathcal{B}_c^{t+1}}\|u_k^{t+1}-u_k^t\|^2\\
        & +  \sum_{k\in\mathcal{B}_2^{t+1}}\langle u_k^t - \hat{h}_k(\vw^{t+1}),h_k(\vw^{t+1})-u_k^t\rangle.
    \end{split}
\end{equation*}
Note that $\frac{1}{2\gamma_2}\sum_{k\notin\mathcal{B}_c^{t+1}}\|u_k^t-h_k(\vw^{t+1})\|^2=\frac{1}{2\gamma_2}\sum_{k\notin\mathcal{B}_c^{t+1}}\|u_k^{t+1}-h_k(\vw^{t+1})\|^2$, which implies that 
\begin{equation*}
    \frac{1}{2\gamma_2}\sum_{k\in\mathcal{B}_c^{t+1}}\left(\|u_k^t-h_k(\vw^{t+1})\|^2 - \|u_k^{t+1}-h_k(\vw^{t+1})\|^2\right)=\frac{1}{2\gamma_2}\left(\|\textbf{u}^t-\textbf{h}(\vw^{t+1})\|^2-\|\textbf{u}^{t+1}-\textbf{h}(\vw^{t+1})\|^2\right).
\end{equation*}
Besides, we also have $\E\left[\sum_{k\in\mathcal{B}_c^{t+1}}\left\|\hat{h}_k(\vw^{t+1}) - h_k(\vw^{t+1})\right\|^2  \right]\leq \frac{|\cB_c|\sigma_h^2}{|\cB_k|}$ and
\begin{equation*}
    \begin{split}
        \E\left[ \sum_{k\in\mathcal{B}_c^{t+1}}\langle u_k^t - \hat{h}_k(\vw^{t+1}),h_k(\vw^{t+1})-u_k^t\rangle \right]&=\frac{|\cB_c|}{m}\sum_{k=1}^m\langle u_k^t - h_k(\vw^{t+1}),h_k(\vw^{t+1})-u_k^t\rangle\\
        & = -\frac{|\cB_c|}{m}\|\textbf{u}^{t}-\textbf{h}(\vw^{t+1})\|^2.
    \end{split}
\end{equation*}
Then we can obtain
\begin{equation*}
    \begin{split}
        &\left(\frac{1}{2}+\frac{1}{2\gamma_2}\right)\E\left[ \|\textbf{u}^{t+1}-\textbf{h}(\vw^{t+1})\|^2 \right]\\
        \leq & \left(\frac{1}{2}+\frac{1}{2\gamma_2}-\frac{|\cB_c|}{m}\right)\E\left[ \|\textbf{u}^{t}-\textbf{h}(\vw^{t+1})\|^2 \right]+\frac{\gamma_2 |\cB_c|\sigma_h^2}{|\cB_k|}-\frac{\gamma_2+1}{8\gamma_2}\E\left[\sum_{k\in\mathcal{B}_c^{t+1}}\|u_k^{t+1}-u_k^t\|^2\right].
    \end{split}
\end{equation*}
Divide both sides by $\frac{\gamma_2+1}{2\gamma_2}$ we can get
\begin{equation*}
    \begin{split}
        \E\left[ \|\textbf{u}^{t+1}-\textbf{h}(\vw^{t+1})\|^2 \right]
        \leq  \frac{\gamma_2+1-2\gamma_2\frac{|\cB_c|}{m}}{\gamma_2+1}\E\left[ \|\textbf{u}^{t}-\textbf{h}(\vw^{t+1})\|^2 \right]&+\frac{2}{\gamma_2+1}\frac{\gamma_2^2 |\cB_c|\sigma_h^2}{|\cB_k|}\\
        &-\frac{1}{4}\E\left[\sum_{k\in\mathcal{B}_c^{t+1}}\|u_k^{t+1}-u_k^t\|^2\right].
    \end{split}
\end{equation*}
Note that $\frac{\gamma_2+1-2\gamma_2\frac{|\cB_c|}{m}}{\gamma_2+1}\leq \frac{\gamma_2(1-\frac{|\cB_c|}{m})+1}{\gamma_2+1}=1-\frac{\gamma_2 |\cB_c|}{(\gamma_2+1)m}\leq 1-\frac{\gamma_2 |\cB_c|}{2m} $ and $\frac{1}{\gamma_2+1}\leq 1$ for $\gamma_2\in (0,1]$. Besides, we have $\|\textbf{u}^{t}-\textbf{h}(\vw^{t+1})\|^2\leq (1+\frac{\gamma_2 |\cB_c|}{4m})\|\textbf{u}^{t}-\textbf{h}(\vw^{t})\|^2+(1+\frac{4m}{\gamma_2 |\cB_c|})\|\textbf{h}(\vw^{t+1})-\textbf{h}(\vw^{t})\|^2$ due to Young's inequality, $(1+\frac{\gamma_2 |\cB_c|}{4m})(1-\frac{\gamma_2 |\cB_c|}{2m})\leq (1-\frac{\gamma_2 |\cB_c|}{4m})$ and $(1+\frac{4m}{\gamma_2 |\cB_c|})(1-\frac{\gamma_2 |\cB_c|}{2m})\leq \frac{5m}{\gamma_2 |\cB_c|}$.
\begin{equation*}
    \begin{split}
        &\E\left[ \Gamma_{t+1} \right]=\E\left[\frac{1}{m} \|\textbf{u}^{t+1}-\textbf{h}(\vw^{t+1})\|^2 \right]\\
        &\leq \left(1-\frac{\gamma_2 |\cB_c|}{4m} \right)\E\left[\frac{1}{m} \|\textbf{u}^{t}-\textbf{h}(\vw^{t})\|^2 \right]+\frac{5mL_h^2\|\vw^{t+1}-\vw^t\|^2}{\gamma_2 |\cB_c|}+\frac{2\gamma_2^2\sigma_h^2 |\cB_c|}{m|\cB_k|}-\frac{1}{4m}\E\left[\sum_{k\in\mathcal{B}_c^{t+1}}\|u_k^{t+1}-u_k^t\|^2\right]\\
        & = \left(1-\frac{\gamma_2 |\cB_c|}{4m} \right)\E\left[\Gamma_t\right]+\frac{5mL_h^2\E[\|\vw^{t+1}-\vw^t\|^2]}{\gamma_2 |\cB_c|}+\frac{2\gamma_2^2\sigma_h^2 |\cB_c|}{m|\cB_k|}-\frac{1}{4m}\E\left[\sum_{k\in\mathcal{B}_c^{t+1}}\|u_k^{t+1}-u_k^t\|^2\right]
    \end{split}
\end{equation*}

\end{proof}

We state the main theorem again for convenience and present the proof.
\begin{thm*}
    Suppose Assumptions~\ref{assumption1}, \ref{assumption2}, \ref{assumption3} and \ref{ass:full_rank} hold, and set $\beta=\frac{1}{\epsilon\delta}$,\\ $\theta=\min\{\frac{\epsilon^4\delta^2\min\{|\cB_k|,|\cB_c|\}}{672(\sigma_{\nabla h}^2+L_h^2)} ,\frac{\epsilon^2\min\{|\cB|,|\cB_{1i}|,|\cB_{2i}|\}}{1344L_f^2(\sigma_{\nabla g}^2+L_g^2)}\}$, $\gamma_1=\gamma_2 = \min\{\frac{5n_0\theta}{3|\cB|}, \frac{5m\theta}{3|\cB_c|},\frac{\epsilon^4\delta^2|\cB_k|}{26880 \sigma_h^2 \Tilde{C}_{\nabla h}^2}\}$ and \\$\eta = \min \left\{ \frac{1}{12(L_F+\beta L_H)}, \frac{\theta}{8\sqrt{3}L_F}, \frac{\theta}{8\sqrt{3}L_H \beta},\frac{\gamma_1 |\cB|}{40\sqrt{6}L_g L_f \Tilde{C}_{\nabla g} n_0},\frac{\gamma_2 |\cB_c|}{40\sqrt{6}\beta L_h\Tilde{C}_{\nabla h} m} \right\}$. Then there exists $\boldsymbol{\lambda}$ such that
    \begin{equation*}
        \begin{split}
           &\mathbb{E}\left[ \|\nabla F(\vw^{\hat{t}})+\nabla \textbf{h}(\vw^{\hat{t}})\boldsymbol{\lambda})\| \right]\leq \epsilon\\
    & \mathbb{E} [\|[\textbf{h}(\vw^{\hat{t}})]_+\|] \leq \epsilon\\
    & \mathbb{E} [\boldsymbol{\lambda}^\top [\textbf{h}(\vw^{\hat{t}})]_+]\leq \epsilon
        \end{split}
    \end{equation*}
with number of iterations $T$ of Algorithm~\ref{alg:clip_class} bounded by $O(\epsilon^{-7}\delta^{-3})$ and $\hat{t}$ selected uniformly at random from $\{1,\cdots,T\}$.
\end{thm*}

\begin{proof}
    Since $\Phi(\vw)$ is $L_{\beta}$-smooth with $L_{\beta}= L_F + \beta L_H$ where $L_F := 2(L_{\nabla g}L_f+L_{\nabla f}L_g^2)$ and $L_H := L_{\nabla h} C_h + L_h C_{\nabla h}$, we have
\begin{equation}
\begin{split}
    \Phi(\vw^{t+1})&\leq \Phi(\vw^{t})+\langle \nabla \Phi(\vw^{t}), \vw^{t+1}-\vw^t \rangle + \frac{L_{\beta}}{2}\|\vw^{t+1}-
    \vw^t\|^2\\
    & = \Phi(\vw^{t})+\langle v^t, \vw^{t+1}-\vw^t \rangle+\langle \nabla \Phi(\vw^{t})-v^t, \vw^{t+1}-\vw^t \rangle + \frac{L_{\beta}}{2}\|\vw^{t+1}-
    \vw^t\|^2\\
    & \leq \Phi(\vw^{t})+\langle v^t, \vw^{t+1}-\vw^t \rangle + \left(\frac{L_{\beta}}{2}+\frac{1}{4\eta}\right)\|\vw^{t+1}-\vw^t\|^2 + \eta\|\nabla \Phi(\vw^{t})-v^t\|^2.
\end{split}
\end{equation}
Since $\vw^{t+1}=\vw^t-\eta v^t$, which is equivalent to $\vw^{t+1}=\argmin_{\vw} \langle v^t,\vw \rangle+\frac{1}{2\eta}\|\vw-\vw^t\|^2$, we have
\begin{equation}
    \langle v^t,\vw^{t+1}-\vw^t\rangle \leq -\frac{1}{2\eta}\|\vw^{t+1}-\vw^t\|^2.
\end{equation}
Then we can get
\begin{equation}
    \begin{split}
        \Phi(\vw^{t+1})&\leq \Phi(\vw^{t})+\left(\frac{L_{\beta}}{2}-\frac{1}{4\eta}\right)\|\vw^{t+1}-\vw^t\|^2 + \eta\|\nabla \Phi(\vw^{t})-v^t\|^2
    \end{split}
\end{equation}
\begin{equation}
    \begin{split}
        \Phi(\vw^{t+1})&\leq\Phi(\vw^{t})+\left(\frac{L_{\beta}}{2}-\frac{1}{4\eta}\right)\|\vw^{t+1}-\vw^t\|^2 + 2\eta\|\nabla \Phi(\vw^{t})-v^t\|^2 -\eta\|\nabla \Phi(\vw^{t})-v^t\|^2
    \end{split}
\end{equation}
\begin{equation}\label{eqn:total}
    \begin{split}
        \eta\|\nabla \Phi(\vw^{t})-v^t\|^2\leq \Phi(\vw^{t})-\Phi(\vw^{t+1})+\left(\frac{L_{\beta}}{2}-\frac{1}{4\eta}\right)\|\vw^{t+1}-\vw^t\|^2 + 2\eta\|\nabla \Phi(\vw^{t})-v^t\|^2. 
    \end{split}
\end{equation}

Then we want to bound $\mathbb{E}\|\nabla \Phi(\vw^{t})-v^t\|^2$.

\begin{equation}
\begin{split}
    \|\nabla \Phi(\vw^{t})-v^t\|^2 &= \| (1-\theta)(v_1^{t-1}+v_2^{t-1})+\theta(G_1^t+G_2^t)-\nabla \Phi(\vw^{t}) \|^2\\
    & = \| (1-\theta)v_1^{t-1}+\theta G_1^t-\nabla F(\vw^{t})+ (1-\theta)v_2^{t-1}+\theta G_2^t-\nabla H(\vw^{t})\|^2\\
    & =\| v_1^t - \nabla F(\vw^{t}) +  v_2^t - \nabla H(\vw^{t}) \|^2
\end{split}
\end{equation}
Since $ \E_t \left[\langle v_1^t - \nabla F(\vw^{t}), v_2^t - \nabla H(\vw^{t})\rangle\right] = 0 $, we have
\begin{equation}
    \E_t\|\nabla \Phi(\vw_{t+1})-v_{t+1}\|^2 = \E_t\| v_1^{t+1} - \nabla F(\vw^{t+1})\|^2+\E_t\|v_2^{t+1} - \nabla H(\vw^{t+1}) \|^2
\end{equation}

Summing (\ref{eqn:lemma1}), $\frac{20\theta L_f^2\Tilde{C}_{\nabla g}^2 n_0}{\gamma_1 |\cB|}\times$(\ref{eqn:lemma2}) and $\frac{20\theta L_f^2\Tilde{C}_{\nabla g}^2 n_0}{\gamma_1 |\cB|}\times$(\ref{eqn:lemma_g2}), we can get
\begin{equation}\label{eqn:delta_1}
    \begin{split}
        \mathbb{E}[\Delta_1^{t+1}]\leq& (1-\theta)\mathbb{E}[\Delta_1^{t}]+\left( \frac{2L_{F}^2}{\theta}+\frac{100\theta L_g^2 L_f^2 \Tilde{C}_{\nabla g}^2 n_0^2}{\gamma_1^2 |\cB|^2} \right)\E[\|\vw^{t+1}-\vw^t\|^2]\\
        & + \frac{20\theta L_f^2 \Tilde{C}_{\nabla g}^2n_0}{\gamma_1 |\cB|}\left( 1-\frac{\gamma_1 |\cB|}{4n_0} \right) \E\left[\Xi_1^t+\Xi_2^t - \Xi^{t+1}_1-\Xi^{t+1}_2\right]\\
        & - L_f^2 \Tilde{C}_{\nabla g}^2\left(\frac{5\theta n_0}{\gamma_1 |\cB|}-3 \right)\mathbb{E}\left[\frac{1}{n_0}\sum_{i\in \mathcal{B}^{t+1}}\left\|u_{1i}^{t+1}-u_{1i}^t\right\|^2+\left\|u_{2i}^{t+1}-u_{2i}^t\right\|^2\right]\\
        & + \frac{2\theta^2 L_f^2(\sigma_{\nabla g}^2+L_g^2)}{\text{min}\{|\cB|,|\cB_{1i}|,|\cB_{2i}|\}}+\frac{80\theta\gamma_1\sigma_g^2 L_f^2 \Tilde{C}_{\nabla g}^2}{\min\{|\cB_{1i}|,|\cB_{2i}|\}}
    \end{split}
\end{equation}
Summing (\ref{eqn:lemma3}) and $\frac{20\theta\beta^2\Tilde{C}_{\nabla h}^2 m}{\gamma_2 |\cB_c|}\times$(\ref{eqn:lemma4}), we can get
\begin{equation}\label{eqn:delta_2}
    \begin{split}
        \mathbb{E}[\Delta_2^{t+1}]\leq& (1-\theta)\mathbb{E}[\Delta_2^{t}]+\left( \frac{2\beta^2L_H^2}{\theta}+\frac{100\theta \beta^2L_h^2 \Tilde{C}_{\nabla h}^2 m^2}{\gamma_2^2|\cB_c|^2} \right)\E[\|\vw^{t+1}-\vw^t\|^2]\\
        & + \frac{20\theta \beta^2 \Tilde{C}_{\nabla h}^2m}{\gamma_2 |\cB_c|}\left( 1-\frac{\gamma_2 |\cB_c|}{4m} \right) \E\left[\Gamma_t - \Gamma_{t+1}\right]\\
        & - \beta^2 \Tilde{C}_{\nabla g}^2\left(\frac{5\theta m}{\gamma_2 |\cB_c|}-3 \right)\mathbb{E}\left[\frac{1}{m}\sum_{k\in \mathcal{B}_c^{t+1}}\left\|u_k^{t+1}-u_k^t\right\|^2\right]\\
        & + \frac{\theta^2 \beta^2(\sigma_{\nabla h}^2+L_h^2)}{\text{min}\{|\cB_c|,|\cB_k|\}}+\frac{40\theta\gamma_2\beta^2\sigma_h^2 \Tilde{C}_{\nabla h}^2}{|\cB_k|}
    \end{split}
\end{equation}
Summing (\ref{eqn:total}), $\frac{4\eta}{\theta}\times$(\ref{eqn:delta_1}) and $\frac{4\eta}{\theta}\times$(\ref{eqn:delta_2}), let $\gamma_1=\gamma_2=\gamma \leq \min\{\frac{5n_0\theta}{3|\cB|}, \frac{5m\theta}{3|\cB_c|}\}$, we can get $\frac{5\theta n_0}{\gamma |\cB|}- 3\geq 0$, $\frac{5\theta m}{\gamma |\cB_c|}- 3\geq 0$ and
\begin{equation}
    \begin{split}
        &\eta\E \|\nabla \Phi(\vw^t)-v_t\|^2\\
        \leq &\E \left[Y_t - Y_{t+1} \right]\\
        & -\left(\frac{1}{4\eta}-\frac{L_F+\beta L_H}{2}-\frac{8\eta L_F^2}{\theta^2}-\frac{400\eta L_g^2 L_f^2 \Tilde{C}_{\nabla g}^2 n_0^2}{\gamma^2 |\cB|^2}-\frac{8\eta \beta^2L_H^2}{\theta^2}-\frac{400\eta\beta^2L_h^2\Tilde{C}_{\nabla h}^2 m^2}{\gamma^2|\cB_c|^2}\right)\E\left[\|\vw^{t+1}-\vw^t\|^2\right]\\
        & + \frac{8\eta\theta L_f^2(\sigma_{\nabla g}^2+L_g^2)}{\text{min}\{|\cB|,|\cB_{1i}|,|\cB_{2i}|\}}+\frac{320\eta\gamma\sigma_g^2 L_f^2 \Tilde{C}_{\nabla g}^2}{\min\{|\cB_{1i}|,|\cB_{2i}|\}}+\frac{4\eta\theta \beta^2(\sigma_{\nabla h}^2+L_h^2)}{\text{min}\{|\cB_c|,|\cB_k|\}}+\frac{160\eta\gamma\beta^2\sigma_h^2 \Tilde{C}_{\nabla h}^2}{|\cB_k|}
    \end{split}
\end{equation}
where
\begin{equation*}
    Y_{t+1} = \Phi(\vw^{t+1})+\frac{4\eta}{\theta}\|\nabla \Phi(\vw^{t+1})-v^{t+1}\|^2+\frac{80\eta L_f^2 \Tilde{C}_{\nabla g}^2n_0}{\gamma |\cB|}(\Xi^{t+1}_1+\Xi^{t+1}_2)+\frac{80\eta \beta^2 \Tilde{C}_{\nabla h}^2m}{\gamma |\cB_c|}\Gamma_{t+1}.
\end{equation*}
If $\eta = \min \left\{ \frac{1}{12(L_F+\beta L_H)}, \frac{\theta}{8\sqrt{3}L_F}, \frac{\theta}{8\sqrt{3}L_H \beta},\frac{\gamma |\cB|}{40\sqrt{6}L_g L_f \Tilde{C}_{\nabla g} n_0},\frac{\gamma |\cB_c|}{40\sqrt{6}\beta L_h\Tilde{C}_{\nabla h} m} \right\}$, we have
\begin{equation}
    \begin{split}
        &\frac{\eta}{24}\E\left[ \eta^{-2}\|\vw^{t+1}-\vw^t\|^2+\|\nabla \Phi(\vw^t)-v_t\|^2 \right]\\
        \leq &\E \left[Y_t - Y_{t+1}\right]
        + \frac{8\eta\theta L_f^2(\sigma_{\nabla g}^2+L_g^2)}{\text{min}\{|\cB|,|\cB_{1i}|,|\cB_{2i}|\}}+\frac{320\eta\gamma\sigma_g^2 L_f^2 \Tilde{C}_{\nabla g}^2}{\min\{|\cB_{1i}|,|\cB_{2i}|\}}+\frac{4\eta\theta \beta^2(\sigma_{\nabla h}^2+L_h^2)}{\text{min}\{|\cB_c|,|\cB_k||\}}+\frac{160\eta\gamma\beta^2\sigma_h^2 \Tilde{C}_{\nabla h}^2}{|\cB_k|}.
    \end{split}
\end{equation}
Dividing both sides by $\frac{\eta}{24}$ and taking the average over $T$ we can get
\begin{equation}
\begin{split}
    &\frac{1}{T}\sum^{T-1}_{t=0}\E\left[ \eta^{-2}\|\vw^{t+1}-\vw^t\|^2+\|\nabla \Phi(\vw^t)-v^t\|^2 \right]\\
    \leq &\frac{24\E[Y_0]}{\eta T}+ \frac{192\theta L_f^2(\sigma_{\nabla g}^2+L_g^2)}{\text{min}\{|\cB|,|\cB_{1i}|,|\cB_{2i}|\}}+\frac{7680\gamma\sigma_g^2 L_f^2 \Tilde{C}_{\nabla g}^2}{\min\{|\cB_{1i}|,|\cB_{2i}|\}}+\frac{96\theta \beta^2(\sigma_{\nabla h}^2+L_h^2)}{\text{min}\{|\cB_c|,|\cB_k|\}}+\frac{3840\gamma\beta^2\sigma_h^2 \Tilde{C}_{\nabla h}^2}{|\cB_k|}.
\end{split}
\end{equation}
\begin{equation*}
\begin{split}
    Y_0 & = \Phi(\vw^0)+\frac{4\eta}{\theta}\|\nabla \Phi(\bw^0)-v^{0}\|^2+\frac{80\eta L_f^2 \Tilde{C}_{\nabla g}^2n_0}{\gamma |\cB|}(\Xi_1^{0}+\Xi_2^{0})+\frac{80\eta \beta^2 \Tilde{C}_{\nabla h}^2m}{\gamma |\cB_c|}\Gamma_{0}\\
    & = \Phi(\bw^0)+\frac{4\eta}{\theta}\|\nabla \Phi(\bw^0)-v^{0}\|^2+\frac{80\eta L_f^2 \Tilde{C}_{\nabla g}^2}{\gamma |\cB|}(\|\textbf{u}_1^0-\textbf{g}_1(\bw^0)\|^2+\|\textbf{u}_2^0-\textbf{g}_2(\bw^0)\|^2)\\
    &\quad \quad \quad +\frac{80\eta \beta^2 \Tilde{C}_{\nabla h}^2}{\gamma |\cB_c|}\|\textbf{u}^0-\textbf{h}(\bw^0)\|^2.
\end{split}
\end{equation*}
Since $\bw^0$ is a feasible solution, we have $\Phi(\bw^0)=\frac{1}{n_0}\sum_{i=1}^{n_0} f(g_i(\vw^0))$. Since $g$ is bounded by Assumption \ref{assumption1} and $f$ is Lipschitz continuous, we can show that there exists a constant $C_F:=\max\{\tau|\log(c_g^2)|, \tau|\log(C_g^2)|\}$ such that $|F(\w,\cD)|\leq C_F$. We assume that $u^0_k=h_k(\bw^0)$, $u^0_{1i} = \hat{g}_{1i}(\bw^0)$, $u^0_{2i} = \hat{g}_{2i}(\bw^0)$ and $v^0= \frac{1}{n_0}\sum_{i=1}^{n_0}(\nabla \hat{g}_{1i}(\vw^0)^\top\nabla f(\hat{g}_{1i}(\vw^0)+\nabla \hat{g}_{2i}(\vw^0)^\top\nabla f(\hat{g}_{2i}(\vw^0))$, we can get 
\begin{equation}
    \begin{split}
        \E[\|\nabla \Phi(\bw^0)-v^{0}\|^2]&\leq 2(C_{\nabla g}^2+\sigma_{\nabla g}^2)C_{\nabla f}^2\\
        \E[\|\textbf{u}_1^0-\textbf{g}_1(\bw^0)\|^2]&\leq \sigma_g^2\\
        \E[\|\textbf{u}_2^0-\textbf{g}_2(\bw^0)\|^2]&\leq \sigma_g^2\\
        \E[\|\textbf{u}^0-\textbf{h}(\bw^0)\|^2] &= 0
    \end{split}
\end{equation}
Therefore, we can get
\begin{equation}
    \begin{split}
        &\frac{1}{T}\sum^{T-1}_{t=0}\E\left[ \eta^{-2}\|\vw^{t+1}-\vw^t\|^2+\|\nabla \Phi(\vw^t)-v^t\|^2 \right]\\
    \leq &\frac{24C_F}{\eta T}+\frac{192(C_{\nabla g}^2+\sigma_{\nabla g}^2)C_{\nabla f}^2}{\theta T}+\frac{1920 L_f^2 \Tilde{C}_{\nabla g}^2\sigma_g^2}{|\cB|\gamma T}\\
    & + \frac{192\theta L_f^2(\sigma_{\nabla g}^2+L_g^2)}{\text{min}\{|\cB|,|\cB_{1i}|,|\cB_{2i}|\}}+\frac{7680\gamma\sigma_g^2 L_f^2 \Tilde{C}_{\nabla g}^2}{\min\{|\cB_{1i}|,|\cB_{2i}|\}}+\frac{96\theta \beta^2(\sigma_{\nabla h}^2+L_h^2)}{\text{min}\{|\cB_c|,|\cB_k|\}}+\frac{3840\gamma\beta^2\sigma_h^2 \Tilde{C}_{\nabla h}^2}{|\cB_k|}.
    \end{split}
\end{equation}

Let $\beta = \frac{1}{\epsilon\delta}$, $\theta = \min\left\{\frac{\epsilon^4\delta^2\min\{|\cB_k|,|\cB_c|\}}{672(\sigma_{\nabla h}^2+L_h^2)} ,\frac{\epsilon^2\min\{|\cB|,|\cB_{1i}|,|\cB_{2i}|\}}{1344L_f^2(\sigma_{\nabla g}^2+L_g^2)}\right\}=O(\epsilon^4\delta^2)$, \\
$\gamma_1=\gamma_2=\gamma \leq \min\left\{\frac{5n_0\theta}{3|\cB|}, \frac{5m\theta}{3|\cB_c|},\frac{\epsilon^4\delta^2|\cB_k|}{26880 \sigma_h^2 \Tilde{C}_{\nabla h}^2}\right\} = O(\epsilon^4\delta^2)$,\\ $\eta = \min \left\{ \frac{1}{12(L_F+\beta L_H)}, \frac{\theta}{8\sqrt{3}L_F}, \frac{\theta}{8\sqrt{3}L_H \beta},\frac{\gamma_1 |\cB|}{40\sqrt{6}L_g L_f \Tilde{C}_{\nabla g} n},\frac{\gamma_2 |\cB_c|}{40\sqrt{6}\beta L_h\Tilde{C}_{\nabla h} m} \right\} = O(\epsilon^5\delta^3)$ and\\ $T = O(\epsilon^{-7}\delta^{-3})$, we have
\begin{equation*}
        \frac{1}{T}\sum^{T-1}_{t=0}\E\left[ \eta^{-2}\|\vw^{t+1}-\vw^t\|^2+\|\nabla \Phi(\vw^t)-v^t\|^2 \right]\\
    \leq O(\epsilon^2)
\end{equation*}

By the definition of $\vw^{t+1}$, we have
\begin{equation*}
    \begin{split}
       & \vw^{t+1}-\vw^t+\eta v^t=0\\
       & \Leftrightarrow \eta^{-1}(\vw^t-\vw^{t+1})+(\nabla \Phi(\vw^t)-v^t) + (\nabla \Phi(\vw^{t+1})-\nabla \Phi(\vw^t))=\nabla \Phi(\vw^{t+1})\\
       & \Leftrightarrow \eta^{-1}(\vw^t-\vw^{t+1})+(\nabla \Phi(\vw^t)-v^t) + (\nabla \Phi(\vw^{t+1})-\nabla \Phi(\vw^t))\\
       & \quad \quad \quad = \nabla F(\vw^{t+1})+\frac{\beta}{m}\nabla \textbf{h}(\vw^{t+1}) [\textbf{h}(\vw^{t+1})]_+
    \end{split}
\end{equation*}
This gives
\begin{equation*}
    \begin{split}
        &\quad \left\|\nabla F(\vw^{t+1})+\frac{\beta}{m}\nabla \textbf{h}(\vw^{t+1})^\top [\textbf{h}(\vw^{t+1})]_+\right\|^2\\
        & \leq 3\left( \eta^{-2}\|\vw^t-\vw^{t+1}\|^2 + \|\nabla \Phi(\vw^t)-v^t \|^2 + \|\nabla \Phi(\vw^{t+1})-\Phi(\vw^t) \|^2 \right)\\
        & \leq 3\left( \eta^{-2}\|\vw^t-\vw^{t+1}\|^2 + \|\nabla \Phi(\vw^t)-v^t \|^2 + L_{\beta}^2\| \vw^t-\vw^{t+1} \|^2 \right)\\
        & \leq 3\left( \eta^{-2}\|\vw^t-\vw^{t+1}\|^2 + \|\nabla \Phi(\vw^t)-v^t \|^2 + \eta^{-2}\| \vw^t-\vw^{t+1} \|^2 \right)\\
        & \leq 6 (\eta^{-2}\|\vw^t-\vw^{t+1}\|^2 + \|\nabla \Phi(\vw^t)-v^t \|^2)
    \end{split}
\end{equation*}
Therefore, we can achieve that 
\begin{equation}
    \begin{split}
        \frac{1}{T}\sum^{T-1}_{t=0}\E\left[\left\|\nabla F(\vw^{t+1})+\frac{\beta}{m}\nabla \textbf{h}(\vw^{t+1}) [\textbf{h}(\vw^{t+1})]_+\right\|^2\right] \leq O(\epsilon^2) 
    \end{split}
\end{equation}
By Jensen's inequality, we can get
\begin{equation}
    \begin{split}
        \E\left[\left\|\nabla F(\vw^{\hat{t}})+\frac{\beta}{m}\nabla \textbf{h}(\vw^{\hat{t}}) [\textbf{h}(\vw^{\hat{t}})]_+\right\|\right] \leq O(\epsilon), 
    \end{split}
\end{equation}
with $\hat{t}$ selected uniformly at random from $\{1,\cdots,T\}$.

Then, with the full rank assumption on the Jacobian, which is $\| \nabla \textbf{h}(\bw^t) [\textbf{h}(\bw^t)]_+ \|\geq \delta \|[\textbf{h}(\bw^t)]_+\|$ as in Assumption~\ref{ass:full_rank}, we can get
\begin{equation}
\begin{split}
    \|[\textbf{h}(\vw^{t+1})]_+\|^2\leq& \frac{1}{\delta^2}\| \nabla \textbf{h}(\vw^{t+1}) [\textbf{h}(\vw^{t+1})]_+ \|^2\\
    = & \frac{m^2}{\beta^2\delta^2}\|\nabla F(\vw^{t+1}) + \frac{\beta}{m} \nabla \textbf{h}(\vw^{t+1}) [\textbf{h}(\vw^{t+1})]_+ -\nabla F(\vw^{t+1})\|^2\\
    \leq & \frac{2m^2}{\beta^2\delta^2}\left[ \left\|\nabla F(\vw^{t+1})\right\|^2 + \left\|\nabla F(\vw^{t+1}) + \frac{\beta}{m} \nabla \textbf{h}(\vw^{t+1}) [\textbf{h}(\vw^{t+1})]_+ \right\|^2 \right]
\end{split}
\end{equation}

Taking the average over $T$, we can get
\begin{equation}
\begin{split}
    \frac{1}{T}\sum^{T-1}_{t=0} \E\|[\textbf{h}(\vw^{t+1})]_+\|^2\leq&\frac{1}{T}\sum^{T-1}_{t=0} \frac{2m^2}{\beta^2\delta^2}\E\left[ \left\|\nabla F(\vw^{t+1})\right\|^2 + \left\|\nabla F(\vw^{t+1}) + \frac{\beta}{m} \nabla \textbf{h}(\vw^{t+1}) [\textbf{h}(\vw^{t+1})]_+ \right\|^2 \right]\\
    \leq & O(\epsilon^2)
\end{split}
\end{equation}
and using $\boldsymbol{\lambda} = \frac{\beta}{m} [\textbf{h}(\vw^{\hat{t}})]_+$. By Jensen's inequality, we can get
\begin{equation}
    \E\|[\textbf{h}(\vw^{\hat{t}})]_+\| \leq O(\epsilon) 
\end{equation}
\begin{equation}
\begin{split}
    \E|\boldsymbol{\lambda}^\top[\textbf{h}(\vw^{\hat{t}})]_+| = \E \left|\frac{\beta}{m}[\textbf{h}(\vw^{\hat{t}})]_+^\top [\textbf{h}(\vw^{\hat{t}})]_+ \right|& = \frac{\beta}{m} \E \|[\textbf{h}(\vw^{\hat{t}})]_+\|^2\\
    & = \frac{1}{m\delta\epsilon} \E \|[\textbf{h}(\vw^{\hat{t}})]_+\|^2\\
    & \leq O(\epsilon).
\end{split}
\end{equation}


\end{proof}

\end{document}